\documentclass[twoside,11pt]{article}

\usepackage[preprint]{jmlr2e}

\usepackage{amsmath}
\usepackage{mathrsfs}  
\usepackage{todonotes}
\usepackage{algorithm, algorithmic}
\usepackage{multirow}
\newcommand{\Prb}{\mathbb{P}}
\newcommand{\E}{\mathbb{E}}
\newcommand{\eq}{\stackrel{d}{=}}
\newcommand{\tP}{\tilde \Prb}
\newcommand{\e}{\varepsilon}
\newcommand{\te}{\tilde \e}
\newcommand{\R}{\mathbb{R}}
\newcommand{\cW}{{{\mathcal W}}}
\newcommand{\bll}{\boldsymbol{1}}
\newcommand{\tE}{\tilde \E}

\newcommand{\abs}[1]{\left|#1\right|}
\newcommand{\norm}[1]{\left\|#1\right\|}
\newcommand{\ip}[1]{\left\langle#1\right\rangle}

\newtheorem{assumption}{Assumption}
\newtheorem{prop}{Proposition}
\newtheorem{defi}{Definition}

\jmlrheading{}{2000}{1-48}{4/00}{10/00}{meila00a}{Kexin Jin, Jonas Latz, Chenguang Liu, and Carola-Bibiane Sch\"onlieb}

\ShortHeadings{Stochastic Gradient Descent: Continuous Time and Data}{Jin et al.}
\firstpageno{1}

\begin{document}

\title{A Continuous-time Stochastic Gradient Descent Method for Continuous Data}

\author{\name Kexin Jin \email kexinj@math.princeton.edu \\
       \addr Department of Mathematics\\
       Princeton University\\
       Princeton, NJ 08544-1000, USA
       \AND
       \name Jonas Latz \email j.latz@hw.ac.uk \\
       \addr School of Mathematical and Computer Sciences\\
       Heriot-Watt University\\
       Edinburgh, EH14 4AS, United Kingdom
       \AND
       \name Chenguang Liu \email C.Liu-13@tudelft.nl \\
       \addr Delft Institute of Applied Mathematics\\
       Technische Universiteit Delft\\
       2628 Delft, The Netherlands
       \AND 
       \name Carola-Bibiane Sch\"onlieb \email cbs31@cam.ac.uk \\
       \addr Department of Applied Mathematics and Theoretical Physics\\
       University of Cambridge\\
       Cambridge, CB3 0WA, United Kingdom
       }

\editor{}

\maketitle

\begin{abstract}%   <- trailing '%' for backward compatibility of .sty file
Optimization problems with continuous data appear in, e.g., robust machine learning, functional data analysis, and variational inference. Here, the target function is given as an integral over a family of (continuously) indexed target functions -- integrated with respect to a probability measure. Such problems can often be solved by stochastic optimization methods:  performing optimization steps with respect to the indexed target function with randomly switched indices. 

In this work, we study a continuous-time variant of the stochastic gradient descent algorithm for optimization problems with continuous data. This so-called stochastic gradient process consists in a gradient flow minimizing an indexed target function that is coupled with a continuous-time index process determining the index.
Index processes are, e.g., reflected diffusions, pure jump processes, or other L\'evy processes on compact spaces. Thus, we study multiple sampling patterns for the continuous data space and allow for data simulated or streamed at runtime of the algorithm. We analyze the approximation properties of the stochastic gradient process and study its longtime behavior and ergodicity under constant and decreasing learning rates. We end with illustrating the applicability of the stochastic gradient process in a polynomial regression problem with noisy functional data, as well as in a physics-informed neural network.
\end{abstract}


\section{Introduction}%\label{subsec_problemsett}
%The stochastic gradient descent method and its variations are ubiquitous in the training of machine learning models. The training is usually represented as a minimisation problem: the model's parameters need to be calibrated such that the model represents the collected data as good as possible. The goodness of fit is here represented by a loss function that sums the losses from the misrepresentation of every single data set. Classical, we would try to minimize all losses at once -- in practice, this is neither computationally feasible nor robust. The stochastic gradient descent method proceeds by optimising only one or few losses in each iteration. Under certain convexity assumptions and 
The training of a machine learning model is often represented through an optimization problem. The goal is to calibrate the model's parameters to optimize its goodness-of-fit with respect to training data. The goodness-of-fit is usually quantified through a loss function that sums up the losses from the misrepresentation of every single training data set, see, e.g., \citet{Goodfellow2016} or \citet{Hansen2010,LeCam90} for similar optimization problems in imaging and statistics.
In `big data' settings, the sum of these loss functions consists of thousands or millions of terms, making classical optimization methods computationally infeasible.
Thus, efficiently solving optimization problems of this form has been a focus of machine learning and optimization research in the past decades. Here, methods often  build upon the popular stochastic gradient descent method.

Originally, stochastic gradient descent was proposed by \citet{RobbinsMonro} to optimize not only sums of loss functions, but also expectations of randomized  functions.\footnote{Actually, the `stochastic approximation method' of \citet{RobbinsMonro} aims at finding roots of functions that are given as expectations of randomized functions. The method they construct resembles stochastic gradient descent for a least squares loss function.} Of course, a normalized sum is just a special case of an expected value, making stochastic gradient descent available for the kind of training problem described above. Based on stochastic gradient descent ideas, improved algorithms have been proposed for optimizing sums of loss functions, such as \citet{Chambolle2018,Defazio2014,Duchi2011}. Unfortunately, these methods often specifically target sums of loss functions and are often infeasible to optimize general expected values of loss functions.

The optimization of expected values of loss functions appears in the presence of countably infinite and continuous data in functional data analysis and non-parametric statistics (e.g., \citet{Sinova2018}), physics-informed deep learning (e.g., \citet{PINN}), inverse problems  (e.g., \citet{Bredies2018}), and continuous data augmentation/adversarial robustness (e.g., \citet{cohen,Shorten2019,AdversarialRobustRL}). Some of these problems are usually studied after discretising the data. Algorithms for discrete data sometimes deteriorate at the continuum limit, i.e. as the number of data sets goes to infinity. Thus, we prefer studying the continuum case immediately. Finally, `continuous data' can also refer to general noise models. Here, expected values are minimized in robust optimization (e.g.,  \citet{Nemirovski2009}), variational Bayesian inference (e.g., \citet{Cherief2019}), and optimal control (e.g., \citet{May2013}). Overall,  the  optimization of general expected values is a very important task in modern data science, machine learning, and related fields.

In this work, we study stochastic gradient descent for general expected values in a continuous-time framework. We now proceed with the formal introduction of the optimization problem, the stochastic gradient descent algorithm, its continuous-time limits, and current research in this area.

\subsection{Problem setting and state of the art} \label{subsec_problemsett}
We study optimization problems of the form
\begin{equation} \label{Eq:OptProb}
    \min_{\theta \in X} \Phi(\theta) :=  \int_S f(\theta, y) \pi(\mathrm{d}y),
\end{equation}
where $X := \mathbb{R}^K$, $S$ is a Polish space,  $f: S \times X \rightarrow \mathbb{R}$ is a measurable function that is continuously differentiable in the first variable, and $\pi$ is a probability measure on $S$. Moreover, we assume that the integral above always exists.
We refer to $X$ as \emph{parameter space}, $S$ as \emph{index set}, $\Phi$ as \emph{full target function}, and $f$ as \emph{subsampled target function}. In these optimization problems, it is usually impossible or intractable to evaluate the integral $\Phi(\theta)$ or its gradient $\nabla \Phi(\theta)$ for $\theta \in X$. 
Hence, traditional optimization algorithms, such as steepest gradient descent or Newton methods are not applicable.

As mentioned above, it is possible to employ stochastic optimization methods, such as the \emph{stochastic gradient descent (SGD)} method, see \citet{KushnerYin,RobbinsMonro}. The stochastic gradient descent method for \eqref{Eq:OptProb} proceeds through the following discrete-time dynamic that iterates over $n \in \mathbb{N}:= \{1,2,\ldots \}$:
\begin{align} \label{Eq:SGD_discrete_time}
    \theta_{n} = \theta_{n-1} - \eta_n \nabla_{\theta}f(\theta_{n-1}, y_n), 
\end{align}
where $y_1, y_2, \ldots \sim \pi$ independent and identically distributed (i.i.d.), $(\eta_n)_{n=1}^\infty \in (0, \infty)^\mathbb{N}$ is a non-increasing sequence of \emph{learning rates}, and $\theta_0 \in X$ is an appropriate initial value. 
Hence, SGD is an iterative method that employs only the gradient of the integrand $f$, but not $\Phi$.  
SGD converges to the minimizer of $\Phi$, if $\eta_n \rightarrow 0$, as $n \rightarrow \infty$, sufficiently slowly, and $f(\cdot, y)$ $(y \in S)$ is strongly convex; see, e.g., \citet{Bubeck}. Additionally, SGD is used in practice also for non-convex optimization problems and with constant learning rate. The constant learning rate setting is popular especially due to its regularizing properties; see \citet{Ali20a,smith2021on}.

To understand, improve, and study discrete-time dynamical systems, it is sometimes advantageous to represent them in continuous time, see, e.g. the works by \citet{deWiljes2018,Kovachki21,Trillos20}.  Continuous-time models allow us to concentrate on the underlying dynamics and omit certain numerical considerations. Moreover, they give us natural ways to construct new, efficient algorithms.

The discrete-time dynamic in \eqref{Eq:SGD_discrete_time} is sometimes represented through a continuous-time diffusion process, see \citet{Ali20a,Li2019, Weinan2, Mandt2016, Mandt2017,wojtowytsch2021stochastic}:
$$
\mathrm{d} \theta_t = -\nabla \Phi(\theta_t) \mathrm{d}t + \sqrt{\eta(t)} \Sigma(\theta_t)^{1/2} \mathrm{d}W_t, 
$$
where $\Sigma(\theta) = \int (\nabla_\theta f(\theta; y)- \nabla_\theta\Phi(\theta)) \otimes (\nabla_\theta f(\theta; y)- \nabla_\theta \Phi(\theta)) \pi(\mathrm{d}y)$, $(W_t)_{t \geq 0}$ is a $K$-dimen\-sional Brownian motion, and $(\eta(t))_{t \geq 0}$ is an interpolation of the learning rate sequence.  While this diffusion approach is suitable to describe the dynamic of the moments of SGD, it does not immediately allow us to construct new stochastic optimization algorithms, as the system depends on the inaccessible $\nabla \Phi$.

\begin{figure}
    \centering
    \input{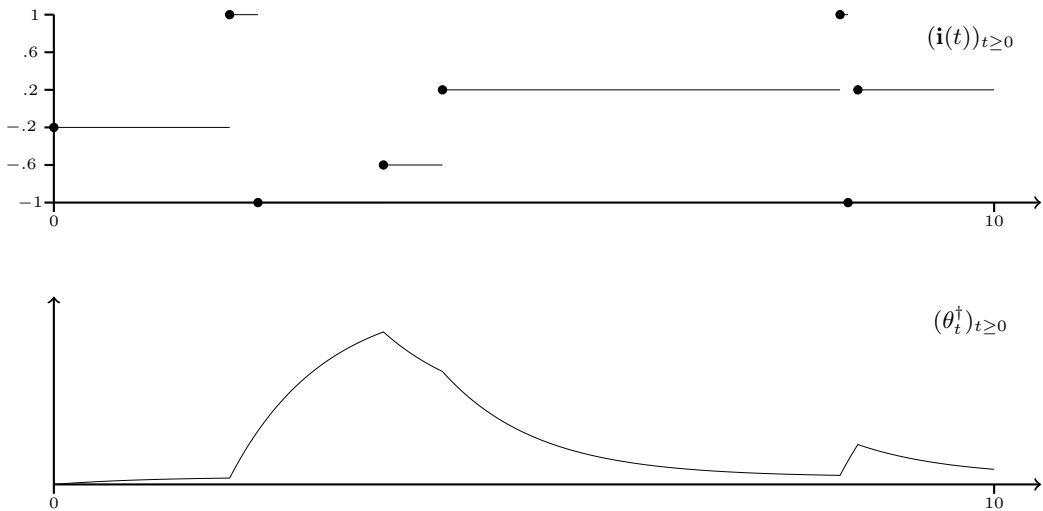}
   \caption{Cartoon of the stochastic gradient process $(\theta^\dagger_t)_{t \geq 0}$ with index process $(\mathbf{i}(t))_{t \geq 0}$ on the discrete index set $S:= \{-1, -0.6,\ldots,1\}$. The index process is a Markov pure jump process on $S$. The process $(\theta^\dagger_t)_{t \geq 0}$ aims at optimizing the $\mathrm{Unif}(S)$-integral of the subsampled target functional is $f(\theta,y) := \frac{1}{2}(\theta - y^2)^2$ $(\theta \in X:=\mathbb{R}, y \in S)$.}
    \label{fig:cartoon_discrete}
\end{figure}

A continuous-time representation of stochastic gradient descent that does not depend on $\Phi$ has recently been proposed by \citet{Latz}. This work only considers the discrete data case, i.e., $S$ is finite and $\pi := \mathrm{Unif}(S)$. SGD is represented by the \emph{stochastic gradient process} $(\theta^\dagger_t)_{t \geq 0}$. It is defined through the coupled dynamical system
\begin{equation}\label{eq_SGPC_discrete}
    \mathrm{d} \theta_t^\dagger = - \nabla_{\theta} f(\theta_t^\dagger; \mathbf{i}(t)) \mathrm{d}t,
\end{equation}
where $(\mathbf{i}(t))_{t \geq 0}$ is a suitable continuous-time Markov process on $S$, which we call \emph{index process}. Hence, the process $(\theta_t^\dagger)_{t \geq 0}$ represents gradient flows with respect to the subsampled target functions that are switched after random waiting times. The random waiting times are controlled by the continuous-time Markov process $(\mathbf{i}(t))_{t \geq 0}$. 
We show an example of the coupling of exemplary process $(\mathbf{i}(t))_{t \geq 0}$ and $(\theta_t^\dagger)_{t \geq 0}$ in Figure~\ref{fig:cartoon_discrete}. The setting is $S := \{-1, -0.6,\ldots,1\}$, $\pi := \mathrm{Unif}(S)$, $X:= \mathbb{R}$, and $f(\theta, y) := \frac{1}{2}(\theta - y^2)^2$ $(\theta \in X, y \in S)$. There, we see that the sample path of $(\theta_t^\dagger)_{t \geq 0}$ is piecewise smooth, with non-smooth behavior at the jump times of   $(\mathbf{i}(t))_{t \geq 0}$.

If the process $(\mathbf{i}(t))_{t \geq 0}$ is homogeneous-in-time, the dynamical system represents a constant learning rate. Inhomogeneous $(\mathbf{i}(t))_{t \geq 0}$ with decreasing mean waiting times, on the other hand, model a decreasing learning rate. Under certain assumptions, the process $(\theta_t^\dagger)_{t \geq 0}$ converges to a unique stationary measure when the learning rate is constant or to the minimizer of $\Phi$ when the learning rate decreases.

\subsection{This work.} \label{Subsec_Intro_thisWork}
We now briefly introduce the continuous-time stochastic gradient descent methods that we study throughout this work. Then, we summarize our main contributions and give a short paper outline.

In the present work, we aim to generalize the dynamical system \eqref{eq_SGPC_discrete} to include more general spaces $S$ and probability measures $\pi$  -- studying the more general optimization problems of type \eqref{Eq:OptProb}.
We proceed as follows: We define a stationary continuous-time Markov process $(V_t)_{t \geq 0}$ on $S$ that is geometrically ergodic and has $\pi$ as its stationary measure. This process $(V_t)_{t \geq 0}$ is now our \emph{index process}. Similarly to \eqref{eq_SGPC_discrete}, we then couple $(V_t)_{t \geq 0}$ with the following gradient flow:
\begin{equation} \label{Eq_SGPC}
    \mathrm{d} \theta_t = - \nabla_{\theta} f(\theta_t, V_t) \mathrm{d}t.
\end{equation}
Note that the index process $(V_t)_{t \geq 0}$ can be considerably more general than the Markov jump processes studied by \citet{Latz}; we discuss  examples below and in Section~\ref{mainprocess}.
As the dynamical system \eqref{Eq_SGPC} contains the discrete version \eqref{eq_SGPC_discrete} as a special case, we refer to $(\theta_t)_{t \geq 0}$ also as \emph{stochastic gradient process}.

We give an example for $(\theta_t, V_t)_{t \geq 0}$ in Figure~\ref{fig:cartoon_continuous}. There, we consider $S :=[-1,1]$,  $\pi := \mathrm{Unif}[-1, 1]$, $X:= \mathbb{R}$, and $f(\theta,y) := \frac{1}{2}(\theta - y^2)^2$ $(\theta \in X, y \in S)$. A suitable choice for  $(V_t)_{t \geq 0}$ is  a reflected Brownian motion on $[-1, 1]$.  Although it is coupled with $(V_t)_{t \geq 0}$, the process  $(\theta_t)_{t \geq 0}$ appears to be relatively smooth. This may be due to the smoothness of the subsampled target function $f$. 
Moreover, we note that the example in Figure~\ref{fig:cartoon_discrete} is a discretized data version of the example here in Figure~\ref{fig:cartoon_continuous}.

\begin{figure}
    \centering
    \input{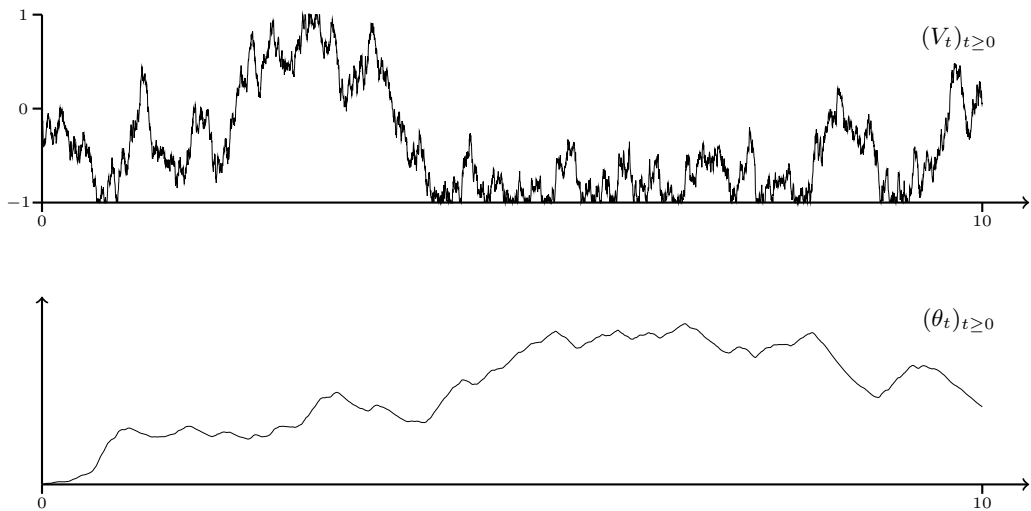}
   \caption{Cartoon of the stochastic gradient process $(\theta_t)_{t \geq 0}$ with index process $(V_t)_{t \geq 0}$, with continuous index set $S:= [-1,1]$. The index process is a reflected Brownian motion on $S$.}
    \label{fig:cartoon_continuous}
\end{figure}

More similarly to the discrete data case \eqref{eq_SGPC_discrete}, one could also choose $(V_t)_{t \geq 0}$ to be a Markov pure jump process on $S$ that has $\pi$ as a stationary measure.
Indeed, the reflected Brownian motion was constructed rather artificially. Sampling from $\mathrm{Unif}[-1, 1]$ is not actually difficult in practice and we just needed a way to find a continuous-time Markov process that is stationary with respect to $\mathrm{Unif}[-1, 1]$. However, there are cases, where one may not be able to sample independently from $\pi$. For instance, $\pi$ could be the measure of interest in a statistical physics simulation or Bayesian inference.  In those cases, Markov chain Monte Carlo methods are used to approximate $\pi$ through a Markov chain stationary with respect to it, see, e.g., \citet{Robert2004}. In other cases, the data might be time series data that is streamed at runtime of the algorithm -- a related problem has been studied by \citet{Sirignano2017SIFIN}.
Hence, in the work, we also discuss stochastic optimization in those cases or -- more generally -- stochastic optimization with respect to data from arbitrary sources.

As the index process $(V_t)_{t \geq 0}$ is stationary, the stochastic gradient process  as defined above would, again, represent the situation of a constant learning rate $(\eta_n)_{n =1}^\infty$. 
However, as  before  we usually cannot hope for convergence to a stationary point if there is not a sense of a decreasing learning rate.
Hence, we need to introduce an inhomogeneous variant of $(V_t)_{t \geq 0}$ that represents a decreasing learning rate.

We start with the stochastic process $(V_t^{\rm dc})_{t \geq 0}$ that represents the index process associated to the discrete-time stochastic gradient descent dynamic \eqref{Eq:SGD_discrete_time} with constant learning rate parameter $\eta = 1$. This index process is given by
$$
V_t^{\rm dc} = \sum_{n=1}^\infty y_n\mathbf{1}[t \in [n-1,n) ] = \sum_{n=1}^\infty y_n\mathbf{1}[t - n +1 \in [0,1) ] \qquad (t \geq 0),
$$
where $y_1, y_2, \ldots \sim \pi$ i.i.d.\ and  $\mathbf{1}[\cdot]$ is the indicator function: $\mathbf{1}[{\rm true}] = 1$ and $\mathbf{1}[{\rm false}] = 0$.
We now want to turn the process $(V_t^{\rm dc})_{t \geq 0}$ into the index process $(V_t^{\rm dd})_{t \geq 0}$ that represents a decreasing learning rate $(\eta_n)_{n=1}^\infty$.  It is defined through:
$$(V_t^{\rm dd})_{t \geq 0}  = \sum_{n=1}^\infty y_n\mathbf{1}\left[t \in \left[H_{n-1}, H_n\right) \right] =\sum_{n=1}^\infty y_n\mathbf{1}\left[ \frac{t- H_{n-1}}{\eta_n} \in \left[0, 1\right) \right], $$
where we denote $H_n := \sum_{m = 1}^n \eta_m$.
Hence, we can represent $(V_t^{\rm dd})_{t \geq 0} := (V_{\beta(t)}^{\rm dc})_{t \geq 0}$, where $\beta: [0, \infty) \rightarrow [0, \infty)$ is given by
\begin{equation} \label{eq_beta_and_eta}
    \beta(t) = \sum_{n= 1}^\infty \frac{t+n-1-H_{n-1}}{\eta_n}\mathbf{1}\left[t \in \left[H_{n-1}, H_n\right) \right] \qquad (t \geq 0)
\end{equation}
is a piecewise linear, non-decreasing function with $\beta(t) \rightarrow \infty,$ as $t \rightarrow \infty$. 

Following this idea, we turn our homogeneous index process $(V_t)_{t \geq 0}$ that represents a constant learning rate into an inhomogeneous process with decreasing learning rate using a suitable rescaling function $\beta$.
In that case, we obtain a stochastic gradient process of type 
\begin{equation*}
    \mathrm{d} \xi_t = - \nabla_{\xi} f(\xi_t, V_{\beta(t)}) \mathrm{d}t,
\end{equation*}
which we will use to represent the stochastic gradient descent algorithm with decreasing learning rate.
Note that while we require $\beta$ to satisfy certain conditions that ensure the well-definedness of the dynamical system, it is not strictly necessary for it to be of the form \eqref{eq_beta_and_eta}. Actually, we later  assume that $\beta$ is smooth.

The main contributions of this work are the following:

\begin{itemize}
\item We study stochastic gradient processes for optimization problems of the form \eqref{Eq:OptProb} with finite, countably infinite, and continuous index sets $S$. 
    \item We give conditions under which the stochastic gradient process with constant learning rate is well-defined and that it can approximate the full gradient flow  $\mathrm{d} \zeta_t = - \nabla \Phi(\zeta_t) \mathrm{d}t$ at any accuracy. In addition, we study the geometric ergodicity of the stochastic gradient process and properties of its stationary measure.
    
    \item We study the well-definedness of the stochastic gradient process with decreasing learning rate and give conditions under which the process converges to the minimizer of $\Phi$ in the optimization problem \eqref{Eq:OptProb}.
    
    \item In numerical experiments, we show the suitability of our stochastic gradient process for (convex) polynomial regression with continuous data and  the (non-convex) training of physics-informed neural networks with continuous sampling of function-valued data.
\end{itemize}

This work is organized as follows. In Section~\ref{mainprocess}, we study the index process $(V_t)_{t \geq 0}$ and give examples for various combinations of index spaces $S$ and probability measures $\pi$. Then, in Sections~\ref{Sec_SGPC} and \ref{Sec_SGPD}, we analyze the stochastic gradient process with constant and decreasing learning rate, respectively. In Section~\ref{Sec_PracticalOptimisation}, we review discretization techniques that allow us to turn the continuous dynamical systems into practical optimization algorithms. We employ these techniques in Section~\ref{Sec_NumExp}, where we present numerical experiments regarding polynomial regression and the training of physics-informed neural networks. We end with conclusions and outlook in Section~\ref{Sec_conclusions}.

\section{The index process: Feller processes and geometric ergodicity}\label{mainprocess}
Before we define the stochastic gradient flow, we introduce and study the class of stochastic processes $(V_t)_{t\ge 0}$ that can be used for the data switching in (\ref{Eq_SGPC}). Moreover, we give an overview of appropriate processes for various measures $\pi$. For more background material on (continuous-time) stochastic processes, we refer the reader to the book by \citet{RY}, the book by \citet{Liggett}, and other standard literature.

Let $\mathcal{S}=(S,m)$ be a compact Polish space and
\begin{align*}
    \Omega=\{\omega: [0,\infty)\to S\ |\ \omega\ \text{is right continuous with left limits.}\}.
\end{align*}
We consider a filtered probability space $(\Omega, \mathcal{F},(\mathcal{F}_t)_{t\ge 0},(\Prb_x)_{x\in S})$, where  $\mathcal{F}$  is the smallest  $\sigma$-algebra on $\Omega$ such that the mapping $\omega \to \omega(t)$ is measurable for any $t\ge 0$ and the filtration $\mathcal{F}_t$ is right continuous. Let  $(V_t)_{t\ge 0}$ be a $(\mathcal{F}_t)_{t\ge 0}$ adapted stochastic process from $\Omega$  to $S$.  We assume that $(V_t)_{t\ge 0}$ is Feller with respect to  $(\mathcal{F}_t)_{t\ge 0}.$ $(\Prb_x)_{x\in S}$ is a collection of probability measures on $\Omega$ such that $\Prb_x(V_0=x)=1.$
For any probability measure $\mu$ on $S$, we define
$$
\Prb_\mu(\cdot):=\int_S \Prb_y(\cdot)\mu(\mathrm{d}y)
$$
and denote expectations with respect to $\Prb_x$ and $\Prb_\mu$ by $\E_x$ and $\E_\mu$, respectively.

%\begin{align*}
 %   \Omega=D[0,\infty)=\{\textit{All c\`adl\`ag function}\ \omega: [0,\infty)\to S \}.
%\end{align*}
%The $\sigma$-algebra $\mathcal{F}$ on $\Omega$ is the
%smallest such that the mapping $\omega\to \omega(t)$ is measurable for each $t>0.$ $V_t(\omega)=\omega(t)$ and $(\theta_s \omega)(t)=\omega(t+s).$

%Next, we are going to follow the definition of Feller process in \cite[Chapter 3 P92]{Liggett}.
%\begin{defi}\label{maindef}
%We say a process $V_t$ is a Feller process on S if it satisfies

%(a) a series of probability measures $\{\Prb_x,x\in S\}$ on $(\Omega,\mathcal{F})$

%(b){\bf Markov} a right continuous filtration $\{\mathcal{F}_t\}_{t\ge 0}$ on $\Omega$ with respect to which
%the random variables $V_t$ are adapted, with $\Prb_x(V_0=x )=1.$
%\begin{align*}
 %   \E_x[Y\circ \theta_s|\mathcal{F}_s]=\E_{V_s}[Y]\ \ \ a.s.\ \Prb_x.
%\end{align*}
%for all $x\in S$ and all bounded measurable $Y$ on $\Omega$.

%(c){\bf Feller} for all $f\in \mathcal{C}(S)$, the mapping $x\to \E_x[f(V_t)]$ is also in $\mathcal{C}(S)$ for all $t\ge 0.$
%\end{defi}
Below we give a set of assumptions on the process $(V_t)_{t\ge 0}$. We need those to ensure that  a certain coupling property holds. We comment on these assumptions after stating them.

\begin{assumption}\label{as1.0}
Let $(V_t)_{t\ge 0}$ be a Feller process on $(\Omega, \mathcal{F}, (\mathcal{F}_t))_{t\ge 0},(\Prb_x)_{x\in S})$. We assume the following:
\begin{itemize}
    \item[(i)]  $(V_t)_{t\ge 0}$ admits a unique invariant measure $\pi$.

\item[(ii)]For any $x\in S$, there exist a family $(V^x_t)_{t\ge 0}$ and a stationary version $(V^\pi_t)_{t\ge 0}$  defined on the same probability space $(\tilde\Omega, \mathcal{\tilde F},\tilde\Prb)$  such that, $(V^x_t)_{t\ge 0}\eq (V_t)_{t\ge 0}$ in $\Prb_x$ and  $(V^\pi_t)_{t\ge 0}\eq (V_t)_{t\ge 0}$ in $\Prb_\pi$, i.e. for any $0\leq t_1< \cdots< t_n$,
$$\tilde\Prb(V_{t_1}^x\in A_1, \cdots, V_{t_n}^x \in A_n) = \Prb_x(V_{t_1}\in A_1, \cdots, V_{t_n} \in A_n),$$
$$\tilde\Prb(V_{t_1}^\pi\in A_1, \cdots, V_{t_n}^\pi \in A_n) = \Prb_\pi(V_{t_1}\in A_1, \cdots, V_{t_n} \in A_n),$$
where $A_1, \cdots, A_n \in \mathcal{B}(S)$.
\item[(iii)] Let $T^x:=\inf{\{t\ge 0\ |\ V^x_t=V^\pi_t \}}$ be a stopping time. There exist constants $C, \delta>0$ such that for any $t\geq 0$, 
$$
\sup_{x\in S}\tilde\Prb(T^x\ge t)\le C\exp({-\delta t}).
$$
\end{itemize}

%Assume there exist a family $(V^x_t)_{t\ge 0}$ and a stationary version $(V^\pi_t)_{t\ge 0}$  defined on the same probability space $(\tilde\Omega, \mathcal{\tilde F},\tilde\Prb)$  such that for any $x\in S$, $(V^x_t)_{t\ge 0}\eq (V_t)_{t\ge 0}$ in $\Prb_x$ and  $(V^\pi_t)_{t\ge 0}\eq (V_t)_{t\ge 0}$ in $\Prb_\pi$.  We define the stopping time $T^x:=\inf{\{t\ge 0|V^x_t=V^\pi_t \}},$
%$$
%\sup_{x\in S}\tilde\Prb(T^x\ge t)\le Ce^{-\delta t}
%$$
%where $C,\delta>0.$
\end{assumption}
First, we assume that $(V_t)_{t\geq 0}$ has a stationary measure $\pi$. 
Second, we assume that for the process $(V_t)_{t\ge 0}$ that starts from $x$ with probability $1$, we can find a coupled process $(V^x_t)_{t\ge 0}$. Also, given that the process $(V_t)_{t\ge 0}$ starts with its invariant measure $\pi$, we can find a stationary version $(V^\pi_t)_{t\ge 0}$ of $(V_t)_{t\ge 0}$. Here, the processes $(V^x_t)_{t\ge 0}$ and $(V^\pi_t)_{t\ge 0}$ are defined on the same probability space. 
Third, we assume that the processes $(V^x_t)_{t\ge 0}$ and $(V^\pi_t)_{t\ge 0}$ intersect exponentially fast. The exponential rate can be chosen uniformly in $x$ since $S$ is compact. 
With Assumption \ref{as1.0}, we have the following lemma.

\begin{lemma}[Geometric Ergodicity]\label{lemma1}
Under Assumption \ref{as1.0}, there exist constants $C, \delta>0$ such that for any $x\in S$ and $t\geq 0$,
\begin{align*}
     \sup_{A\in \mathcal{B}(S)} |\Prb_x(V_t\in A)-\pi(A)|\le C\exp({-\delta t}),
\end{align*}
 where $\mathcal{B}(S)$ is the set of all Borel measurable sets of $S$.
\end{lemma}

\begin{proof} 
For any given $x\in S$,
we construct the following process by coupling $(V^x_t)_{t\ge 0}$ and $(V^\pi_t)_{t\ge 0}$:
\begin{align*}
    \tilde V^x_t=\left\{
\begin{aligned}
  &V^x_t,\ \ \ 0\le t\le T^x,\\
 &V^\pi_t,\ \ \  t> T^x.
\end{aligned}
\right.
\end{align*}
By the strong Markov property, $(\tilde V^x_t)_{t\ge 0}\eq  (V^x_t)_{t\ge 0}$. For any $A\in \mathcal{B}(S)$, notice that 
\begin{align*}
    \abs{\Prb_x(V_t\in A)-\pi(A)}\\
    =& |\tP(V^x_t\in A)-\tP(V^\pi_t\in A)|\\
    =&|\tP(\tilde V^x_t\in A)-\tP(V^\pi_t\in A)|\\
    =& |\tP(\tilde V^x_t\in A, \tilde V^x_t\ne V^\pi_t) + \tP(\tilde V^x_t\in A, \tilde V^x_t = V^\pi_t) \\
    &\ \ \ -(\tP(V^\pi_t\in A, \tilde V^x_t\ne V^\pi_t)+\tP(V^\pi_t\in A, \tilde V^x_t= V^\pi_t))|\\
     =& |\tP(\tilde V^x_t\in A, \tilde V^x_t\ne V^\pi_t)   -\tP(V^\pi_t\in A, \tilde V^x_t\ne V^\pi_t)|\\
    \le& 2 \tP(\tilde V^x_t\ne V^\pi_t)\\
    \le& 2 \tP(T^x\ge t)\le C\exp({-\delta t}).
\end{align*}
From the third assumption in Assumption \ref{as1.0}, $C$ and $\delta$ are independent of $x$ and this completes the proof.
\end{proof}

In the lemma above, we have shown geometric ergodicity of $(V_t)_{t \geq 0}$ in the total variation distance. Next, we show that the same rate of convergence of $(V_t)_{t \geq 0}$ holds in the weak topology.

\begin{corollary}\label{corergodic}
 Under Assumption \ref{as1.0}, there exist constants $C, \delta>0$ such that for any $h\in \mathcal{C}(S),$ i.e. the set of all continuous function on $S$, we have
\begin{align*}
    \sup_{x\in S}\abs {\E_x[h(V_t)]- \int_S h(y)\pi(\mathrm{d}y)} \le C\norm{h}_{\infty}\exp({-\delta t})
\end{align*}
where $\norm{h}_{\infty}:=\sup_{x\in S}|h(x)|$
\end{corollary}

\begin{proof}
Rewrite $\E_x[h(V_t)]$ as
\begin{align*}
    \E_x[h(V_t)]= \int_S  h(y) \Prb_x(V_t\in \mathrm{d}y).
\end{align*}
Then we have
\begin{align*}
    \abs {\E_x[h(V_t)]- \int_S h(y)\pi(\mathrm{d}y)}&=\abs { \int_S  h(y) [\Prb_x(V_t\in \mathrm{d}y)-\pi(\mathrm{d}y)]}\\
    &\le \norm{h}_\infty\abs { \Prb_x(V_t)-\pi}( S),
\end{align*}
where $\frac{1}{2}\abs { \Prb_x(V_t)-\pi}( S)$ is the total variation of the measure $\Prb_x(V_t)-\pi.$ 
Notice that $$\abs { \Prb_x(V_t)-\pi}( S)=2\sup_{A\in\mathcal{B} (S)} |\Prb_x(V_t\in A)-\pi(A)|.$$
By Lemma  \ref{lemma1}, we have
\begin{align*}
    \sup_{x\in S}\abs {\E_x[h(V_t)]- \int_S h(y)\pi(\mathrm{d}y)}\le& \norm{h}_\infty\sup_{x\in S}\abs { \Prb_x(V_t)-\pi}( S)\\
    \le& 2\norm{h}_\infty\sup_{x\in S}\sup_{A\in\mathcal{B}(S)}|\Prb_x(V_t\in A)-\pi(A)|\\
    \le& C\norm{h}_{\infty}\exp({-\delta t}),
\end{align*}
which completes the proof.
\end{proof}

We now study four examples for processes that satisfy our assumptions: L\'evy processes with two-sided reflections on a compact interval,  continuous-time Markov processes on finite and countably infinite spaces, and processes on rectangular sets with independent coordinates.
\subsection{Example 1: L\'evy processes with two-sided reflection} \label{Subsec_Ex1_Levy_refle} For any $b>0,$ we say a triplet $((X_t)_{t\ge 0},(L_t)_{t\ge 0},(U_t)_{t\ge 0})$ is a solution to the Skorokhod problem of the L\'evy process $(X_t)_{t \geq 0}$ on the space $S := [0,b]$ if for all $t\geq0,$
\begin{align}\label{BMR}
    V_t= X_t +L_t-U_t,
\end{align}
where $(L_t)_{t \geq 0},\ (U_t)_{t \geq 0}$ are non-decreasing right continuous processes such that
\begin{align*}
    \int_0^\infty V_t\mathrm{d}L_t=\int_0^\infty (b-V_t)\mathrm{d}U_t=0.
\end{align*}
In other words, $(L_t)_{t \geq 0}$ and $(U_t)_{t \geq 0}$ can only increase when $(V_t)_{t \geq 0}$ is at the lower boundary $0$ or the upper boundary $b$. From \citet[Proposition~5.1]{Andersen1}, we immediately have that the process $(V_t)_{t \geq 0}$ in (\ref{BMR}) satisfies Assumption \ref{as1.0}. The geometric ergodicity follows from \citet[Remark~5.3]{Andersen1}.
As an example, the standard Brownian Motion (BM) reflected at 0 and 1 can be written as 
$$V_t = B_t + \tilde{L}_t^0 - \tilde{L}_t^1,$$
where $(B_t)_{t \geq 0}$ is a standard BM and $(\tilde{L}^a_t)_{t \geq 0}$ is the symmetric local time of $(V_t)_{t \geq 0}$ at $a$. Intuitively, a local time describes the time spent at a given point of a continuous stochastic process. The formal definition of symmetric local time of continuous semimartingales can be found, for example, in \citet[Chapter~VI]{RY}. For the optimization problem (\ref{Eq:OptProb}) with $S = [0, 1]$ and $\pi$ being the uniform measure on $S$, the corresponding stochastic process in (\ref{Eq_SGPC}) can be chosen to be this Brownian Motion with two-sided refection since its invariant measure is the uniform measure on $[0, 1]$. 
To see this, for $x\in[0, 1]$, from \citet[Theorem~5.4]{Andersen1}, 
\begin{align*}
    \pi([x,1])=\Prb(B_{\tau_x\land \tau_{x-1}}=x)=\Prb(\tau_x<\tau_{x-1})=1-x,
\end{align*}
where $\tau_a = \inf\{t\ge0| B_t = a\}$.

\subsection{Example 2: Continuous-time Markov processes} 
We consider a continuous-time Markov process $V_t$ on state space $I=\{1,2,\ldots,N\}$ with transition rate matrix
$$\mathbf{A}_N=\mathbf{\Lambda}_N-N\lambda  \mathbf{I}_{N},$$
where $\lambda>0$, $\mathbf{\Lambda}_N$ is a $N\times N$ matrix whose entry is always $\lambda$, and $\mathbf{I}_{N}$ is the identity matrix.
%$$A=
% \begin{pmatrix}
%  \lambda & ... &  \lambda \\
%... & \lambda & ... \\
%\lambda& ...& \lambda
% \end{pmatrix}-N\lambda \cdot Id_I $$
From \citet{Latz}, we know that the transition probability is given by
 \begin{align*}
     \Prb(V_{t+s}=i|V_s=j)=\frac{1-\exp(- \lambda Nt)}{N}+\exp(- \lambda Nt)\mathbf{1}[{i=j}].
 \end{align*}
The invariant measure $\pi$ of $V_t$ is the uniform measure on $I$, i.e. $\pi(i)=1/N$ for $i\in \{1,...,N\}$. To see that $(V_t)_{t \geq 0}$ satisfies the rest of Assumption \ref{as1.0}, consider a stationary version $(\hat V_t)_{t\ge 0}$ that is independent of $(V_t)_{t\ge 0}$. 
Let $V_0=1$. We define $T=\inf\{t\ge 0, V_t=\hat V_t\}$. Moreover, for $i,\ j\in \mathbb{N}$, we denote the $i$-th and $j$-th jump time of $(V_t)_{t\ge 0}$ and $(\hat V_t)_{t\ge 0}$ by $T_i$ and $\hat T_j$, respectively. Then we have
$$
\Prb(T=0)=\Prb(V_0=1,\hat V_0=1)=\Prb(V_0=1)\Prb(\hat V_0=1)=\frac{1}{N}.
$$
For any $i,\ j\in \mathbb{N}$, since $T_i$ and $\hat T_j$ are independent, $\Prb(T_i=\hat T_j)=0$. Let $Y_t=(V_t, \hat V_t)$ be a Markov process on $I\times I$ with transition probability:
\begin{align*}
     \Prb(Y_{t+s}=(i,j)|Y_s=(i_0,j_0))&=\frac{1-\exp(- 2\lambda Nt)}{2N}\Big(\mathbf{1}[{i=i_0}]+\mathbf{1}[{j=j_0}]\Big) \\ &\qquad +\exp(- 2\lambda Nt)\mathbf{1}[{(i,j)=(i_0,j_0)}].
\end{align*}
Thus, $T$ is the first time when $Y_t$ hits $\{(i,i)|i=1,...,N\}.$ Let the n-th jump time of $Y_t$ be $\tau_n$, for $t>0$, we have
\begin{align*}
    \Prb(T\ge t)=& \sum_{n\ge 1}\Prb(T=\tau_n,\tau_n\ge t)\\
    =&  \sum_{n\ge 1}\exp({-2(N-1)n\lambda t})\frac{N-1}{N}\frac{1}{N-1}\Big(\frac{N-2}{N-1}\Big)^{n-1}\\
    \le& C\exp({-2(N-1)\lambda t}),
\end{align*}
where the second equality follows from $$\Prb(T=\tau_n)= \frac{N-1}{N}\frac{1}{N-1}\Big(\frac{N-2}{N-1}\Big)^{n-1}$$  since there are $2N-4$ states available for the next jump.
Thus, $(V_t)_{t \geq 0}$ satisfies Assumption~\ref{as1.0}.

%{\color{red}
%\subsection{Semi-Group with reflecting boundaries}
%We consider the process on $[0,b]$ with reflecting boundaries at  $0$ and $b.$ It is for twice continuously
%differentiable functions $f$ satisfy $\mL f(x)=\frac{1}{2}f''(x)$ in $f'(0)=f'(b)=0.$
%}

\subsection{Example 3: Continuous-time Markov processes with countable states}
We consider a continuous-time Markov process $(V_t)_{t \geq 0}$ on state space $\mathbb{N}_0 := \mathbb{N} \cup \{0\}$ with exponential jump times. At time $t$, if $V_t\in\mathbb{N},$ it jumps to $0$ with probability $1$ at the next jump time. Otherwise, if $V_t = 0$, it jumps to $i$ with probability $1/2^i.$   
 %\begin{align*}
  %   \Prb(V_{t+s}=i|V_s=0)=\frac{1-\exp(- 2t)}{2^{i+1}}+\exp(- 2t)\boldsymbol{1}_{i=0}\\
   %  \textit{and for}\ i\ne j,\\
    %  \Prb(V_{t+s}=i|V_s=j)=\frac{1-2\exp(- t)+\exp(-2t)}{2^{i+1}}+\frac{\exp(- t)-\exp(-2t)}{2^i}\boldsymbol{1}_{i=0}\\ 
     % \Prb(V_{t+s}=i|V_s=i)=\exp(-t)+.
 %\end{align*} 
It is easy to verify that the invariant measure $\pi$ of $V_t$ is  $\pi(\{i\})=1/2^{i+1}$. One may consider $\mathbb{N}$ as one state and view $V_t$ as a Markov process with two states.

To verify that $(V_t)_{t \geq 0}$ satisfies the rest of Assumption \ref{as1.0}, similarly to the previous example, we consider a stationary version $(\hat V_t)_{t\ge 0}$ that is independent of $(V_t)_{t\ge 0}$. 
Let $V_0=0$ and $T=\inf\{t\ge 0, V_t=\hat V_t=0\}$. For $i,\ j\in \mathbb{N}_0$, we denote the $i$-th and $j$-th jump time of $(V_t)_{t\ge 0}$ and $(\hat V_t)_{t\ge 0}$ by $T_i$ and $\hat T_j$, respectively. Then we have
$$
\Prb(T=0)=\Prb(V_0=0,\hat V_0=0)=\Prb(V_0=0)\Prb(\hat V_0=0)=\frac{1}{2}.
$$
For any $i,\ j\in \mathbb{N}_0$, since $T_i$ and $\hat T_j$ are independent, $\Prb(T_i=\hat T_j)=0$. Let $Y_t=(V_t, \hat V_t)$ be a Markov process on $\mathbb{N}_0\times \mathbb{N}_0$.
Notice that $T$ is the first time when $Y_t$ hits $(0,0).$ Let the n-th jump time of $Y_t$ be $\tau_n$, for $t>0$, we have
\begin{align*}
    \Prb(T\ge t)=& \sum_{n\ge 1}\Prb(T=\tau_n,\tau_n\ge t)\\
    =& \sum_{k\ge 0}\Prb(T=\tau_{2k+1},\tau_{2k+1}\ge t)\\
    =& \sum_{k\ge 0}\exp({-2(2k+1) t})\frac{1}{2^{k+2}}
    \le \exp({-t}),
\end{align*}
where the second and the third equality follows from $\Prb(T=\tau_{2k+1})= 2^{-k-1} $ and    $\Prb(T=\tau_{2k})=0$ for $k\ge 1$. Since $\inf\{t\ge 0, V_t=\hat V_t\}$ is upper bounded by $T$, $(V_t)_{t \geq 0}$ satisfies Assumption~\ref{as1.0}.

\subsection{Example 4: Multidimensional processes}
For multidimensional processes, Assumption \ref{as1.0} is satisfied if each component satisfies Assumption \ref{as1.0} and all components are mutually independent. We illustrate this by discussing the 2-dimensional case -- higher-dimensional processes can be constructed inductively. Multidimensional processes arise, e.g., when the underlying space $S$ is multidimensional. They also arise when the $S$ is one-dimensional, but we run multiple processes in parallel to obtain a mini-batch SGD instead of single-draw SGD.

Let $(S^1,m^1)$ and $(S^2,m^2)$ be two compact Polish spaces. We consider the probability triples $(\Omega^1, (\mathcal{F}^1_t))_{t\ge 0},(\Prb^1_a)_{a\in S^1})$ and $(\Omega^2, (\mathcal{F}^2_t))_{t\ge 0}, (\Prb^2_b)_{b\in  S^2})$ with $\Prb^1_a((V^1_0=a)=\Prb^2_b(V^2_0=b)=1$. Let  $(V^1_t)_{t\ge 0}$ and $(V^2_t)_{t\ge 0}$ be  $(\mathcal{F}^1_t))_{t\ge 0}$ and
$(\mathcal{F}^2_t))_{t\ge 0}$ adapted and from $\Omega^1$ to $S^1$ and  $\Omega^2$ to $S^2$ respectively. 
In the following proposition, we construct a 2-dimensional process $(V^1_t,V^2_t)_{t\ge 0}$ from $\Omega^1\times \Omega^2$  to $(S^1\times S^2,m^1+m^2)$ with a family of probability measures $(\Prb_{(a,b)})_{(a,b)\in S^1\times S^2}$ such that $\Prb_{(a,b)}(A\times B)=\Prb^1_a(A)\Prb^2_b(B)$ for $A\in \mathcal{F}^1$ and $B\in \mathcal{F}^2$. 

We now show that the joint process $(V^1_t,V^2_t)_{t\ge 0}$ is Feller and satisfies Assumption \ref{as1.0}, if the marginals do.
\begin{prop}
  Let $(V^1_t)_{t\ge 0}$ and $(V^2_t)_{t\ge 0}$ be c\`adl\`ag and  Feller with respect to  $(\mathcal{F}^1_t)_{t\ge 0}$ and  $(\mathcal{F}^2_t)_{t\ge 0},$ respectively, and satisfy Assumption \ref{as1.0} with probability $(\Prb^1_a)_{a\in  S^2}$ and $(\Prb^2_b)_{b\in  S^2}$, respectively.  Then  $(V^1_t,V^2_t)_{t\ge 0}$ is also c\`adl\`ag and  Feller with respect to  $\sigma(\mathcal{F}^1_t\times\mathcal{F}^2_t)_{t\ge 0}$ and satisfies Assumption \ref{as1.0} with $(\Prb_{(a,b)})_{(a,b)\in S^1\times S^2}$. 
\end{prop}
\begin{proof}
It is obvious that the process $(V^1_t,V^2_t)_{t\ge 0}$ is c\`adl\`ag and Markovian. To verify the Feller property, we show that for any continuous function $F$ on $S^1\times S^2$, $\E_{(x,y)}[F(V^1_t,V^2_t)]$ is continuous in $(x, y)$. We shall prove this by showing this property for separable $F$ and approximate general continuous functions using this special case. Let $f$ and $g$ be continuous functions on $S^1$ and $S^2$ respectively, then we have
\begin{align}\label{app1}
    \E_{(x,y)}[f(V^1_t)g(V^2_t)]=\E^1_x[f(V^1_t)]\E^2_y[g(V^2_t)],
\end{align}
which implies $\E_{(x,y)}[f(V^1_t)g(V^2_t)]$ is continuous in $(x,y)$ since $(V_t^1)_{t\ge 0}$ and $(V_t^2)_{t\ge 0}$ are Feller. By the Stone–Weierstrass theorem, for any $k\ge1$, any continuous function $F$ on $S^1\times S^2$ can be approximated as the following, 
\begin{align*}
    \sup_{(x,y)\in S^1\times S^2}\abs{F(x,y)-\sum_{i=1}^{n_k}f_i^k(x)g_i^k(y)}\le \frac{1}{k}
\end{align*}
where $f_i^k$ and $g_i^k$ are continuous. 
%we have for any continuous function $F$ on $S^1\times S^2,$ there exists a sequence of finite sum of type $fg$ convergences to $F$ uniformly in $S^1\times S^2$, it is there exist $\sum_{i=1}^{n_k}f_i^kg_i^k$) such that for any $k\ge 1,$
From (\ref{app1}), this implies $\E_{(x,y)}[F(V^1_t,V^2_t)]$ is continuous on $S^1\times S^2$.

Next, we prove that $(V^1_t,V^2_t)_{t\ge 0}$ satisfies Assumption \ref{as1.0}.
Let $\pi^1$ and $\pi^2$ be the invariant measures of $(V^1_t)_{t\ge 0}$ and $(V_t^2)_{t\ge 0}$, respectively. Then $\pi^1\times \pi^2$ is the invariant measure of $(V^1_t,V^2_t)_{t\ge 0}$ since $(V^1_t)_{t\ge 0}$ and $(V^2_t)_{t\ge 0}$ are independent. 
From Assumption  \ref{as1.0}, we know there exist $(\tilde\Omega^1, \mathcal{\tilde F}^1,\tP^1),$ $(\tilde\Omega^2, \mathcal{\tilde F}^2,\tP^2),$  such that  for any $a\in S^1$ and $b\in S^2$ , $(V^{1,a}_t)_{t\ge 0}\eq (V^1_t)_{t\ge 0}$ in $\Prb^1_a$ and $(V^{2,b}_t)_{t\ge 0}\eq (V^2_t)_{t\ge 0}$ in $\Prb^2_b$. We define $\tP$ on $\tilde\Omega^1\times \tilde\Omega^2$ such that
$$
\tP(A\times B)=\tP^1(A)\tP^2(B), \ \  (A\in \mathcal{\tilde F}^1, \ B\in \mathcal{\tilde F}^2).
$$
Then we have that $((V^{1,a}_t)_{t\ge 0})_{a\in S^1}$ and $(V^{1,\pi^1}_t)_{t\ge 0}$ are independent of $((V^{2,b}_t)_{t\ge 0})_{b\in S^2}$ and $(V^{2,\pi^2}_t)_{t\ge 0}$ under $\tP.$ 
Similar to the proof of Lemma \ref{lemma1}, we construct the following processes by the coupling method:
\begin{align*}
    \tilde V^{1,a}_t=\left\{
\begin{aligned}
  &V^{1,a}_t,\ \ \ 0\le t\le T^{1,a},\\
 &V^{1,\pi^1}_t,\ \ \  t> T^{1,a},
\end{aligned}
\right.
\end{align*}
and 
\begin{align*}
    \tilde V^{2,b}_t=\left\{
\begin{aligned}
  &V^{2,b}_t,\ \ \ 0\le t\le T^{2,b},\\
 &V^{2,\pi^2}_t,\ \ \  t> T^{2,b}.
\end{aligned}
\right.
\end{align*}
Then the distribution of $(\tilde V^{1,a}_t,\tilde V^{2,b}_t)_{t\ge 0}$ under $\tP$ is the same as the distribution of  $(V^1_t,V^2_t)_{t\ge 0}$ under $\Prb_{(a,b)};$
the distribution of $(V^{1,\pi^1}_t, V^{2,\pi^2}_t)_{t\ge 0}$ under $\tP$ is the same as the distribution of $(V^1_t,V^2_t)_{t\ge 0}$ under $\Prb_{\pi^1\times\pi^2}.$ Moreover, $(\tilde V^{1,a}_t,\tilde V^{2,b}_t)_{t\ge 0}$ intersects the invariant state  $(V^{1,\pi^1}_t, V^{2,\pi^2}_t)_{t\ge 0}$ at time $T^{1,a}\lor T^{1,b}.$ For any $(a,b)\in S^1\times S^2$,
$$
\tP(T^{1,a}\lor T^{1,b}\ge t)\le\tP(T^{1,a}\ge t)+ \tP(T^{1,b}\ge t)\le C\exp({-\delta t}).
$$
\end{proof}
We have now discussed various properties of potential index processes and move on to study the stochastic gradient process.

\section{Stochastic gradient processes with constant learning rate} \label{Sec_SGPC}

We now define and study the stochastic gradient process with constant learning rate. 
Here, the switching between data sets is performed in a homogeneous-in-time way. Hence, it models the discrete-time stochastic gradient descent algorithm when employed with a constant learning rate. Although, one can usually not hope to converge to the minimizer of the target functional in this case, this setting is popular in practice.

To obtain the stochastic gradient process with constant learning rate, we will couple the gradient flow \eqref{Eq_SGPC} with the an appropriate process $(V_{t/\e})_{t\ge0}$. Here, $(V_t)_{t\ge0}$ is a Feller process introduced in Section \ref{mainprocess} and $\e > 0$ is a scaling parameter that allows us to uniformly control a switching rate parameter. To define the stochastic process associated with this stochastic gradient descent problem, we first introduce the following assumptions that guarantee the existence and uniqueness of the solution of the associated stochastic differential equation. After its formal definition and the proof of well-definedness, we move on to the analysis of the process. Indeed, we show that the process approximates the full gradient flow \eqref{eq:AS:th}, as $\e \downarrow 0$. Moreover, we show that the process has a unique stationary measure to which it converges in the longtime limit at geometric speed. 

We commence with regularity properties of the subsampled target function $f$ that are necessary to show the well-definedness of the stochastic gradient process.
\begin{assumption}\label{asSGPf}
Let $f(\theta, y)\in \mathcal{C}^2(\R^K\times S,\R)$.

1. $\nabla_{\theta} f,$ $H_\theta f$ are continuous. 

2. $\nabla_{\theta} f(\theta,y)$ is Lipschitz in $x$ and the Lipschitz constant is uniform for $y\in S.$   

3. For $\theta\in \mathbb{R}^K$, $f(\theta,\cdot)$ and $\nabla_{\theta} f$ are integrable w.r.t to the probability measure $\pi(\cdot)$.

\end{assumption}
Now, we move on to the formal definition of the stochastic gradient process.

\begin{defi}
For $\e>0$, the \emph{stochastic gradient process with} \emph{constant learning rate} \emph{(SGPC)} is a solution of the following stochastic differential equation,
\begin{equation}\label{eq:AS:theta}
\left\{ \begin{array}{l}
\mathrm{d}\theta^\e_t = - \nabla_{\theta} f(\theta^\e_t, V_{ t/\e})\mathrm{d}t, \\
\theta_0^\e = \theta_0,
\end{array} \right.
\end{equation}
where f satisfies Assumption \ref{asSGPf} and $(V_t)_{t\ge0}$ is a Feller process that satisfies Assumption \ref{as1.0}.
\end{defi}
Given these two assumptions, we can indeed show that the SGPC is a well-defined Markov process. Moreover, we show that the stochastic gradient process is Markovian, a property it shares with the discrete-time stochastic gradient descent method.

\begin{prop}\label{thwk30} Let Assumptions \ref{as1.0} and \ref{asSGPf} hold. Then, equation (\ref{eq:AS:theta}) has a unique  strong solution, i.e. the solution $(\theta^\e_t)_{t \geq 0}$ is measurable with respect to $\mathcal{F}^\e_t:=\mathcal{F}_{t/\e}$ for any $t\ge 0$. For $y\in S$,  $(\theta_t^\e,V_{t/\e})_{t\ge 0} $ is a Markov process under $\Prb_y$ with respect to $(\mathcal{F}^\e_t)_{t\ge 0}$.
  %with semi-group $(Q^\e_t)_{t\ge 0}$, where   $(Q^\e_t)_{t\ge 0}$ is defined by for any function $h$ bounded measurable on $\R^K\times S$
 %$$
  %Q^\e_t h(x,y):=\E_y[h(\theta^\e_t,V_\frac{t}{\e})|\theta^\e_0=x ]
 %$$
 \end{prop} 
 
 \begin{proof}
The existence and the uniqueness of the strong solution to the equation (\ref{eq:AS:theta}) can be found in \citet[Chapter 2, Theorem 4.1]{Kushner1}.
To prove the Markov property, we define the operator $(Q^\e_t)_{t\ge 0}$ such that 
 $$
  Q^\e_t h(x,y):=\E_y[h(\theta^\e_t,V_{t/\e})|\theta^\e_0=x ],
 $$
 for any function $h$ bounded and measurable on $\R^K\times S$. For any $ s,t\ge 0$, we want to show 
 \begin{align*}
     \E[h(\theta^\e_{t+s},V_{(t+s)/\e})|\mathcal{F}^\e_s]=Q^\e_t h(\theta^\e_s,V_{s/\e}).
 \end{align*}
 We set
$ \hat \theta^\e_t:=\theta^\e_{t+s},\ \mathcal{\hat F}_t:=\mathcal{F}^\e_{t+s},\ \hat V_{t/\e}:=V_{(t+s)/\e}.$
 Since
 \begin{align*}
     \theta^\e_{t+s}=\theta^\e_s-\int_s^{t+s}\nabla_{\theta} f(\theta^\e_m, V_{m/\e}) \mathrm{d}m,
 \end{align*}
we have
 \begin{align*}
     \hat\theta^\e_t=\hat\theta^\e_0-\int_0^t\nabla_{\theta} f(\hat\theta^\e_m, \hat V_{m/\e}) \mathrm{d}m.
 \end{align*}
Hence $\hat\theta^\e_t$ is the solution of equation (\ref{eq:AS:theta}) with $\hat\theta^\e_0=\theta^\e_s$ and $ \hat V_0=V_{ s/ \e}.$ %{\color{orange} And by the Markov property of $(V^\e_t)_{t\ge 0}$,  $\hat V^\e_t$ can be seen as a Markov process has the same transition probability with $V^\e_t$ with the initial value $\hat V_0$. And $\hat\theta^\e_t$ can also be seen as the solution of (\ref{eq:AS:theta}) with initial value   $\hat \theta^\e_0$ and $\hat V_0.$ And since $\theta^\e_{t+s}$ is $\mathcal{F}^\e_s$ adapted, under the condition $\mathcal{F}^\e_s,$ $(\theta^\e_{t+s},V_{(t+s)/\e})$ can be determined by $(\theta^\e_s,V_{s/\e})$}.   
Moreover,
\begin{align*}
     \E[h(\theta^\e_{t+s},V_{(t+s)/\e})|\mathcal{F}^\e_s]&=\E[h(\hat \theta^\e_t,\hat V_{t/\e})|\hat \theta^\e_0=\theta^\e_s,\hat V_0=V_{s/\e}]\\
     &=\E_{\hat V_0}[h(\hat \theta^\e_t,\hat V_{t/\e})|\hat \theta^\e_0=\theta^\e_s]\\
     &= Q^\e_t h(\theta^\e_s,V_{s/\e}), 
\end{align*}
where the second equality and third equality follow from the homogeneous Markov property of $(V^\e_t)_{t\ge 0}$.

%{\bf Feller property}
%We are going to verify $Q_t h(x,y)$ is continuous in $\R^K\times S$ while $h$ is continuous in $\R^K\times S.$
 \end{proof} 
 \subsection{Approximation of the full gradient flow} \label{Subsec:ApproxFullGradientFlow}
 We now let $\e\to0$ and study the limiting behavior of SGPC.  Indeed, we aim to show that here the SGPC converges to the \emph{full gradient flow} 
\begin{equation}\label{eq:AS:th}
  \mathrm{d}\zeta_t = -\Big[\int_{S} \nabla_{\zeta} f(\zeta_t, v)\pi(\mathrm{d}v)\Big]\mathrm{d}t.
\end{equation}
We study this topic for two reasons: First, we aim to understand the interdependence of $(V_t)_{t\ge0}$ and $(\theta_t^{\e})_{t\ge0}$. Second, we understand SGPC as an approximation to the full gradient flow \eqref{eq:AS:th}, as motivated in the introduction. Hence, we should show that SGPC can approximate the full gradient flow at any accuracy.

We now denote $g(\cdot):=\int_{S} \nabla_{\zeta} f(\cdot, v)\pi(\mathrm{d}v)\in \mathcal{C}^1(\R^K,\R^K).$ Then, we can define $(\zeta_t)_{t\ge 0}$ through the dynamical system $\mathrm{d}\zeta_t = - g(\zeta_t)\mathrm{d}t$. Moreover, let $\mathcal{C}([0,\infty):\R^K)$ be the space of continuous functions from $[0,\infty)$ to $\R^K$ equipped with the distance
$$
\rho\Big((\varphi_t)_{t\ge 0},(\varphi'_t)_{t\ge 0}\Big):= \int_0^\infty \exp({-t}) (1\land\sup_{0\le s\le t}\norm{\varphi_s-\varphi_s'})\mathrm{d}t,
$$ 
where $(\varphi_t)_{t\ge 0},(\varphi'_t)_{t\ge 0} \in \mathcal{C}([0,\infty):\R^K)$. {We study the weak limit of the system (\ref{eq:AS:theta}) as $\e\to0$. Similar problems have been discussed in, for example, \citet{Kushner1} and \citet{Kushner2}.}

\begin{theorem}\label{wcovtheta}
 Let $\theta^\e_0=\theta_0$ and $\zeta_0=\theta_0$. Moreover, let $(\theta^\e_t)_{t\ge 0}$ and $(\zeta_t)_{t\ge 0}$ solve (\ref{eq:AS:theta}) and (\ref{eq:AS:th}), respectively. Then $(\theta^\e_t)_{t\ge 0}$ under $\Prb_\pi$ converges weakly to $(\zeta_t)_{t\ge 0}$ in $\mathcal{C}([0,\infty):\R^K)$ as $\e\to 0$, i.e.
 for any bounded continuous function $F$ on $\mathcal{C}([0,\infty):\R^K)$,
 $$\E_\pi [F\big((\theta^\e_t)_{t\ge 0}\big)] \to \E_\pi [F\big((\zeta_t)_{t\ge 0}\big)]=F\big((\zeta_t)_{t\ge 0}\big).$$
\end{theorem}
\begin{proof}
We shall verify that $(\theta^\e_t)_{t\ge 0}$ is tight by checking: 
\begin{align*}
    1.&\ \sup_{0<\e<1}\norm{\theta^\e_0}<+\infty;\\
    2.&\ \textit{For any fixed } T>0,\  \lim_{\delta\to 0} \sup_{0<\e<1}\sup_{s,t\in[0,T],|s-t|\le \delta}\norm{\theta^\e_t-\theta^\e_s}\to 0.
\end{align*}
The first condition follows from $\theta^\e_0=\theta_0$. For the second condition, by Assumption \ref{asSGPf}, let $C_0=\sup_{y\in S}\norm{\nabla_{\theta} f(0, y)}$ and $L_f$ be the Lipschitz constant of $\nabla_{\theta} f(\cdot, y)$, we have
\begin{align*}
 \frac{\mathrm{d}\norm{\theta^\e_t}^2}{\mathrm{d}t}
    &=-2 \ip{\theta^\e_t,\nabla_{\theta} f(\theta^\e_t, V_{t/\e})}\\
    &=-2 \ip{\theta^\e_t,\nabla_{\theta} f(\theta^\e_t, V_{t/\e})-\nabla_{\theta} f(0, V_{t/\e})}-2\ip{\theta^\e_t,\nabla_{\theta} f(0, V_{t/\e})}\\
    &\le  2L_f \norm{\theta^\e_t}^2+ 2C_0\norm{\theta^\e_t}\\
    &\le 2L_f \norm{\theta^\e_t}^2+ \norm{\theta^\e_t}^2+C^2_0\\
    &= (2L_f+1) \norm{\theta^\e_t}^2+C^2_0.
\end{align*}
By Gr\"onwall's inequality,
\begin{align}
\norm{\theta^\e_t}^2\le (\norm{\theta_0}^2+C^2_0)e^{(2L_f+1) t}.
\end{align}
Therefore, $\theta^\e_t$ is bounded on any finite time interval.
For any fixed $T>0,$ let $$C_{T,f,\theta_0}=\sup_{\norm{x}\le (\norm{\theta_0}^2+C^2_0)e^{(2L_f+1) T},\ y\in S }\norm{\nabla_{\theta} f(x, y)}.$$
Then for any $s,t\in[0,T]$,
\begin{align*}
   \norm{\theta^\e_t-\theta^\e_s}\le  \int_s^t\norm{\nabla_{\theta} f(\theta^\e_m, V_{m/\e})} \mathrm{d}m \le C_{T,f,\theta_0} |t-s|. 
\end{align*}
Hence, $(\theta^\e_t)_{t\ge 0}$ is tight in $\mathcal{C}([0,\infty):\R^K)$. By Prokhorov's theorem, let $(\theta_t)_{t\ge 0}$ be a weak limit of $(\theta^\e_t)_{t\ge 0}$. We shall verify that $(\theta_t)_{t\ge 0}$ satisfies equation (\ref{eq:AS:th}), which is equivalent to show that for any bounded differentiable function $\varphi$, $h$ 
\begin{align*}
    \E_\pi\Big[\Big(\varphi(\theta_t)- \varphi(\theta_s)+\int_s^t\ip{\nabla_{\theta}\varphi(\theta_m),g(\theta_m)}\mathrm{d}m\Big)h\Big((\theta_{t_i})_{i=1,...,n}\Big)\Big]=0,
\end{align*}
$\forall\ 0\leq t_1 < \cdots < t_n\leq s.$ The case $t=0$ is obvious. Since $(\theta^\e_t)_{t\ge 0}$ is a strong solution to equation (\ref{eq:AS:theta}), for any $0\le s<t$,
\begin{align}\label{net}
    \varphi(\theta^\e_t)= \varphi(\theta^\e_s)-\int_s^t\ip{\nabla_{\theta}\varphi(\theta^\e_m),\nabla_{\theta} f(\theta^\e_m, V_{m/\e})}\mathrm{d}m.
\end{align}
Hence, we have
\begin{align*}
    \E_\pi\Big[\Big(\varphi(\theta^\e_t)- \varphi(\theta^\e_s)+\int_s^t\ip{\nabla_{\theta}\varphi(\theta^\e_s),\nabla_{\theta} f(\theta^\e_m, V_{m/\e})}\mathrm{d}m\Big)h\Big((\theta^\e_{t_i})_{i=1,...,n}\Big)\Big]=0,
\end{align*}
Moreover, when $\e\to 0,$
\begin{align*}
     \E_\pi\Big[\Big(\varphi(\theta^\e_t)- \varphi(\theta^\e_s)\Big)h\Big((\theta^\e_{t_i})_{i=1,...,n}\Big)\Big]\to \E_\pi\Big[\Big(\varphi(\theta_t)- \varphi(\theta_s)\Big)h\Big((\theta_{t_i})_{i=1,...,n}\Big)\Big].
\end{align*}
Hence, all we need to show is the following
\begin{align}
    \E_\pi\Big[\Big(\int_s^t\ip{\nabla_{\theta}\varphi(\theta^\e_m),\nabla_{\theta} f(\theta^\e_m, V_{m/\e})}\mathrm{d}m-\int_s^t\ip{\nabla_{\theta}\varphi(\theta^\e_m),g(\theta^\e_m)}\mathrm{d}m\Big)h\Big((\theta^\e_{t_i})_{i=1,...,n}\Big)\Big]\to 0,
\end{align}
which is equivalent to prove that 
\begin{align}
    \E_\pi\Big[\int_s^t\ip{\nabla_{\theta}\varphi(\theta^\e_m),\nabla_{\theta} f(\theta^\e_m, V_{m/\e})}\mathrm{d}m-\int_s^t\ip{\nabla_{\theta}\varphi(\theta^\e_m),g(\theta^\e_m)}\mathrm{d}m\Big|\mathcal{F}^\e_s\Big]\to 0. \label{important}
\end{align}
Let $\te:= 1/[1/\sqrt{\e}]$, where $[x]$ is the greatest integer less than or equal to $x$.  Then we have the following decomposition
\begin{align*}
    &\E_\pi\Big[\int_s^t\ip{\nabla_{\theta}\varphi(\theta^\e_m),\nabla_{\theta} f(\theta^\e_m, V_{m/\e})}\mathrm{d}m-\int_s^t\ip{\nabla_{\theta}\varphi(\theta^\e_m),g(\theta^\e_m)}\mathrm{d}m\Big|\mathcal{F}^\e_s\Big]\\
=& \te\sum^{1/\te}_{i=0}\te^{-1}\E_\pi\Big[\int_{s+i(t-s)\te}^{s+(i+1)(t-s)\te}\ip{\nabla_{\theta}\varphi(\theta^\e_m),\nabla_{\theta} f(\theta^\e_m, V_{m/\e})-g(\theta^\e_m)}\mathrm{d}m\Big|\mathcal{F}^\e_s\Big]\\
=&\te\sum^{1/\te}_{i=0}\E_\pi\Big[\te^{-1}\E_\pi\Big[\int_{s+i(t-s)\te}^{s+(i+1)(t-s)\te}\ip{\nabla_{\theta}\varphi(\theta^\e_m),\nabla_{\theta} f(\theta^\e_m, V_{m/\e})-g(\theta^\e_m)}\mathrm{d}m\Big|\mathcal{F}^\e_{s+i(t-s)\te}\Big]\Big|\mathcal{F}^\e_s\Big].
\end{align*}
We claim that as $\e\to 0$,
\begin{align}\label{correct}
   \sup_{0\le r<t} \te^{-1}\E_\pi\Big[\int_r^{r+(t-s)\te}G(\theta^\e_m,V_{m/\e})\mathrm{d}m\Big|\mathcal{F}^\e_r\Big]\to 0,
\end{align}
where $G(x,y):=\ip{\nabla_{\theta}\varphi(x),\nabla_{\theta} f(x, y)-g(x)}$.  
Notice that for any fixed $t>0$, $(\theta^\e_s)_{ 0\le s\le t}$ is uniformly equicontinuous. Hence, we have
\begin{align*}
   \sup_{0\le r\le t} \sup_{r\le m\le r+(t-s)\te}\norm{\theta^\e_m-\theta^\e_r}&= \sup_{0\le r\le t} \sup_{r\le m\le r+(t-s)\te}\int_r^{r+(t-s)\te}\norm{\nabla_{\theta} f(\theta^\e_m, V_{m/\e})} \mathrm{d}m\\
   &\le \te \sup_{0\le m\le t} \norm{\nabla_{\theta} f(\theta^\e_m, V_{m/\e})}.
\end{align*}
Therefore, as $\e \to 0,$
\begin{align*}
   \sup_{0\le r<t} \abs{ \te^{-1}\E_\pi\Big[\int_r^{r+(t-s)\te}G(\theta^\e_m,V_{m/\e})\mathrm{d}m\Big|\mathcal{F}^\e_r\Big]-\te^{-1}\E_\pi\Big[\int_r^{r+(t-s)\te}G(\theta^\e_r,V_{m/\e})\mathrm{d}m\Big|\mathcal{F}^\e_r\Big]}\to 0.
\end{align*}
Hence, (\ref{correct}) is equivalent to
\begin{align}\label{alsocorrect}
    \sup_{0\le r<t} \abs{ \te^{-1}\E_\pi\Big[\int_r^{r+(t-s)\te}G(\theta^\e_r,V_{m/\e})\mathrm{d}m\Big|\mathcal{F}^\e_r\Big]}\to 0.
\end{align}
By Corollary \ref{corergodic},  
\begin{align*}
    &\sup_{0\le r<t} \abs{ \te^{-1}\E_\pi\Big[\int_r^{r+(t-s)\te}G(\theta^\e_r,V_{m/\e})\mathrm{d}m\Big|\mathcal{F}^\e_r\Big]}\\
    =& \sup_{0\le r<t} \abs{ \te^{-1}\E_{V_{r/\e},x=\theta^\e_r}\Big[\int_r^{r+(t-s)\te}G(x,V_{(m-r)/\e})\mathrm{d}m\Big]}\\
    =& \sup_{0\le r<t} \abs{ \te^{-1}\E_{V_{r/\e},x=\theta^\e_r}\Big[\int_r^{r+(t-s)\te}\ip{\nabla_{\theta}\varphi(x),\nabla_{\theta} f(x, V_{(m-r)/\e})-g(x)}\mathrm{d}m\Big]}\\
     =& \sup_{0\le r<t} \abs{ \te^{-1}\int_r^{r+(t-s)\te}\E_{V_{r/\e},x=\theta^\e_r}\Big[\ip{\nabla_{\theta}\varphi(x),\nabla_{\theta} f(x, V_{(m-r)/\e})-g(x)}\Big]\mathrm{d}m}\\
     \le& \sup_{0\le r<t}  \te^{-1}\int_r^{r+(t-s)\te}\abs{\E_{V_{r/\e},x=\theta^\e_r}\Big[\ip{\nabla_{\theta}\varphi(x),\nabla_{\theta} f(x, V_{(m-r)/\e})-g(x)}\Big]}\mathrm{d}m\\
     \le & \sup_{0\le r<t}\te^{-1}\norm{\ip{\nabla_{\theta}\varphi(\theta^\e_r),\nabla_{\theta} f(\theta^\e_r, \cdot)}}_\infty \int_r^{r+(t-s)\te} e^{-\delta (m-r)/\e}dm\\
    =& \sup_{0\le r<t}\te^{-1}\norm{\ip{\nabla_{\theta}\varphi(\theta^\e_r),\nabla_{\theta} f(\theta^\e_r, \cdot)}}_\infty \int_0^{(t-s)\te} e^{-\delta k/\e}dk\\
    \le & C_{t,\varphi,f}\frac{\e}{\delta\te}\le C_{t,\varphi,f}\frac{\sqrt{\e}}{2\delta}\to0.
\end{align*}
This completes the proof of (\ref{important}).
Hence, any weak limit of $(\theta^\e_t)_{t\ge 0}$ is a martingale solution to equation (\ref{eq:AS:th}). Since equation (\ref{eq:AS:th}) is a deterministic ordinary differential equation and $\theta^\e_0=\theta_0$ is independent of $\e$, we have
 $(\theta^\e_t)_{t\ge 0}$ converges weakly to $(\zeta_t)_{t\ge 0}$ as $\e\to 0$.

\end{proof}

Instead of looking at the full trajectories of the processes, we can also study their distributions and show convergence in the Wasserstein distance. We first need to introduce some notation.

Let $\nu$ and $\nu'$ be two probability measures on $(\R^K, \mathcal{B}(\R^K) )$. We define the Wasserstein distance between those measures by
\[
   \cW_d(\nu,\nu') =  \inf_{\Gamma\in\mathcal{H}(\nu,\nu') }\int_{\R^K\times\R^K}d(y,y')\Gamma(\mathrm{d}y,\mathrm{d}y'),
\]
where $d(y,y'):=1\land\norm{y-y'}$ and $\mathcal{H}(\nu,\nu')$ is the set of coupling between $\nu$ and $\nu'$, i.e.
$$\mathcal{H}(\nu,\nu') = \{ 
    \Gamma \in \text{Pr} (\R^K\times\R^K) : 
    \Gamma(A \times \R^K) = \nu(A),  
    \Gamma(\R^K \times B) = \nu'(B), \forall 
    A, B \in \mathcal{B}(\R^K)
\}.$$
To simplify the notation, for $B\in \mathcal{B}(\R^K)$, $\theta\in \R^K$, and $y\in S$, we denote 
 $$C^\e_t(B|\theta, y):=\Prb_y(\theta^\e_t\in B|\theta^\e_0=\theta),$$  $$C^\e_t(B|\theta,\pi):=\Prb_\pi(\theta^\e_t\in B|\theta^\e_0=\theta),$$ 
 where $\pi$ is the invariant measure of $(V_t)_{t \geq 0}$. 
 
Now we study the approximation property of SGPC in the Wasserstein distance. Indeed, the following corollary follows immediately from Theorem \ref{wcovtheta}.

\begin {corollary}\label{deuxthe}
There exists a function $\alpha: (0,1) \to [0,1],$ such that
 $$
 \cW_d(C^\e_t(\cdot|\theta_0,\pi),\delta(\cdot-\zeta_t))\le (\exp(t)\alpha(\e))\land 1.
 $$
Moreover, $\lim_{\e\to 0}\alpha(\e)=0$.
 
\end {corollary}
\begin{proof}
By Theorem \ref{wcovtheta}, we have $(\theta^\e_t)_{t\ge 0}\Rightarrow(\zeta_t)_{t\ge 0}$. By  Skorokhod's representation theorem, there exists a sequence  $(\tilde\theta^\e_t)_{t\ge 0}$ such that
\begin{align*}
    &(\tilde\theta^\e_t)_{t\ge 0}\overset{d}{=}(\theta^\e_t)_{t\ge 0}\ \text{under}\ \Prb_\pi,\\
    &\rho\Big((\tilde\theta^\e_t-\zeta_t)_{t\ge 0},0\Big)\to 0\ \text{almost\ surely\ in}\ \Prb_\pi.
\end{align*}
This implies
\begin{align*}
    \E_\pi[F((\tilde\theta^\e_t-\zeta_t)_{t\ge 0})]\to F(0),
\end{align*}
for any bounded continuous function $F$ on $\mathcal{C}([0,\infty):\R^K)$. By taking 
$$F((\tilde\theta^\e_t-\zeta_t)_{t\ge 0}))= \sup_{t\ge 0} \exp({-t}) \left(1\land\sup_{0\le s\le t}\norm{\tilde\theta^\e_t-\zeta_t}\right)$$ 
and
\begin{align*}
    \alpha(\e):=\E_\pi\Big[\sup_{t\ge 0} \exp({-t}) \left(1\land\sup_{0\le s\le t}\norm{\tilde\theta^\e_t-\zeta_t}\right)\Big]\to 0,
\end{align*}
we have, for all $t\ge 0$, 
\begin{align*}
    \E_\pi\Big[1\land\norm{\tilde\theta^\e_t-\zeta_t}\Big]\le \exp(t) \alpha(\e).
\end{align*}
Since $1\land\norm{\tilde\theta^\e_t-\zeta_t}\le 1$ and $\tilde\theta^\e_t \overset{d}{=} \theta^\e_t$, denoting the distribution of $\tilde\theta^\e_t$ as $F_{\tilde\theta^\e_t}$,
\begin{align*}
    \cW_d(C^\e_t(\cdot|\theta_0,\pi),\delta(\cdot-\zeta_t))\le&  \cW_d(F_{\tilde\theta^\e_t}\ ,C^\e_t(\cdot|\theta_0,\pi))+\E_\pi\Big[1\land\norm{\tilde\theta^\e_t-\zeta_t}\Big] \\
    \le& (\exp(t)\alpha(\e))\land 1.
\end{align*}
\end{proof}

Finally in this section, we look at a technical result concerning the asymptotic behavior of the full gradient flow $(\zeta_t)_{t \geq 0}.$ First, we will additionally assume that the subsampled target function $f(\cdot, y)$ in the optimization problem is strongly convex, with a convexity parameter that does not depend on $y \in S$. We state this assumption below.

\begin{assumption}[Strong Convexity]\label{as1.2}
For any   $x_1,x_2\in \R^K,$
$$
\ip{x_1-x_2,\nabla_{\theta} f(x_1, y)-\nabla_{\theta} f(x_2, y)}\ge  \kappa\norm{x_1-x_2}^2
$$
where $\kappa>0$ and $\kappa$ is independent of $y\in S$. 
\end{assumption}

Strong convexity implies, of course, that the full target function $g := \int_S \nabla_{\theta}f(\cdot, y)  \pi(\mathrm{d}y)$ has a unique minimizer $\theta^*$. It also implies that the associated full gradient flow $(\zeta_t)_{t \geq 0}$ converges at exponential speed to this unique minimizer. We give a short proof of this statement below.

\begin{lemma}\label{norandomthe}
Let $(\zeta_t)_{t\ge 0}$ be the process that solves (\ref{eq:AS:th}) with initial data $\theta_0$. Under Assumption \ref{as1.2}, we have
\begin{align*}
    \norm{\zeta_t-\theta_*}^2\le \norm{\theta_0-\theta_*}^2\exp({-\kappa t}),
\end{align*}
 where $\theta_*$ is a stationary solution of (\ref{eq:AS:th}).
\end{lemma}
\begin{proof}
Since $\theta_*$ is a stationary solution, $$g(\theta_*)=0\ \ \text{and}\ \  \mathrm{d}(\zeta_t-\theta_*) = -(g(\zeta_t)-g(\theta_*))\mathrm{d}t.$$
%\begin{align*}
%    \norm{\zeta_t-\theta_*}^2=\norm{\theta_0-\theta_*}^2-2\int_0^t \ip{\theta_s-\theta_*,g(\theta_s)-g(\theta_*)} ds,
%\end{align*}
Therefore,
\begin{align*}
    \frac{\mathrm{d}\norm{\zeta_t-\theta_*}^2}{\mathrm{d}t} = 2\ip{\zeta_t-\theta_*, \frac{\mathrm{d}(\zeta_t-\theta_*)}{\mathrm{d}t}} =-2 \ip{\zeta_t-\theta_*,g(\zeta_t)-g(\theta_*)}\le -2\kappa \norm{\zeta_t-\theta_*}^2.
\end{align*}
By Gr\"onwall's inequality, 
\begin{align*}
    \norm{\zeta_t-\theta_*}^2\le \norm{\theta_0-\theta_*}^2\exp({-\kappa t}).
\end{align*}
\end{proof}

\subsection{Longtime behavior and ergodicity} We now study the longtime behavior of SGPC, i.e. the behavior and distribution of $(\theta_t^\e, V_{t/\e})$ for $t \gg 0$ large. 
Indeed, the main result of this section will be the geometric ergodicity of this coupled process and a study of its stationary measure. 
Initially, we study stability of the stochastic gradient process $(\theta^{\varepsilon}_t)_{t \geq 0}.$  
\begin{lemma}\label{boundedtheta}
Under Assumption \ref{as1.2}, we have
$$
\norm{\theta^\e_t}^2\le \norm{\theta^\e_0}^2 \exp({-\kappa t})+\frac{8K_f^2 }{\kappa^2 },
$$
where $K_f:= \sup_{y\in S} \norm{\nabla_{\theta} f(0,y)}$.
\end{lemma}
\begin{proof}
By  It\^o's formula, we have
%\begin{align*}
%    \norm{\theta^\e_t}^2 =  \norm{\theta_0^\e}^2-2\int_0^t \ip{\theta^\e_s,\nabla_{\theta} f(\theta^\e_s, V_\frac{s}{\e})} ds,
%\end{align*}
\begin{align}\label{gronvxi}
    \frac{\mathrm{d}\norm{\theta^\e_t}^2}{\mathrm{d}t} = 
    2\ip{\theta^\e_t, d\theta^\e_t/dt}
    =-2 \ip{\theta^\e_t,\nabla_{\theta} f(\theta^\e_t, V_{t/\e})} .
\end{align}
By Assumption \ref{as1.2},
\begin{align*}
    \ip{\theta^\e_t,\nabla_{\theta} f(\theta^\e_t, V_{t/\e})}=& \ip{\theta^\e_t-0,\nabla_{\theta} f(\theta^\e_t, V_{t/\e})-\nabla_{\theta} f(0, V_{t/\e})}+\ip{\theta^\e_t,\nabla_{\theta} f(0, V_{t/\e})}\\
    \ge& \kappa \norm{\theta^\e_t}^2-\norm{\theta^\e_t}\norm{\nabla_{\theta} f(0,V_{t/\e})}\\
    \ge& \frac{\kappa}{2} \norm{\theta^\e_t}^2- \frac{4}{\kappa }\norm{\nabla_{\theta} f(0, V_{t/\e})}^2\\
    \ge& \frac{\kappa}{2} \norm{\theta^\e_t}^2- \frac{4K_f^2}{\kappa }.
\end{align*}
Hence (\ref{gronvxi}) implies
\begin{align}\label{gronvxi1}
    \frac{\mathrm{d}\norm{\theta^\e_t}^2}{\mathrm{d}t} \le  -\kappa \norm{\theta^\e_t}^2+ \frac{8K_f^2}{\kappa }.
\end{align}
Multiplying $\exp({\kappa t})$ on both sides of (\ref{gronvxi1}), we get
\begin{align*}
    \frac{\mathrm{d}(\norm{\theta^\e_t}^2\exp({\kappa t}))}{\mathrm{d}t} \le  \frac{8K_f^2 \exp({\kappa t})}{\kappa },
\end{align*}
that is
\begin{align*}
    \norm{\theta^\e_t}^2\exp({\kappa t})-\norm{\theta^\e_0}^2\le  \frac{8K_f^2 (\exp({\kappa t})-1)}{\kappa^2 }\le \frac{8K_f^2 \exp({\kappa t})}{\kappa^2 }.
\end{align*}
Therefore,
$$
\norm{\theta^\e_t}^2\le \norm{\theta^\e_0}^2 \exp({-\kappa t})+\frac{8K_f^2 }{\kappa^2 }.
$$
\end{proof}
Using this lemma, we are now able to prove the first main result of this section, showing geometric ergodicity of $(\theta^\e_t,V_{t/\e})_{t \geq 0}.$ First, we introduce a Wasserstein distance, on the space on which $(\theta^\e_t,V_{t/\e})_{t \geq 0}$ lives.
Let $\Pi$ and $\Pi'$ be two probability measures on $(\R^K\times S, \mathcal{B}(\R^K\times S) )$. We define the Wasserstein distance between those measures by
\[
   \widetilde\cW_{\tilde d}(\Pi,\Pi') =  \inf_{\widetilde{\Gamma}\in\mathcal{H}(\Pi,\Pi') }\int_{(\R^K\times S)\times(\R^K\times S)}\tilde d((u,v),(u',v'))\widetilde{\Gamma}(\mathrm{d}u\mathrm{d}v,\mathrm{d}u'\mathrm{d}v'),
\]
where $\tilde d((u,v),(u',v')):=\bll_{v\ne v'}+(1\land\norm{u-u'})\bll_{v=v'}$. For $a\in S$ and $m\in\R^K$, let $H^\e_t(\cdot|m,a)$ be the distribution of $(\theta^\e_t,V_{t/\e})$ under $\Prb_a$ with $\theta^\e_0=m$.
Moreover, recall that $(V_t)_{t\ge 0}$ is a Feller process that satisfies Assumption \ref{as1.0}. More specifically, it satisfies Assumption \ref{as1.0} (iii) with a constant $\delta$:
$$
\sup_{x\in S}\tilde\Prb(T^x\ge t)\le C\exp({-\delta t}),
$$
where $T^x:=\inf{\{t\ge 0\ |\ V^x_t=V^\pi_t \}}$. With the constant $\delta$ defined this way, we have the following theorem.
\begin{theorem}\label{dispitheta}
 Under Assumption \ref{as1.2}, for any $0< \varepsilon\le 1\land ({\delta}/{2\kappa})$, the (coupled) process $(\theta_t^\e,V_{t/\e})_{t\ge 0} $ admits an unique  stationary
measure $\Pi^\e$ on $(\R^K\times S, \mathcal{B}(\R^K\times S)).$ Moreover, 
\begin{align}\label{ineq2.6}
    \widetilde\cW^2_{\tilde d}(H^\e_t(\cdot|m,a),\Pi^\e)\le C_f\exp({-\kappa t})\int_{\R^K}(1+ \norm{x-m}^2)\Pi^\e(\mathrm{d}x,S),
\end{align}
\begin{align}\label{ineq:w_pi}
    \widetilde\cW^2_{\tilde d}(H^\e_t(\cdot|m,\pi),\Pi^\e)\le C_f\exp({-\kappa t})\int_{\R^K} \norm{x-m}^2\Pi^\e(\mathrm{d}x,S),
\end{align}
where the constant $C_f$ only depends on $f$. 
\end{theorem}

\begin{proof} To obtain the existence of the invariant measure, we apply the weak form of Harris' Theorem in \citet[Theorem 3.7]{Martin} by verifying the Lyapunov condition, the $\tilde{d}$-contracting condition, and the $\tilde{d}$-small condition. 

\noindent(i) {\bf Lyapunov condition:} Let $V(x,y)=\norm{x}^2$. To verify that it satisfies (2.1) in \citet[Definition 2.1]{Martin}, we take $x=\theta^\e_0$ and $(P_t)_{t \geq 0}$ to be the semi-group associated with $(\theta^\e_t)_{t \geq 0}$. Then Lemma \ref{boundedtheta} yields that $V(x,y)$ is a Lyapunov function. The existence of the Lyapunov function can be understood as that the coupled process does not go to infinity.

\noindent(ii) {\bf$\tilde d$-contracting condition:} $\tilde d$-contracting states that there exists $t^*>0$ such that for any $t>t^*$, there exists some $\alpha<1$ such that 
$$ \widetilde\cW_{\tilde d}(\delta_{(m,a)}P^\e_t, \delta_{(n,b)}P^\e_t)\leq \alpha \tilde d ((m,a), (n,b))$$
for any $(m, a), (n, b)\in \R^K\times S$ such that $\tilde d ((m,a), (n,b))<1$. Here, $(P^\e_t)_{t \geq 0}$ is the semi-group operator associated with (\ref{eq:AS:theta}). Notice that $\tilde d((m,a),(n,b))<1$ implies $a=b$. Let $ (\theta^{(m,a)}_t)_{t \geq 0},\ (\theta^{(n,a)}_t)_{t \geq 0}$ solve the following equations: 
\begin{align*}
    \theta^{(m,a)}_t= m-\int_0^t \nabla_{\theta} f(\theta^{(m,a)}_t,  \tilde V^a_{ s/\e}) \mathrm{d}s,\\
    \theta^{(n,a)}_t= n-\int_0^t \nabla_{\theta} f(\theta^{(n,a)}_t,  \tilde V^a_{s/\e}) \mathrm{d}s,
\end{align*}
where $\tilde V^a_t= V^a_t$ for $t\le T^a$ and $\tilde V^a_t= V^\pi_t$ for $t>T^a.$
Then by It\^o's formula and Assumption \ref{as1.2},
%\begin{align*}
 %   \norm{\theta^{(m,a)}_t-\theta^{(n,a)}_t}^2=\norm{m-n}^2-2\int_0^t\ip{\theta^{(m,a)}_s-\theta^{(n,a)}_s,\nabla_{\theta} f(\theta^{(m,a)}_s, \tilde V^a_{ \frac{s}{\e}})-\nabla_{\theta} f(\theta^{(n,a)}_s, \tilde V^a_{ \frac{s}{\e}})} ds.
%\end{align*}
\begin{align*}
    \mathrm{d}\norm{\theta^{(m,a)}_t-\theta^{(n,a)}_t}^2/\mathrm{d}t=&-2\ip{\theta^{(m,a)}_t-\theta^{(n,a)}_t,\nabla_{\theta} f(\theta^{(m,a)}_t, \tilde V^a_{ s/\e})-\nabla_{\theta} f(\theta^{(n,a)}_t, \tilde V^a_{s/\e})}\\
    \le& -\kappa\norm{\theta^{(m,a)}_t-\theta^{(n,a)}_t}^2.
\end{align*}
By Gr\"onwall's inequality, 
\begin{equation}\label{eq:aa}
\norm{\theta^{(m,a)}_t-\theta^{(n,a)}_t}^2\le \exp({-\kappa t})\norm{m-n}^2.
\end{equation}
Noticing that $\tilde d((m,a),(n,b))<1$ implies $\norm{m-n}^2<1$, by choosing $t\ge \frac{1}{\kappa},$ we obtain
\begin{align*}
 \norm{\theta^{(m,a)}_t-\theta^{(n,a)}_t}^2 &\le \exp({-1})\norm{m-n}^2 \\ &= \exp({-1})(\norm{m-n}^2\land 1)=\exp({-1})\tilde{d}^2((m,a),(n,a)).   
\end{align*}
Therefore, with $t^* = 1/\kappa$,
\begin{align*}
    \widetilde\cW_{\tilde d}(H^\e_t(\cdot|m,a),H^\e_t(\cdot|n,b))\le \tE\Big[\norm{\theta^{(m,a)}_t-\theta^{(n,a)}_t}\Big]\le \exp({-1})\tilde d((m,a),(n,b)).
\end{align*}

\noindent(iii) {\bf $\tilde d$-small condition:} We shall verify that there exists $t_*>0$ such that for any $t>t_*$, the sublevel set $\mathscr{V}:=\{(x, y)\in \R^K\times S\ |\ V(x, y)\leq {32K_f^2 }/{\kappa^2 }\}$ is $\tilde d$-small for $(P^\e_t)_{t \geq 0}$, meaning that there exists a constant $\zeta$ such that 
$$ \widetilde\cW_{\tilde d}(\delta_{(m,a)}P^\e_t, \delta_{(n,b)}P^\e_t)\leq 1-\zeta,$$
for all $(m, a), (n, b)\in \mathscr{V}$. Let $ (\theta^{(m,a)}_t)_{t \geq 0},\ (\theta^{(n,b)}_t)_{t \geq 0}$ solve the following equations: 
\begin{align*}
    \theta^{(m,a)}_t= m-\int_0^t \nabla_{\theta} f(\theta^{(m,a)}_t,  \tilde V^a_{ s/\e}) \mathrm{d}s,\\
    \theta^{(n,b)}_t= n-\int_0^t \nabla_{\theta} f(\theta^{(n,b)}_t,  \tilde V^b_{s/\e}) \mathrm{d}s,
\end{align*}
where $\tilde V^a_t= V^a_t$ for $t\le T^a$ and $\tilde V^a_t= V^\pi_t$ for $t>T^a$; $\tilde V^b_t= V^b_t$ for $t\le T^b$ and $\tilde V^b_t= V^\pi_t$ for $t>T^b$.
By It\^o's formula, Assumption \ref{as1.2}, and the $\varepsilon$-Young inequality, 
%\begin{align*}
%    \norm{\theta^{(m,a)}_t-\theta^{(n,b)}_t}^2=\norm{m-n}^2-2\int_0^t\ip{\theta^{(m,a)}_s-\theta^{(n,b)}_s,\nabla_{\theta} f(\theta^{(m,a)}_s, \tilde V^a_{ \frac{s}{\e}})-\nabla_{\theta} f(\theta^{(n,b)}_s, \tilde V^b_{ \frac{s}{\e}})} ds.
%\end{align*}
\begin{align*}
    &\mathrm{d}\norm{\theta^{(m,a)}_t-\theta^{(n,b)}_t}^2/\mathrm{d}t\\
    =&-2\ip{\theta^{(m,a)}_t-\theta^{(n,b)}_t,\nabla_{\theta} f(\theta^{(m,a)}_t, \tilde V^a_{t/\e})-\nabla_{\theta} f(\theta^{(n,b)}_t, \tilde V^b_{ t/\e})}\\
    =& -2\ip{\theta^{(m,a)}_t-\theta^{(n,b)}_t,\nabla_{\theta} f(\theta^{(m,a)}_t, \tilde V^a_{t/\e})-\nabla_{\theta} f(\theta^{(n,b)}_t, \tilde V^a_{ t/\e})}\\
    &\ \ -2\ip{\theta^{(m,a)}_t-\theta^{(n,b)}_t,\nabla_{\theta} f(\theta^{(n,b)}_t, \tilde V^a_{ t/\e})-\nabla_{\theta} f(\theta^{(n,b)}_t, \tilde V^b_{ t/\e})}\\
    \le& -2\kappa \norm{\theta^{(m,a)}_t-\theta^{(n,b)}_t}^2+2\abs{\ip{\theta^{(m,a)}_t-\theta^{(n,b)}_t,\nabla_{\theta} f(\theta^{(n,b)}_t, \tilde V^a_{ t/\e})-\nabla_{\theta} f(\theta^{(n,b)}_t, \tilde V^b_{t/\e})}}\\
    \le& -\kappa \norm{\theta^{(m,a)}_t-\theta^{(n,b)}_t}^2+\frac{4}{\kappa}\norm{\nabla_{\theta} f(\theta^{(n,b)}_t, \tilde V^a_{ t/\e})-\nabla_{\theta} f(\theta^{(n,b)}_t, \tilde V^b_{ t/\e})}^2.
\end{align*}
Multiplying $\exp({\kappa t})$ on both sides, we obtain
\begin{align*}
    \mathrm{d}\Big(\exp({\kappa t})\norm{\theta^{(m,a)}_t-\theta^{(n,b)}_t}^2\Big)/\mathrm{d}t
    \le \frac{4\exp({\kappa t})}{\kappa}\norm{\nabla_{\theta} f(\theta^{(n,b)}_t, \tilde V^a_{ t/\e})-\nabla_{\theta} f(\theta^{(n,b)}_t, \tilde V^b_{t/\e})}^2,
\end{align*}
that is 
\begin{align*}
 \exp({\kappa t})\norm{\theta^{(m,a)}_t-\theta^{(n,b)}_t}^2 &\le \norm{m-n}^2 \\&\qquad+ \frac{4}{\kappa}\int_0^t \exp({\kappa s})\norm{\nabla_{\theta} f(\theta^{(n,b)}_s, \tilde V^a_{ t/\e})-\nabla_{\theta} f(\theta^{(n,b)}_s, \tilde V^b_{ t/\e})}^2\mathrm{d}s.
\end{align*}
Notice that $\tilde V^a_t= \tilde V^b_t$ if $t>T^a\lor T^b$. Hence, we have
\begin{align*}
 &\exp({\kappa t})\norm{\theta^{(m,a)}_t-\theta^{(n,b)}_t}^2 \\ & \qquad \le  \norm{m-n}^2+\frac{4}{\kappa}\int_0^{t\land \e (T^a\lor T^b)} \exp({\kappa s})\norm{\nabla_{\theta} f(\theta^{(n,b)}_s, \tilde V^a_{  t/\e})-\nabla_{\theta} f(\theta^{(n,b)}_s, \tilde V^b_{  t/\e})}^2\mathrm{d}s.
\end{align*}
 For $(m,a), (n,b)\in \mathscr{V}\times S$, by Lemma \ref{boundedtheta}, we have that $$\norm{\nabla_{\theta} f(\theta^{(m,a)}_t,\tilde V^a_{ t/\e})} \text{ and }\norm{\nabla_{\theta} f(\theta^{(n,b)}_t,\tilde V^b_{ t/\e})}$$ are bounded by some constant $C_f.$ 
 This implies
\begin{align}\label{eqprop1}
\norm{\theta^{(m,a)}_t-\theta^{(n,b)}_t}^2\le  \exp({-\kappa t})\norm{m-n}^2+4 C_f \exp({-\kappa t}) \exp({\kappa \e (T^a\lor T^b)}).
\end{align}
Recall $0< \varepsilon\le 1\land \frac{\delta}{2\kappa}$. By (iii), Assumption \ref{as1.0},
\begin{align*}
    \tE [\exp({\kappa \e (T^a\lor T^b)})]=&\kappa \e\int_0^\infty \exp({\kappa \e x})\tP(T^a\lor T^b\ge x)\mathrm{d}x\\
    \le&  C\kappa \e\int_0^\infty \exp({\kappa \e x})\exp({-\delta x})\mathrm{d}x\\
    \le& \frac{C\delta}{2}\int_0^\infty \exp\left({-\frac{\delta x}{2} }\right)\mathrm{d}x= C.
\end{align*}
Therefore,
\begin{align*}
\tE \Big[\norm{\theta^{(m,a)}_t-\theta^{(n,b)}_t}^2\Big]\le  \exp({-\kappa t})\norm{m-n}^2+ C_f \exp({-\kappa t}).
\end{align*}
Moreover,  
$$
\tP(\tilde V^a_{t/\e}\ne \tilde V^b_{t/\e})=\tP(T^a\lor T^b\ge {t}/{\e})\le  C\exp({-\delta {t}/{\varepsilon}}) \le C\exp({-2\kappa t}).
$$
Hence for any $(m,a),(n,b)\in\mathscr{V}\times S$, by taking $t\ge \frac{1}{\kappa}[\log(8C_f)+\log({512K_f^2 }/{\kappa^2 })+\frac{1}{2}\log(8C)],$ we have
$$
\widetilde\cW^2_{\tilde d}(H^\e_t(\cdot|m,a),H^\e_t(\cdot|n,b))\le \tP(\tilde V^a_{t/\e}\ne \tilde V^b_{t/\e})+\tE \Big[\norm{\theta^{(m,a)}_t-\theta^{(n,b)}_t}^2\Big]\le \frac{1}{2}.
$$
We have verified all three conditions from \citet[Theorem 3.7]{Martin} and hence conclude the existence and uniqueness of the invariant measure of $(\theta_t^\e,V_{t/\e})_{t\ge 0}$ denoted as $\Pi^\e$.

Next, we are going to prove (\ref{ineq2.6}) and \eqref{ineq:w_pi}. Define $\Theta^\varepsilon$ such that  $( \Theta^\varepsilon,  \tilde V^\pi_0)\sim \Pi^\varepsilon$ in $\tP.$ Let $\theta^{\Pi^\varepsilon}_t$ and $\theta^{(m,\pi)}_t$ solve following equations
\begin{align*}
    \theta^{\Pi^\varepsilon}_t= \Theta^\varepsilon-\int_0^t \nabla_{\theta} f(\theta^{\Pi^\varepsilon}_t,  \tilde V^\pi_{t/\e}) \mathrm{d}s,\\
    \theta^{(m,\pi)}_t= m-\int_0^t \nabla_{\theta} f(\theta^{(m,\pi)}_t,  \tilde V^\pi_{t/\e}) \mathrm{d}s.
\end{align*}
Recall that from (\ref{eqprop1}) and (\ref{eq:aa}), we have
\begin{align*}
\norm{\theta^{(m,a)}_t-\theta^{\Pi^\varepsilon}_t}^2\le&  \exp({-\kappa t})\norm{m-\Theta^\varepsilon}^2+4 C_f \exp({-\kappa t}) \exp({\kappa \e T^a}),\\
\norm{\theta^{(m,\pi)}_t-\theta^{\Pi^\varepsilon}_t}^2\le& \exp({-\kappa t})\norm{m-\Theta^\varepsilon}^2.
\end{align*}
Therefore, by (iii), Assumption \ref{as1.0} and $0< \varepsilon\le 1\land ({\delta}/{2\kappa})$, 
\begin{align*}
    \widetilde\cW^2_{\tilde d}(H^\e_t(\cdot|m,a),\Pi^\e)\le& \tP(\tilde V^a_{t/\e}\ne \tilde V^\pi_{t/\e})+\tE \Big[\norm{\theta^{(m,a)}_t-\theta^{\Pi^\varepsilon}_t}^2\Big]\\
    \le& \tP(T^a\ge \frac{t}{\e})+ \exp({-\kappa t})\tE[\norm{m-\Theta^\varepsilon}^2]+4 C_f \exp({-\kappa t}) \tE[\exp({\kappa \e T^a})]\\
    \le& C_f\exp({-\kappa t})\int_{\R^K}(1+ \norm{x-m}^2)\Pi^\e(\mathrm{d}x,S),
\end{align*}
\begin{align*}
    \widetilde\cW^2_{\tilde d}(H^\e_t(\cdot|m,\pi),\Pi^\e)\le& \tE \Big[\norm{\theta^{(m,\pi)}_t-\theta^{\Pi^\varepsilon}_t}^2\Big]\\
    \le&  \exp({-\kappa t})\tE[\norm{m-\Theta^\varepsilon}^2]\\
    \le& C_f\exp({-\kappa t})\int_{\R^K} \norm{x-m}^2\Pi^\e(\mathrm{d}x,S).
\end{align*}

\end{proof}
In the following corollaries, we study the integrals on the right-hand side of the inequalities in Theorem~\ref{dispitheta}. Moreover, we show that the result above immediately implies not only geometric ergodicity of the coupled process $(\theta^{\e}_t, V_{t/\e})_{t \geq 0}$, but also of its marginal, the stochastic gradient process $(\theta^{\e}_t)_{t \geq 0}$.
\begin{corollary}\label{cor:w_conv}
Under the same assumptions as Theorem \ref{dispitheta}, there exists a constant 
$C_{f,m}$ that depends only on $f$ and the initial value $m=\theta_0^\e$, such that 
\begin{align}\label{disthepi}
 \widetilde\cW_{\tilde d}(H^\e_t(\cdot|m,a),\Pi^\e)\le C_{f,m}\exp\left({-\frac{\kappa t}{2} }\right),
\end{align}
\begin{align}\label{disthepi2}
 \widetilde\cW_{\tilde d}(H^\e_t(\cdot|m,\pi),\Pi^\e)\le C_{f,m}\exp\left({-\frac{\kappa t}{2} }\right),
\end{align}
\begin{align}\label{distheta}
    \cW_d(C^\e_t(\cdot|m,a),\Pi^\e(\cdot,S))\le C_{f,m}\exp\left({-\frac{\kappa t}{2} }\right),
\end{align}
\begin{align}\label{distheta1}
    \cW_d(C^\e_t(\cdot|m,\pi),\Pi^\e(\cdot,S))\le C_{f,m}\exp\left({-\frac{\kappa t}{2} }\right).
\end{align}
\end{corollary}
\begin{proof}
By Lemma \ref{boundedtheta},
$$
\int_{\norm{x}^2\ge\frac{8K^2 }{\kappa^2 }+\norm{m}^2+1 }\Pi^\e(\mathrm{d}x,S)=\lim_{t\to \infty} \tP(\norm{\theta^{(m,a)}_t}\ge \frac{8K_f^2 }{\kappa^2 }+\norm{m}^2+1 )=0.
$$
Let $C_{f,m}=C_f(2\norm{m}+\frac{4K_f }{\kappa }+2)$. From (\ref{ineq2.6}), we have
\begin{align*}
    \tilde\cW_{\tilde d}(H^\e_t(\cdot|m,a),\Pi^\e)\le&  C_f\exp\left({\frac{-\kappa t}{2}}\right)\Big(\int_{\R^K}(1+ \norm{x-m}^2)\Pi^\e(\mathrm{d}x,S)\Big)^{1/2}\\
    =&C_f\exp\left({\frac{-\kappa t}{2}}\right)\Big(\int_{\norm{x}^2\le\frac{8K_f^2 }{\kappa^2 }+\norm{m}^2+1} (1+\norm{x-m}^2)\Pi^\e(\mathrm{d}x,S)\Big)^{1/2}\\
    \le& C_{f,m}\exp\left({\frac{-\kappa t}{2}}\right).
\end{align*}
(\ref{disthepi2}) can be derived similarly from (\ref{ineq:w_pi}). Moreover, notice that
\begin{align}\label{eqs:api}
    \cW_d(C^\e_t(\cdot|m,a),\Pi^\e(\cdot,S))\le& \tE \Big[\norm{\theta^{(m,a)}_t-\theta^{\Pi^\varepsilon}_t}\Big],
\end{align}
\begin{align}\label{eqs:pipi}
    \cW_d(C^\e_t(\cdot|m,\pi),\Pi^\e(\cdot,S))\le& \tE \Big[\norm{\theta^{(m,\pi)}_t-\theta^{\Pi^\varepsilon}_t}\Big].
\end{align}
(\ref{distheta}) and (\ref{distheta1}) can be derived similarly from (\ref{eqs:api}) and (\ref{eqs:pipi}).
\end{proof}
Combining (\ref{distheta}) and (\ref{distheta1}), we immediately have:
\begin{corollary}\label{cor3}
Under the same assumptions as Theorem \ref{dispitheta},
\begin{align*}
    \cW_d(C^\e_t(\cdot|m,a),C^\e_t(\cdot|m,\pi))\le C_{f,m}\exp\left({-\frac{\kappa t}{2} }\right).
\end{align*} 
\end{corollary}

So far, we have shown that the stochastic gradient process $(\theta^\varepsilon_t)_{t \geq 0}$ converges to a unique stationary measure $\Pi^\varepsilon(\cdot, S)$. It is often not possible to determine this stationary measure. However, we can comment on its asymptotic behavior as $\varepsilon \rightarrow 0$. Indeed, we will show that $\Pi^\varepsilon(\cdot, S)$ concentrates around the minimizer $\theta_*$ of the full target function.

\begin{prop}\label{dispixing}
 Under Assumption \ref{as1.2}, the  
measure $\Pi^\e(\cdot,S)$ on $(\R^K, \mathcal{B}(\R^K))$ approximates $\delta(\cdot-\theta_*).$  In other words, we have
\begin{align*}
    \cW_d(\Pi^\e(\cdot,S),\delta(\cdot-\theta_*))\le \rho(\e)
\end{align*}
where $\rho: (0,1) \rightarrow [0,1]$ and $\lim_{\e\to 0}\rho(\e)=0.$
\end{prop}
\begin{proof}
By the triangle inequality, 
\begin{align*}
    &\cW_d(\Pi^\e(\cdot,S),\delta(\cdot-\theta_*))\nonumber\\ 
    \le& \cW_d(\Pi^\e(\cdot,S),C^\e_t(\cdot|m,\pi))
    +\cW_d(C^\e_t(\cdot|m,\pi),\delta(\cdot-\zeta_t))+ \cW_d(\delta(\cdot-\zeta_t),\delta(\cdot-\theta_*)).
\end{align*}
Let $\theta_0=\theta^\e_0=\theta_*.$ Then by Lemma \ref{norandomthe}, we have the last term 
$$
\cW_d(\delta(\cdot-\zeta_t),\delta(\cdot-\theta_*))\le \norm{\theta_*-\theta_*}\exp({-\kappa t})=0.
 $$
By (\ref{disthepi}) and Corollary \ref{deuxthe}, for any $t\ge 0,$
\begin{align*}
     \cW_d(\Pi^\e(\cdot,S),C^\e_t(\cdot|m,\pi))+\cW_d(C^\e_t(\cdot|m,\pi),\delta(\cdot-\zeta_t))
    \le C_{f,\theta_*}\exp\left({-\frac{\kappa t}{2}}\right)+(\exp(t)\alpha(\e))\land 1.
\end{align*}
By choosing $t=-\log({{1\land\alpha(\e)}})/2,$ we get
\begin{align*}
    \cW_d(\Pi^\e(\cdot,S),\delta(\cdot-\theta_*))\le& C_{f,\theta_*}\exp\left({-\frac{\kappa t}{2}}\right)+(\exp(t)\alpha(\e))\land 1\\
    \le&  C_{f,\theta_*}(\alpha(\e))^\frac{\kappa}{4}+(\alpha(\e))^{1/2}.
\end{align*}
Taking $\rho(\e):=\big(C_{f,\theta_*}(\alpha(\e))^\frac{\kappa}{4}+(\alpha(\e))^{1/2}\big)\land 1$ completes the proof.
\end{proof}

\section{Stochastic gradient processes with decreasing learning rate} \label{Sec_SGPD}

Constant learning rates are popular in some practical situations, but the associated stochastic gradient process usually does not converge to the minimizer of $\Phi$. This is also true for the discrete-time stochastic gradient descent algorithm. However, SGD can converge to the minimizer if the learning rate is decreased over time. In the following, we discuss a decreasing learning version of the stochastic gradient process and show that this dynamical system indeed converges to the minimizer of $\Phi$.

As discussed in Section~\ref{Subsec_Intro_thisWork}, we obtain the stochastic gradient process with decreasing learning rate by non-linearly rescaling the time in the constant-learning-rate index process $(V_t)_{t \geq 0}$. Indeed, we choose a function $\beta:[0, \infty) \rightarrow [0,\infty)$ and then define the decreasing learning rate index process by $(V_{\beta(t)})_{t \geq 0}$. We have discussed an intuitive way to construct a rescaling function $\beta$ also in Section~\ref{Subsec_Intro_thisWork}.

In the following, we define $\beta$ through an integral $\beta(t) = \int_0^t \mu(s) \mathrm{d}s,$ $t \geq 0$. We commence this section with necessary growth conditions on $\mu$ which allow us to then give the formal definition of the stochastic gradient process with decreasing learning rate. Then, we study the longtime behavior of this process.

\begin{assumption}\label{asmu}
Let $\mu:[0,\infty)\to (0,\infty)$ be a non-decreasing continuously differentiable function with $\lim_{t\to\infty} \mu(t)=\infty$ and
\begin{align*}
    \lim_{t\to\infty}\frac{\mu'(t)t}{\mu(t)}=0.
\end{align*}
\end{assumption}

Assumption \ref{asmu} implies that $\mu$ goes to infinity, but  at a very slow pace. Indeed it says that $\lim_{t\to\infty}\mu(t)/t^\gamma=0$, $\gamma>0$, that is $\mu(t)$ grows slower than any polynomial. %In other words, the learning rate decays slower than $t^{-\gamma}$, $\forall \gamma>0$. \textcolor{teal}{}

\begin{defi}
The \emph{stochastic gradient process with decreasing learning rate (SGPD)} is a solution of the following stochastic differential equation,
\begin{equation}\label{eq:AS:xi}
\left\{ \begin{array}{l}
\mathrm{d}\xi_t = - \nabla_\xi f(\xi_t, V_{\beta{(t)}})\mathrm{d}t, \\
\xi_0 = \theta_0,
\end{array} \right.
\end{equation}
where f satisfies Assumption \ref{asSGPf}, $(V_t)_{t\ge0}$ is a Feller process that satisfies Assumption \ref{as1.0}, and $\beta(t)=\int_0^t \mu(s)\mathrm{d}s$ with $\mu$ satisfying Assumption \ref{asmu}.
\end{defi}

To see that $(\xi_t)_{t \geq 0}$ is well-defined, consider the following: $t \mapsto \beta(t)$ is an increasing continuous function. Thus, $(V_{\beta(t)})_{t \geq 0}$ is c\`adl\`ag and Feller with respect to $(\mathcal{F}_{\beta(t)})_{t \geq 0}$. We then obtain well-definedness of $(\xi_t)_{t \geq 0}$ by replacing $(V_{t/\e})_{t \geq 0}$ by $(V_{\beta(t)})_{t \geq 0}$ in the proof of Proposition \ref{thwk30}.

%Next, we are going to consider the following equation (SGPD), 
%\begin{equation}\label{eq:AS:xi}
 % \xi_t =  \theta_0-\int_0^t \nabla_{\xi} f(\xi_s, V_{\beta(s)}) ds  
%\end{equation}
%where $V_t$ was introduced in Section \ref{mainprocess}, and $\beta(t)=\int_0^t \mu(s)ds$ for $\mu(s)>0$. 
%We also assume that $(\mu'(t)t)/\mu(t)\to 0.$ 

We now move on to studying the longtime behavior of the SGPD $(\xi_t)_{t \geq 0}$. In a first technical result, we establish a connection between SGPD $(\xi_t)_{t \geq 0}$ and a time-rescaled version of SGPC $(\theta_t^\e)_{t \geq 0}$. To this end, note that
$\dot{\beta}(t) = \mu(t) >0$, $\beta(t)$ is strictly increasing. Hence, the inverse function of $\beta(t)$ exists and $$(\beta)^{-1}(t)=\int_0^t \frac{1}{\mu( (\beta)^{-1}(s)))}\mathrm{d}s.$$ 
This gives us the following inequality.

\begin{prop}\label{hardprop} For any $0<\e<1$, 
\begin{align*}
\norm{\xi_{t}-\theta^\e_{\e \beta(t)}}^2\le C_{f,\theta_0,\mu}\left[\frac{\exp({-2\e\kappa (\beta(t)-\beta(\frac{t}{2}))})}{\e}+\frac{1}{\e}\Big(\abs{\frac{1}{\mu(t)}-\e}+\abs{\frac{1}{\mu(\frac{t}{2})}-\e}\Big)\right]
\end{align*}
almost surely,
where the constant $C_{f,\theta_0,\mu}$ depending only on $f$, the initial data $\theta_0$, and $\mu$. 
\end{prop}
\begin{proof}
From (\ref{eq:AS:xi}) and (\ref{eq:AS:theta}),
\begin{align*}
    \xi_{t} =&  \theta_0-\int_0^{\beta(t)} \nabla_{\xi} f(\xi_{(\beta)^{-1}(s)}, V_{s}) \mathrm{d}(\beta)^{-1}(s),\\
    \theta^\e_{ t}=& \theta_0-\e\int_0^{\frac{t}{\e}} \nabla_\theta f(\theta^\e_{\e s}, V_{s}) \mathrm{d}s.
\end{align*}
Let $b_t:=\mathrm{d}(\beta)^{-1}(t)/\mathrm{d}t=1/\mu( (\beta)^{-1}(t)))>0,$ we have
\begin{align*}
   \xi_{(\beta)^{-1}(t)} =&  \theta_0-\int_0^{t} \nabla_{\xi} f(\xi_{(\beta)^{-1}(s)}, V_{s}) b_s \mathrm{d}s\\
   \theta^\e_{\e t}=& \theta_{0}-\e\int_0^{t} \nabla_\theta f(\theta^\e_{\e s}, V_{s})  \mathrm{d}s
\end{align*}
%\begin{align*}
%    \norm{\theta^\e_{\e t}-\xi_{(\beta)^{-1}(t)}}^2 =&   -2\int_0^t \ip{\theta^\e_{\e s}-\xi_{(\beta)^{-1}(s)},\e\nabla_{\xi} f(\theta^\e_{\e s}, V_{s})-b_s\nabla_{\xi} f(\xi_{(\beta)^{-1}(s)}, V_{s})} ds\\
%    =& -2\int_0^t\ip{\theta^\e_{\e s}-\xi_{(\beta)^{-1}(s)},\e\nabla_{\xi} f(\theta^\e_{\e s}, V_{s})-\e\nabla_{\xi} f(\xi_{(\beta)^{-1}(s)}, V_{s})} ds\\
 %   &-2\int_0^t \ip{\theta^\e_{\e s}-\xi_{(\beta)^{-1}(s)},\e\nabla_{\xi} f(\xi_{(\beta)^{-1}(s)}, V_{s})-b_s\nabla_{\xi} f(\xi_{(\beta)^{-1}(s)}, V_{s})} ds.
Therefore, by It\^o's formula and Assumption \ref{as1.2},
\begin{align*}
    \mathrm{d}\norm{\theta^\e_{\e t}-\xi_{(\beta)^{-1}(t)}}^2/\mathrm{d}t=& -2\ip{\theta^\e_{\e t}-\xi_{(\beta)^{-1}(t)},\e\nabla_\theta f(\theta^\e_{\e t}, V_t)-\e\nabla_{\xi} f(\xi_{(\beta)^{-1}(t)}, V_t)} \\
    &-2(\e-b_t) \ip{\theta^\e_{\e t}-\xi_{(\beta)^{-1}(t)},\nabla_{\xi} f(\xi_{(\beta)^{-1}(t)}, V_t)}\\
    \le& -2\e\kappa\norm{\theta^\e_{\e t}-\xi_{(\beta)^{-1}(t)}}^2+C_{f,\theta_0}\abs{b_t-\e} ,
\end{align*}
where the last step follows from the boundedness of $\theta^\e_{\e t}$, $\xi_{(\beta)^{-1}(t)}$, and $\nabla_{\xi} f(\xi_{(\beta)^{-1}(t)}, V_t)$. $\xi_{(\beta)^{-1}(t)}$ is bounded can be showed similarly to Lemma \ref{boundedtheta}. Multiplying $\exp({2\e\kappa t})$ on both sides, we obtain
\begin{align*}
    \mathrm{d}\Big(\exp({2\e\kappa t})\norm{\theta^\e_{\e t}-\xi_{(\beta)^{-1}(t)}}^2\Big)/\mathrm{d}t
    \le& C_{f,\theta_0}\abs{b_t-\e}\exp({2\e\kappa t}),
\end{align*}
which implies
\begin{align*}
    \norm{\theta^\e_{\e t}-\xi_{(\beta)^{-1}(t)}}^2\le& C_{f,\theta_0}\exp({-2\e\kappa t})\int_0^t\abs{b_s-\e}\exp({2\e\kappa s})\mathrm{d}s.
\end{align*}
Notice that $b_s$ is bounded and non-increasing, hence we have
\begin{align*}
\norm{\xi_{t}-\theta^\e_{\e \beta(t)}}^2\le& C_{f,\theta_0}\exp({-2\e\kappa \beta(t)})\int_0^{\beta(t)}\abs{b_s-\e}\exp({2\e\kappa s})\mathrm{d}s\\
=& C_{f,\theta_0}\exp({-2\e\kappa \beta(t)})\Big(\int_0^{\beta(\frac{t}{2})}+\int^{\beta(t)}_{\beta(\frac{t}{2})}\Big)\abs{b_s-\e}\exp({2\e\kappa s})\mathrm{d}s\\
\le& C_{f,\theta_0,\mu}\frac{\exp({-2\e\kappa (\beta(t)-\beta(\frac{t}{2}))})}{\e} \\ &\qquad+C_{f,\theta_0}\exp({-2\e\kappa \beta(t)})\int^{\beta(t)}_{\beta(\frac{t}{2})}\abs{b_s-\e}\exp({2\e\kappa s})\mathrm{d}s\\
    \le&  C_{f,\theta_0,\mu}\left[\frac{\exp({-2\e\kappa (\beta(t)-\beta(\frac{t}{2}))})}{\e}+\frac{1}{\e}\Big(\abs{\frac{1}{\mu(t)}-\e}+\abs{\frac{1}{\mu(\frac{t}{2})}-\e}\Big)\right].
\end{align*}
\end{proof} 
Now, we get to the main result of this section, where we show the convergence of $(\xi_t)_{t \geq 0}$ to the minimizer $\theta_*$ of $\Phi$. In the following, we denote  
\begin{align*}
     D_t(B|\theta_0, a)&:=\Prb_a(\xi_t\in B|\xi_0=\theta_0), \\ D_t(B|\theta_0,\pi)&:=\Prb_\pi(\xi_t\in B|\xi_0=\theta_0) \qquad \qquad (B\in \mathcal{B}(\R^K), \theta_0\in \R^K),
\end{align*}
 where $a \in S$ and $\pi$ is the invariant measure of $(V_t)_{t \geq 0}$, respectively.
\begin{theorem}
Under Assumption  \ref{as1.2}, given $\theta_0\in \R^K$ and $a\in S$, there exists $T>0$ such that for any $t>T$,
%\begin{align}\label{th2.1}
 %     \cW_d(D_t(\cdot|\theta_0,\pi),\delta(\cdot-\theta_*))\le C_{f,\theta_0,\mu,\e}\Big[e^{-c_{\e,f} (\beta(t)-\beta(\frac{t}{2}))}+\mathcal{A}(t,\e)\Big]\\
  %    \cW_d(D_t(\cdot|\theta_0,a),\delta(\cdot-\theta_*))\le C_{f,\theta_0,\mu,\e}\Big[e^{-c_{\e} (\beta(t)-\beta(\frac{t}{2}))}+\mathcal{A}(t,\e)\Big]
 %\end{align}
 \begin{align}\label{th2.1}
      \cW_d(D_t(\cdot|\theta_0,\pi),\delta(\cdot-\theta_*))\le C_{f,\theta_0,\mu}A(t),
 \end{align}
\begin{align}\label{th2.2}
     \cW_d(D_t(\cdot|\theta_0,a),\delta(\cdot-\theta_*))\le C_{f,\theta_0,\mu}A(t),
 \end{align}
where $$A(t):=\exp\left({\frac{-\kappa t}{8}}\right)+\left[\frac{\mu(t)-\mu(\frac{t}{2})}{\mu(t)}\right]^{1/2}+\rho\left(\frac{1}{\mu(\frac{t}{2})}\right)$$
and $\lim_{t\to\infty}A(t)=0.$
\end{theorem}

\begin{proof}
To prove (\ref{th2.1}), by the triangle inequality,
\begin{align*}
    &\cW_d(D_t(\cdot|\theta_0,\pi),\delta(\cdot-\theta_*))\\
    \le& \cW_d(D_t(\cdot|\theta_0,\pi),C^\e_{\e\beta(t)}(\cdot|\theta_0,\pi))+\cW_d(C^\e_{\e\beta(t)}(\cdot|\theta_0,\pi),\Pi^\e(\cdot,S))
    +\cW_d(\Pi^\e(\cdot,S),\delta(\cdot-\theta_*)).
\end{align*}
For the last two terms, by (\ref{distheta1}) and Proposition \ref{dispixing}, 
\begin{align*}
    \cW_d(C^\e_{\e\beta(t)}(\cdot|\theta_0,\pi),\Pi^\e(\cdot,S))+\cW_d(\Pi^\e(\cdot,S),\delta(\cdot-\theta_*))\le C_{f,m}\exp({{-\kappa \e\beta(t)}/{2}})+\rho(\e).
\end{align*}
For the first term, by Proposition \ref{hardprop}, 
\begin{align*}
    &\cW_d(D_t(\cdot|\theta_0,\pi),C^\e_{\e\beta(t)}(\cdot|\theta_0,\pi)) \\&\qquad \le C_{f,\theta_0,\mu}\left[\frac{\exp({-2\e\kappa (\beta(t)-\beta(\frac{t}{2}))})}{\e}+\frac{1}{\e}\Big(\abs{\frac{1}{\mu(t)}-\e}+\abs{\frac{1}{\mu(\frac{t}{2})}-\e}\Big)\right]^{1/2}.
\end{align*}
%{\color{blue}
%By choosing $c_{\e,f}=2\e\kappa$, $C_{f,\theta_0,\mu,\e}= \frac{C_{f,\theta_0,\mu}}{\e}$ and $\mathcal{A}(t,\e)=\Big(\abs{\frac{1}{\mu(t)}-\e}+\abs{\frac{1}{\mu(\frac{t}{2})}-\e}\Big)\Big]^{1/2}+\e\rho(\e)$
%}

Since $\lim_{t\to\infty}\mu(t)=\infty$, there exists $T>0$ such that $1/\mu(\frac{T}{2}) < \frac{\delta}{2\kappa}$.
Let $\e=1/\mu(\frac{t}{2})$, $t>T$, we have
 \begin{align*}
     \exp\left({\frac{-\kappa \e\beta(t)}{2}}\right)=\exp\left({\frac{-\kappa \int_0^t\mu(s)\mathrm{d}s}{2\mu(\frac{t}{2})}}\right)\le \exp\left({\frac{-\kappa t}{8}}\right)
 \end{align*}
 and
 \begin{align*}
     \frac{\exp\left({-2\e\kappa (\beta(t)-\beta(\frac{t}{2}))}\right)}{\e}&=\mu\left(\frac{t}{2}\right)\exp\left({\frac{-\kappa \int_{\frac{t}{2}}^t\mu(s)\mathrm{d}s}{\mu(\frac{t}{2})}}\right) \\ &\le \mu\left(\frac{t}{2}\right)\exp\left({\frac{-\kappa t}{2}}\right)\le C \exp\left({\frac{-\kappa t}{8}}\right).
 \end{align*}
Therefore,
 \begin{align*}
    \cW_d(D_t(\cdot|\theta_0,\pi),\delta(\cdot-\theta_*))\le C_{f,\theta_0,\mu}\Big[\exp\left({\frac{-\kappa t}{8}}\right)+\left(\frac{\mu(t)-\mu(\frac{t}{2})}{\mu(t)}\right)^{1/2}+\rho\left(\frac{1}{\mu\left(\frac{t}{2}\right)}\right)\Big]
\end{align*}
From Assumption \ref{asmu}, by the mean value theorem, 
\begin{align*}
    \frac{\mu(t)-\mu(\frac{t}{2})}{\mu(t)}=\frac{t\mu'(\tau_t)}{2\mu(t)}=\frac{\tau_t\mu'(\tau_t)}{\mu(\tau_t)}\frac{t}{2\tau_t}\frac{\mu(\tau_t)}{\mu(t)}\le \frac{\tau_t\mu'(\tau_t)}{\mu(\tau_t)}\to 0
\end{align*}
where $\tau_t\in[\frac{t}{2},t].$ Thus (\ref{th2.1}) is obtained by taking %$A(t):=\exp({\frac{-\kappa t}{8}})+[\frac{\mu(t)-\mu(\frac{t}{2})}{\mu(t)}]^{1/2}+\rho(\frac{1}{\mu(\frac{t}{2})})$.
$$A(t):=\exp\left({\frac{-\kappa t}{8}}\right)+\left[\frac{\mu(t)-\mu(\frac{t}{2})}{\mu(t)}\right]^{1/2}+\rho\left(\frac{1}{\mu(\frac{t}{2})}\right).$$
To prove (\ref{th2.2}), by the triangle inequality,
\begin{align*}
\cW_d(D_t(\cdot|\theta_0,a),\delta(\cdot-\theta_*))&\le \cW_d(D_t(\cdot|\theta_0,a),C^\e_{\e\beta(t)}(\cdot|\theta_0,a)) \\ &\ \ \ \ \ +\cW_d(C^\e_{\e\beta(t)}(\cdot|\theta_0,a),C^\e_{\e\beta(t)}(\cdot|\theta_0,\pi))\\ &\ \ \ \ \ +\cW_d(C^\e_{\e\beta(t)}(\cdot|\theta_0,\pi),\Pi^\e(\cdot,S)) \\ &\ \ \ \ \
    +\cW_d(\Pi^\e(\cdot,S),\delta(\cdot-\theta_*)).    
\end{align*}
By Corollary \ref{cor3}, we have
$$\cW_d(C^\e_{\e\beta(t)}(\cdot|\theta_0,a),C^\e_{\e\beta(t)}(\cdot|\theta_0,\pi))\le C_{f,m}\exp\left({-\frac{\kappa \varepsilon \beta(t)}{2} }\right).$$
Notice that $C_{f,m}\exp({-\frac{\kappa t}{4} })\le C_{f,m}A(t)$
when $\varepsilon=1/\mu(\frac{t}{2}).$ 
Similar to the proof of (\ref{th2.1}), we have
\begin{align*}
    \cW_d(D_t(\cdot|\theta_0,a),C^\e_{\e\beta(t)}(\cdot|\theta_0,a))+\cW_d(C^\e_{\e\beta(t)}(\cdot|\theta_0,\pi),\Pi^\e(\cdot,S))
    +&\cW_d(\Pi^\e(\cdot,S),\delta(\cdot-\theta_*))\\
    & \qquad \qquad \qquad \le C_{f,\theta_0,\mu}A(t),
\end{align*}
which completes the proof.
\end{proof}

Thus, we have shown that the distribution of $(\xi_t)_{t \geq 0}$ converges in Wasserstein distance to the Dirac measure concentrated in the minimizer $\theta_*$ of $\Phi$. This result is independent of whether we initialize the index process $(V_{\beta(t)})_{t \geq 0}$ with its stationary measure or with any deterministic value. 
\section{From continuous dynamics to practical optimization.} \label{Sec_PracticalOptimisation}
So far, we have discussed the stochastic gradient process as a continuous-time coupling of an ODE and a stochastic process. In order to apply the stochastic gradient process in practice, we need to discretize ODE and stochastic process with appropriate time-stepping schemes. That means, for a given increasing sequence $(t(k))_{k=0}^\infty$, with $t(0) := 0$ and $\lim_{k \rightarrow \infty} t(k) = \infty,$ we seek discrete-time stochastic processes $(\widehat{V}_k, \widehat{\theta}_k)_{k=0}^\infty$, such that $(\widehat{V}_k, \widehat{\theta}_k)_{k=0}^\infty \approx (V_{t(k)}, \theta_{t(k)}^\e)_{k=0}^\infty$ and analogous discretizations for $(V_{\beta(t)}, \xi_t)_{t \geq 0}$.

In the following, we propose and discuss time stepping strategies and the algorithms arising from them. We discuss the index process and gradient flow separately, which we consider sufficient as the coupling is only one-sided.

\subsection{Discretization of the index process}
We have defined the stochastic gradient process for a huge range of potential index processes $(V_t)_{t \geq 0}$. The discretization of such processes has been the topic of several works, see, e.g., \citet{Gillespie1977,lord_powell_shardlow_2014}. In the following, we focus on one case  and refer to those previous works for other settings and details. 

Indeed, we study the setting $S := [-1, 1]$ and $\pi := \mathrm{Unif}[-1, 1]$ and discuss the discretization of $(V_t)_{t \geq 0}$ as a Markov pure jump process and as a reflected Brownian motion. 

\subsubsection*{Markov pure jump process.} A suitable Markov pure jump process is a piecewise constant c\`adl\`ag process $(V_t)_{t \geq 0}$ with Markov transition kernel
\begin{align*}
    \mathbb{P}_x(V_t \in \cdot) = \exp(-\lambda t)\delta(\cdot - x) + (1-\exp(-\lambda t))\mathrm{Unif}[-1,1] \qquad \qquad (t \geq 0),
\end{align*}
where $\lambda > 0$ is a rate parameter. We can now discretize the process $(V_t)_{t \geq 0}$ just through sampling from this Markov kernel for our discrete time points. We describe this in Algorithm~\ref{alg:MPJ}.

\begin{algorithm}[hptb]\caption{Discretized Markov pure jump process}
\begin{algorithmic}[1]
  %\scriptsize
  \STATE initialize $\widehat{V}_0$, $\lambda > 0$, and a sequence of points $(t(k))_{k=0}^\infty$
  \FOR{$k = 1, 2,\ldots$}
  \STATE sample $U \sim \mathrm{Unif}[0,1]$
  \IF{$U \leq \exp(-\lambda (t(k)-t(k-1))$}
   \STATE $\widehat{V}_k \leftarrow \widehat{V}_{k-1}$ \COMMENT{process stays at its current position}
   \ELSE 
   \STATE sample $\widehat{V}_k \sim \mathrm{Unif}[-1,1]$ \COMMENT{process jumps to a new position}
  \ENDIF
  \ENDFOR
  \RETURN $(\widehat{V}_k)_{k=0}^\infty$
\end{algorithmic}
\label{alg:MPJ}
\end{algorithm}

\subsubsection*{Reflected Brownian motion.} We have defined the reflected Brownian Motion on a non-empty compact interval through the Skorohod problem in Subsection~\ref{Subsec_Ex1_Levy_refle}. 

Let $\sigma >0$ and $(W_t)_{t \geq 0}$ be a standard Brownian motion. Probably the easiest way to sample a reflected Brownian motion is by discretizing the rescaled Brownian motion $(\sigma \cdot W_t)_{t \geq 0}$ using the Euler--Maruyama scheme and projecting back to $S$, whenever the sequence leaves $S$. This scheme has been studied by \citet{Petterson}. We describe the full scheme in Algorithm~\ref{alg:RBM}.

\citet{Petterson} shows that this scheme converges at a rather slow rate. As we usually assume that the domain on which we move is rather low-dimensional and the sampling is rather cheap, we can afford  small discretization stepsizes $t(k)-t(k-1)$, for $k \in \mathbb{N}$. Thus, the slow rate of convergence is manageable. Other schemes for the discretization of reflected Brownian motions have been discussed by, e.g., \citet{blanchet_murthy_2018, LIU1995}.

\begin{algorithm}[hptb]\caption{Discretized Reflected Brownian motion on $S$}
\begin{algorithmic}[1]
  %\scriptsize
  \STATE initialize $\widehat{V}_0$, $\sigma > 0$,  a sequence of points $(t(k))_{k=0}^\infty$, and the projection operator  $\mathrm{proj}_S$ mapping onto $S$ 
  \FOR{$k = 1, 2,\ldots$}
  \STATE $V' \leftarrow V_{k-1} + \sigma \sqrt{t(k)-t(k-1)} \psi $, \quad $\psi \sim \mathrm{N}(0,1^2)$ \COMMENT{Euler-Maruyama update}
  \IF{$V' \not\in S$} 
   \STATE $\widehat{V}_k \leftarrow \mathrm{proj}_S V'$ \COMMENT{project back}
   \ELSE
  \STATE $\widehat{V}_k \leftarrow V'$ \COMMENT{accept Euler--Maruyama update}
  \ENDIF
  \ENDFOR
  \RETURN $(\widehat{V}_k)_{k=0}^\infty$
\end{algorithmic}
\label{alg:RBM}
\end{algorithm}

\subsection{Discretization of the gradient flow}
We now briefly discuss the discretization of the gradient flow in the stochastic gradient process. Based on these ideas, we will conduct numerical experiments in Section~\ref{Sec_NumExp}.

\subsubsection*{Stochastic gradient descent} In stochastic gradient descent, the gradient flow is discretized with a forward Euler method. This method leads to an accurate discretization of the respective gradient flow if the stepsize/learning rates are sufficiently small. In the presence of rather large stepsizes and stiff vector fields, however, the forward Euler method may be  inaccurate and unstable, see, e.g., \citet{Quarteroni2007}.

\subsubsection*{Stability} Several ideas have been proposed to mitigate this problem. The stochastic proximal point method, for instance, uses the backward Euler method to discretize the gradient flow; see \citet{Bianchi2015}.
Unfortunately, such implicit ODE integrators require us to invert a possibly highly non-linear and complex vector field. In convex stochastic optimization this inversion can be replaced by evaluating a proximal operator. For strongly convex optimization, on the other hand, \citet{eftekhari} proposes stable explicit methods. 

\subsubsection*{Efficient optimizers} Plenty of highly efficient methods for stochastic optimization methods are nowadays available, especially in machine learning. Those have often been proposed without necessarily thinking of the stable and accurate discretization of a gradient flow: such are adaptive methods \citet{Adam}, variance reduced methods (e.g., \citet{Defazio2014}), or momentum methods (e.g., \citet{Kovachki21} for an overview),  which have been shown in multiple works to be highly efficient; partially also in non-convex optimization.  We could understand those methods also as certain discretizations of the gradient flow. Thus, we may also consider the combination of a feasible index process $(V_t)_{t \geq 0}$ with the discrete dynamical system in, e.g., the Adam method (\citet{Adam}).

%\begin{algorithm}\caption{SGD with SGP}
%\begin{algorithmic}[1]
  %\scriptsize
%  \STATE initialize $\theta_0, \theta_*\in X$, %$x_0\in D$ 
%  \STATE define learning rate $\eta_k$
%  \FOR{$k\leq max\_iter$}
%  \STATE sample $\Delta_k \sim \Delta V$
%  \STATE $x_{k}\gets x_{k-1} +\Delta_k $
%  \STATE $\theta_{k}\gets \theta_{k-1} - \eta_k \nabla \Phi(x_k)$
%  \IF{$\Phi(\theta_k) < \Phi(\theta_*)$}
%   \STATE $\theta_* \gets \theta_{k} $
%  \ENDIF
%  \ENDFOR
%  \RETURN $\theta_*$
%\end{algorithmic}
%\label{alg:SGD_SGP}
%\end{algorithm}

\section{Applications} \label{Sec_NumExp}
We now study two fields of application of the stochastic gradient process for continous data. In the first example, we consider regularized polynomial regression with noisy functional data. In this case, we can easily show that the necessary assumptions for our analysis hold. Thus, we use it to illustrate our analytical results and especially to learn about the implicit regularization that is put in place due to different index proccesses.

In the second example, we study so-called physics-informed neural networks. In these continuous-data machine learning problems, a deep neural network is used to approximate the solution of a partial differential equation. The associated optimization problem is usually non-convex. Our analysis does not hold in this case: We study it to get more insights in the behavior of the stochastic gradient process in state-of-the-art deep learning problems.

\begin{figure}
    \centering
    \includegraphics[width = 0.8\textwidth]{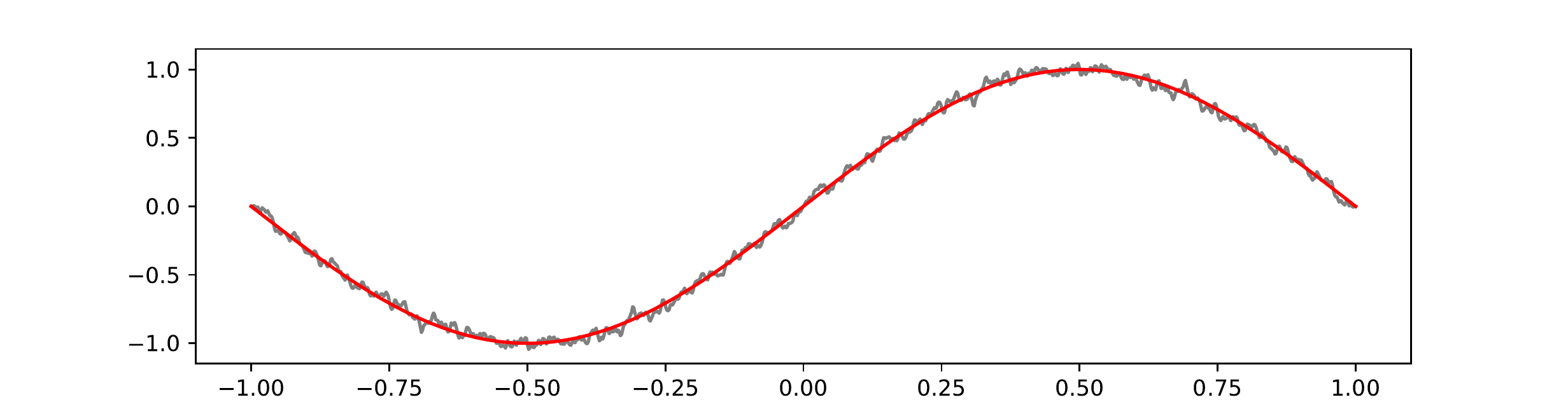}
    \caption{True function $\Theta$ (red) and noisy observation $g$ (grey) in the polynomial regression example.}
    \label{fig:truth_polyn}
\end{figure}
\subsection{Polynomial regression with functional data}
We begin with a simple polynomial regression problem with noisy functional data. We observe the function $g:[-1,1] \rightarrow \mathbb{R}$ which is given through
\begin{equation*}
    g(y) := \Theta(y) + \Xi(y) \qquad (y \in [-1,1]),
\end{equation*}
where $\Theta: [-1, 1] \rightarrow \mathbb{R}$ is a smooth function and $\Xi$ is a Gaussian process with highly oscillating, continuous realizations. We aim at identifying the unknown function $\Theta$ subject to the observational noise $\Xi$. Here, we represent the function $\Theta$ on a basis consisting of a finite number of Legendre polynomials  on $[-1,1]$. We denote this basis of Legendre polynomials by $(\ell_k)_{k=1}^K$.
To estimate the prefactors of the polynomials, we minimize the potential
\begin{equation} \label{eq:polyn_full_opt}
    \Phi(\theta) := \frac{1}{2} \int_{[-1,1]} \left(g(y)-\sum_{k = 1}^K \theta_k\ell_k(y)\right)^2 \mathrm{d}y + \frac{\alpha}{2} \| \theta \|^2_2 \qquad \qquad (\theta \in X),
\end{equation}
where $\alpha > 0$ is a regularization parameter. This can be understood as a maximum-a-posteriori estimation of the unknown $\theta$ with Gaussian prior under the (misspecified) assumption that the data is perturbed with Gaussian white noise. 
We employ the following associated subsampled potentials:
\begin{equation}
    f(\theta, y) := \frac{1}{2} \left(g(y)-\sum_{k = 1}^K \theta_k\ell_k(y)\right)^2  + \frac{\alpha}{2} \| \theta \|^2_2 \qquad \qquad (\theta \in X, y \in [-1, 1]).
\end{equation}
Those subsampled potentials satisfy the strong convexity assumption, i.e., Assumption~\ref{as1.2}.

\begin{figure}
    \centering
    
     \includegraphics[width=0.29\textwidth]{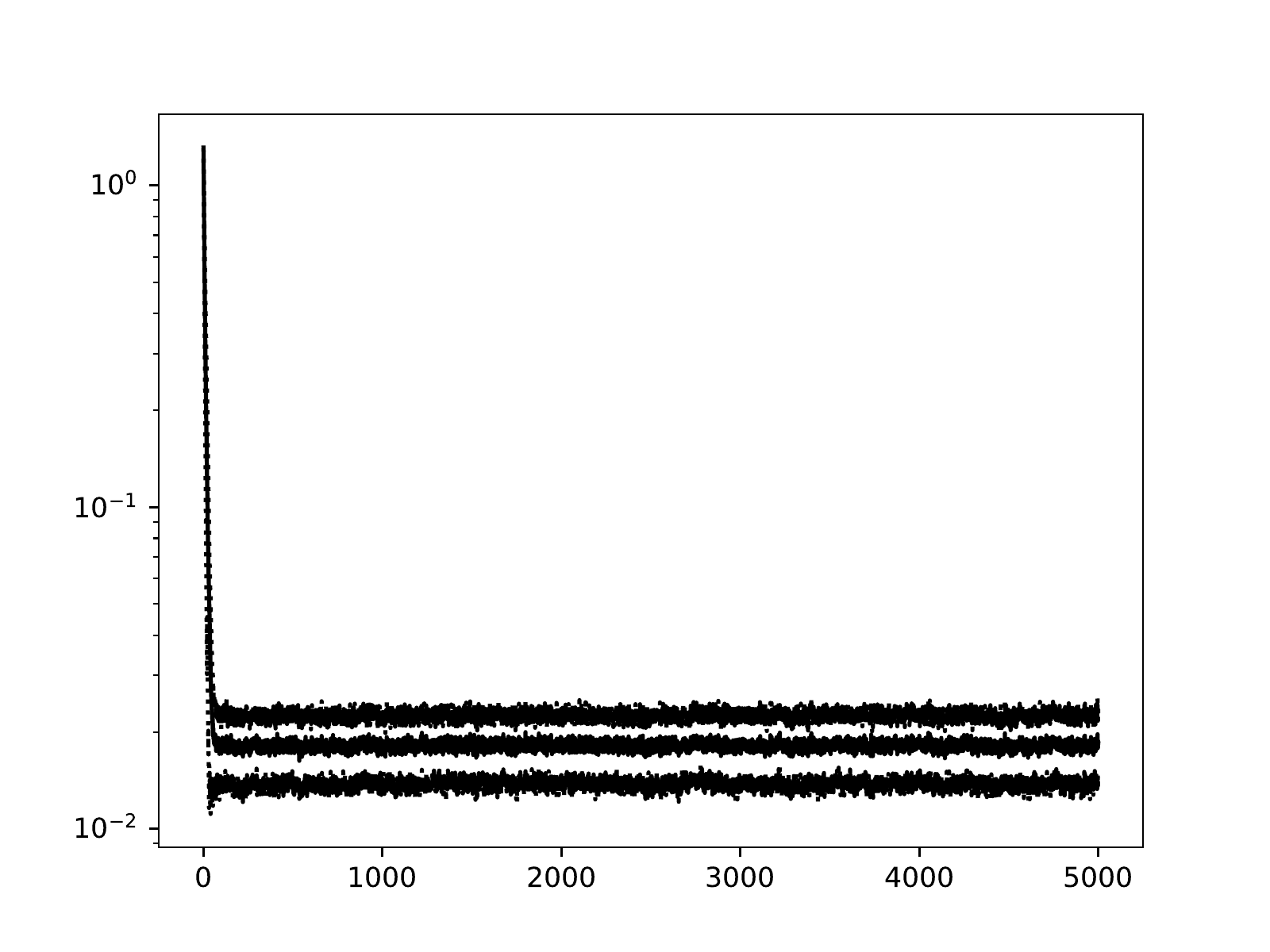} 
    \includegraphics[width=0.29\textwidth]{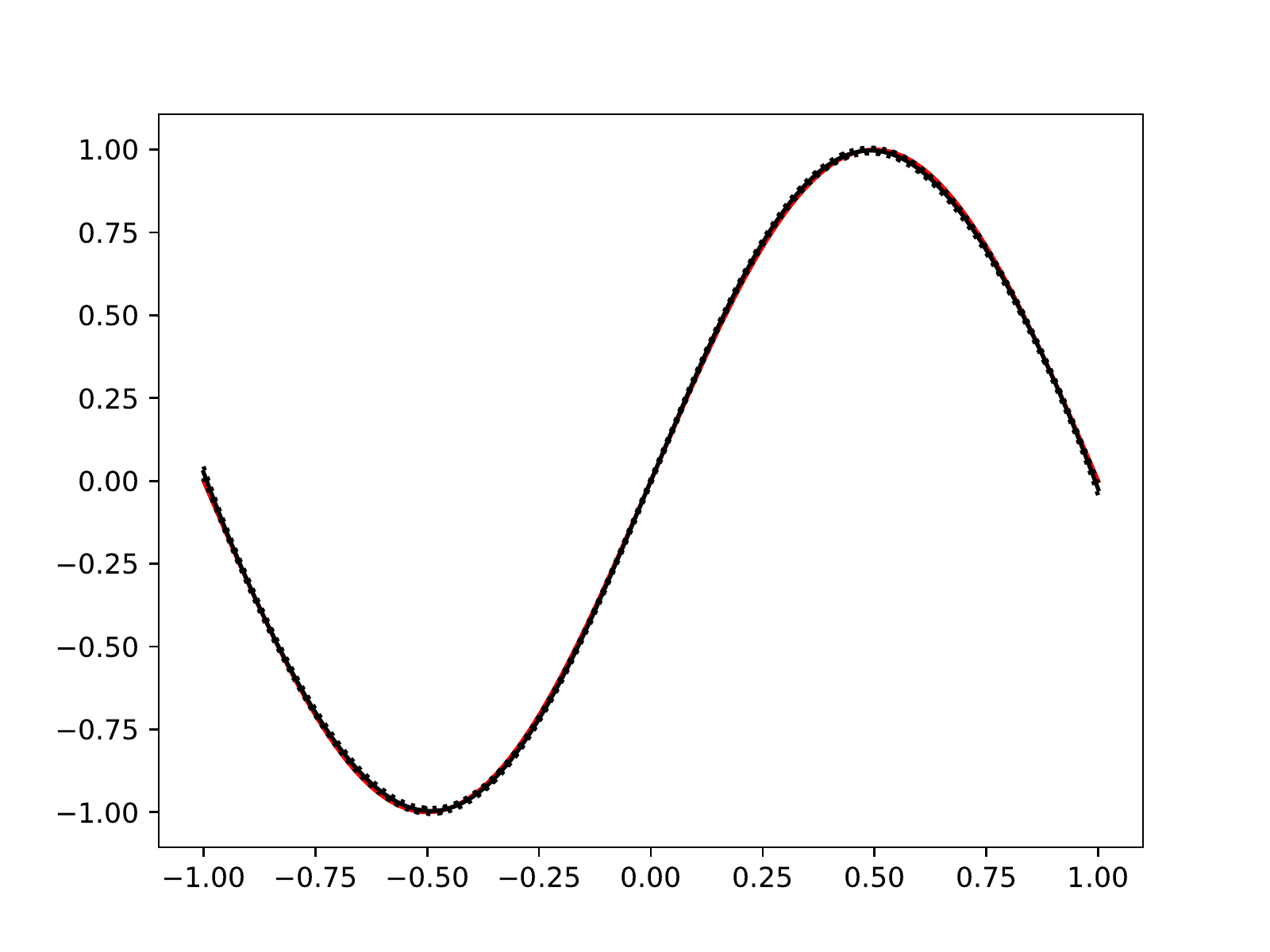}
    \includegraphics[width = 0.26\textwidth]{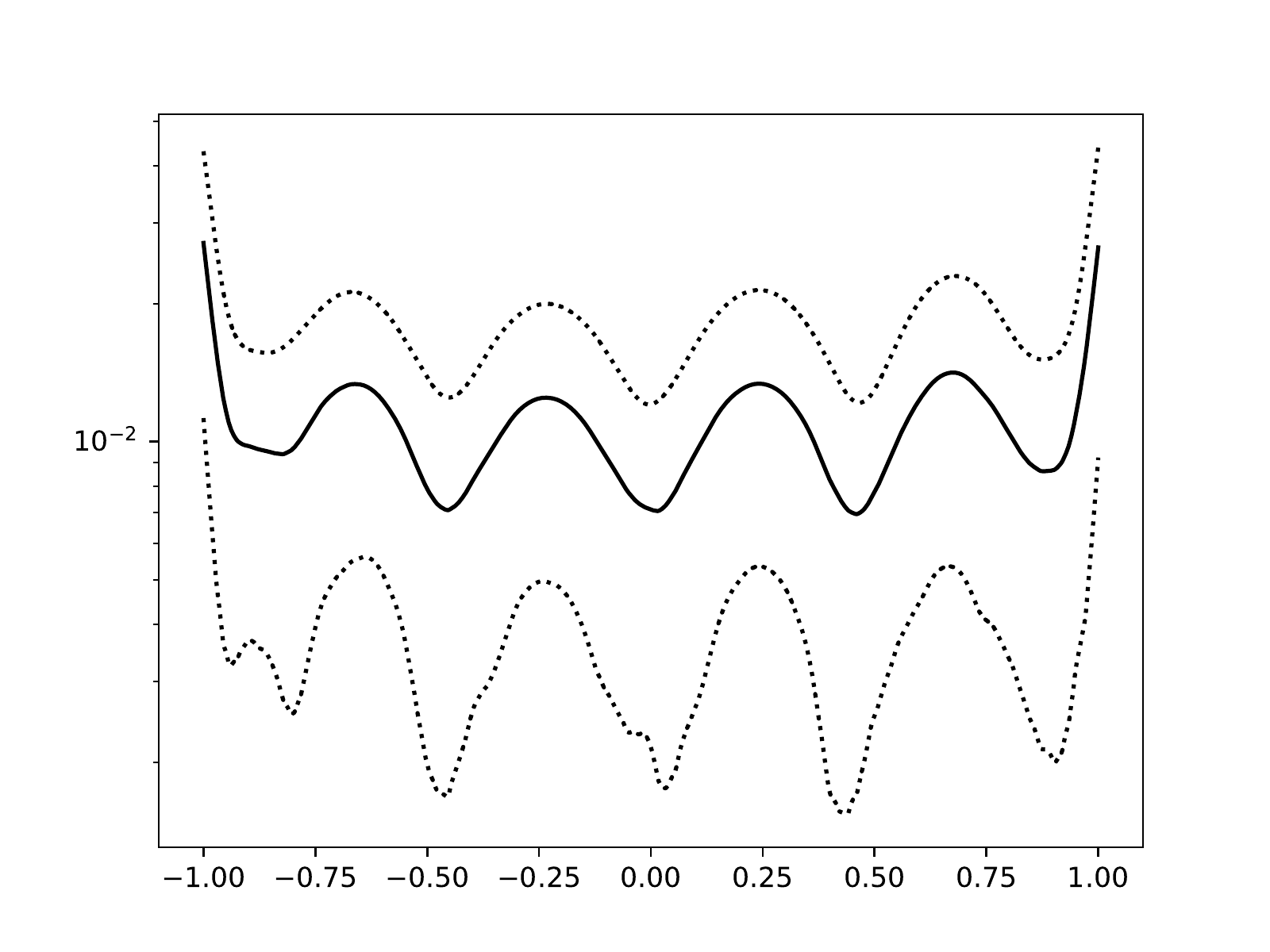}
    
    \includegraphics[width=0.29\textwidth]{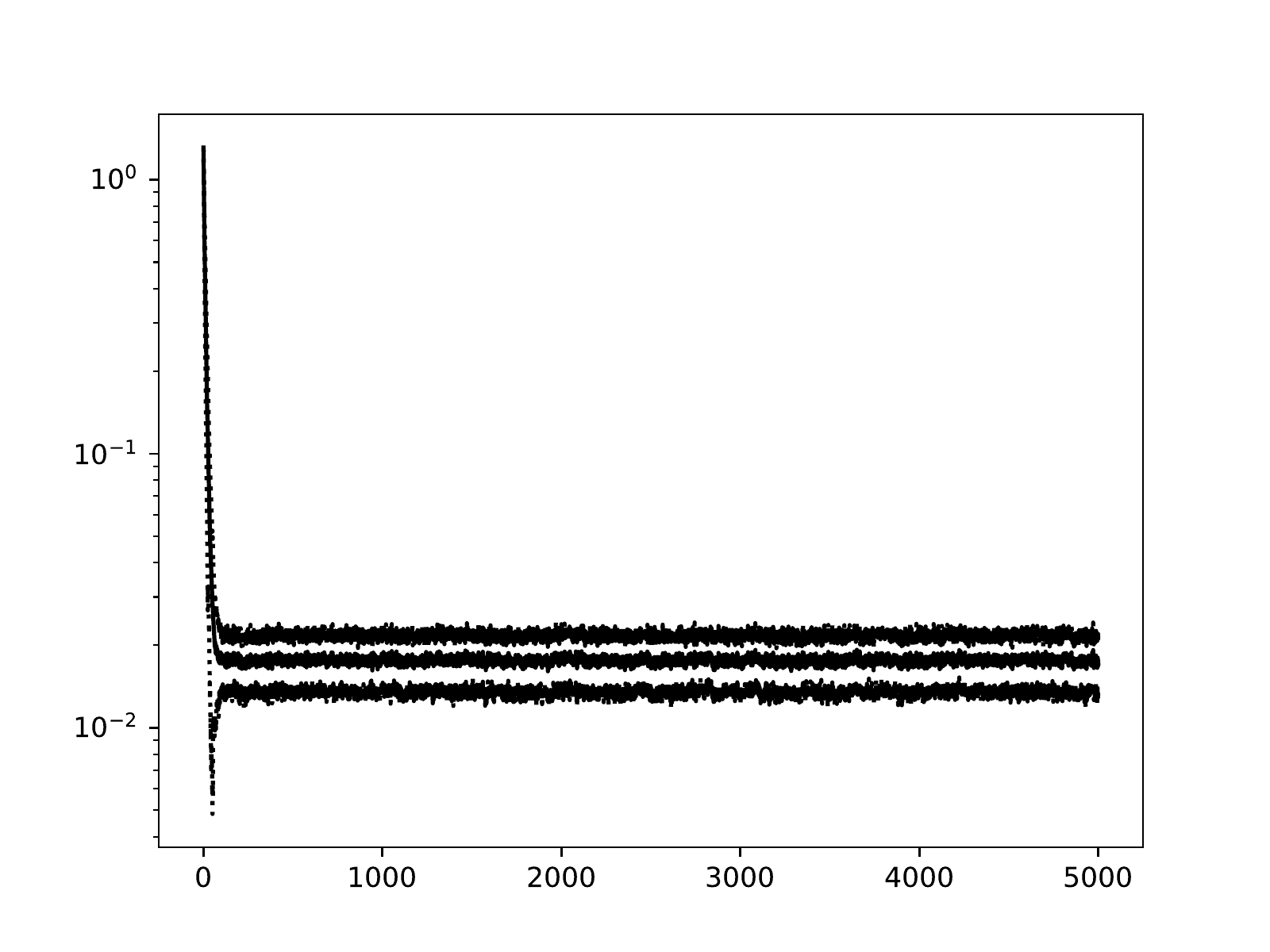} 
    \includegraphics[width=0.29\textwidth]{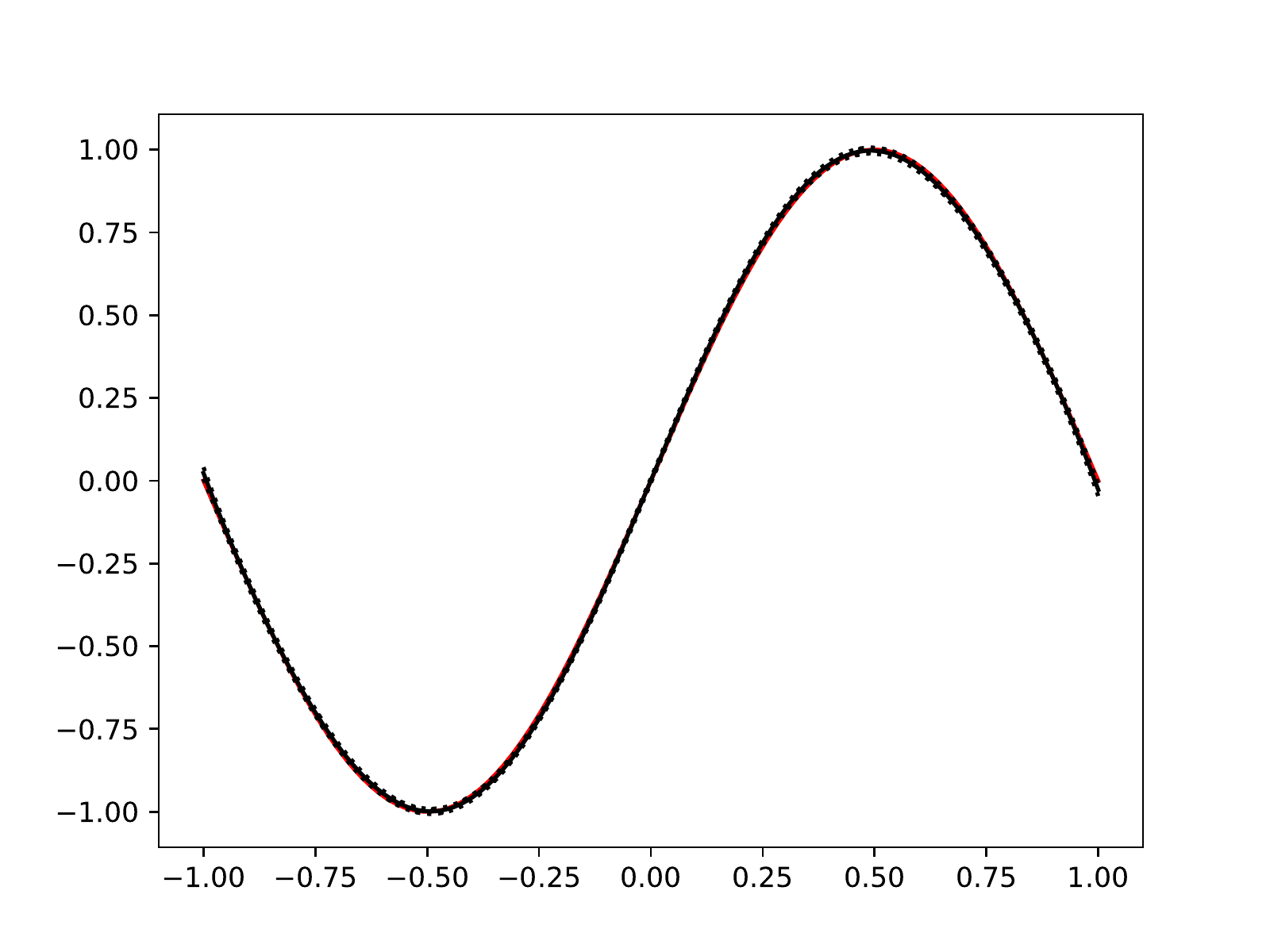}
    \includegraphics[width = 0.26\textwidth]{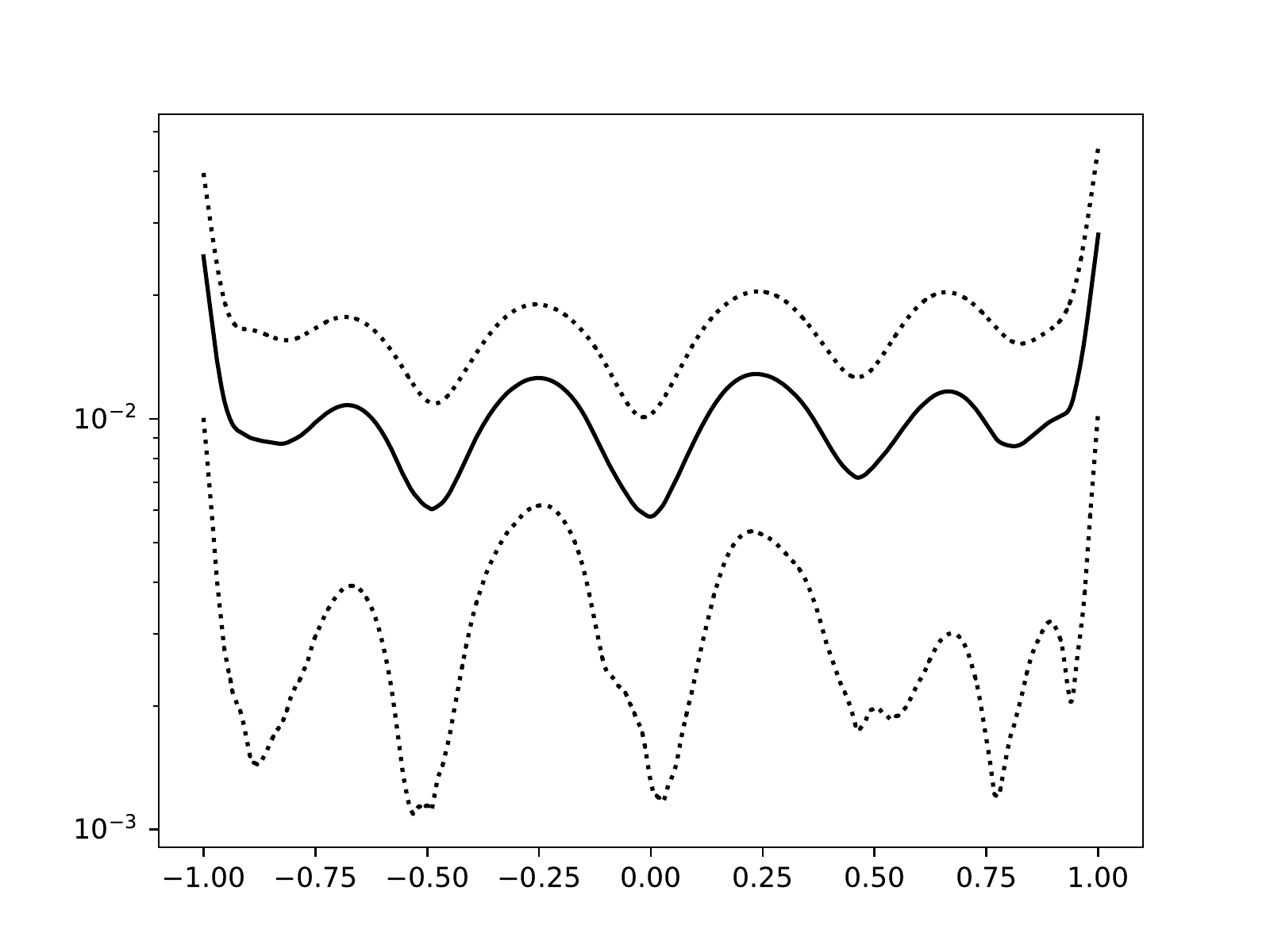}
  \caption{Estimation results of the polynomial regression problem using  stochastic gradient  descent with constant learning rate $\eta = 0.1$ (top row) and a version of stochastic gradient descent that uses the implicit midpoint rule (bottom row).   The figures depict the mean over 100 runs (black solid line), mean $\pm$ standard deviation (black dotted line). Left column: trajectory of the rel\_err over time; centre column: comparison of $\Theta$ (solid red line) and estimated polynomial; right column: estimation error in terms of abs\_err.}
    \label{fig:sgd_polyn}
\end{figure}
\subsubsection*{Setup}
In particular, we have produced artificial data $g$, by setting $\Theta := \sin(\pi \cdot)$ and choosing 
$$
\Xi(x) = \sum_{j=1}^{200}\frac{10}{1000 + (\pi j)^{3/2}}\sin(2\pi j(x-0.5))\Xi_j \qquad \qquad (x \in [-1,1])
$$
and i.i.d.\ random variables $\Xi_1,\ldots, \Xi_{200} \sim \mathrm{N}(0,1^2)$. Note that $\Xi$ is a Gaussian random field given through the truncated Karhunen-Lo\`eve expansion of a covariance operator that is related to the Mat\'ern family, see, e.g., \citet{Lindgren2011}.

We show $\Theta$ and $g$ in Figure \ref{fig:truth_polyn}. For our estimation, we set $\alpha := 10^{-4}$ and use the $K=9$ Legendre polynomials with degrees $0,\ldots,8$. We employ the stochastic gradient process with constant learning rate, using either a reflected diffusion process or a pure Markov jump process for the index process $(V_t)_{t \geq 0}$. We discretize the gradient flow using the implicit midpoint rule: an ODE $z' = q(z), z(0) = z_0$ is then discretized with stepsize $h > 0$ by successively solving the implicit formula $$z_k = z_{k-1} + \frac{h}{2}q(z_k) + \frac{h}{2}q(z_{k-1}) \qquad (k \in \mathbb{N}).$$ In our experiments, we choose $h = 0.1$. We use Algorithms~\ref{alg:MPJ} and \ref{alg:RBM} to discretize the index processes with constant stepsize $t(\cdot)-t(\cdot -1) = 10^{-2}.$
We perform $J := 100$ repeated runs for each of the considered settings for $N := 5\cdot 10^4$ time steps and thus, obtain a family of trajectories $(\theta^{(j,n)})_{n = 1,\ldots,N, j=1,\ldots,J}$. In each case, we choose the initial values $V(0) := 0$ and the $\theta^{(j,0)} := (0.5,\ldots, 0.5).$ 

We study the distance of the estimated polynomial to the true function $\Theta$ 
by the relative error: $$\mathrm{rel\_err}_{n,j} := \frac{\sum_{l = 1}^{L}\left(\Theta(x_l)- \sum_{k = 1}^K \theta_k^{(j,n)}\ell_k(x_l) \right)^2}{\sum_{l' = 1}^{L}\Theta(x_{l'})^2},$$ for trajectory $j \in \{1,\ldots, J\}$ and time step $n \in \{1,\ldots,N\}$. Here $(x_l)_{l=1}^{L}$ are $L := 10^3$ equispaced points in $[-1,1]$. Moreover, we compare the estimated polynomial to the  true function $\Theta$ by 
$$\mathrm{abs\_err}_{j,x} := \left\lvert\Theta(x)- \sum_{k = 1}^K \theta_k^{(j,N)}\ell_k(x) \right\rvert $$ for trajectory $j \in \{1,\ldots, J\}$ at position $x \in [-1,1].$ In each case, we study mean and standard deviation (StD) computed over the $100$ runs.

\begin{figure}[htb]
    \centering
    \includegraphics[width = 0.26\textwidth]{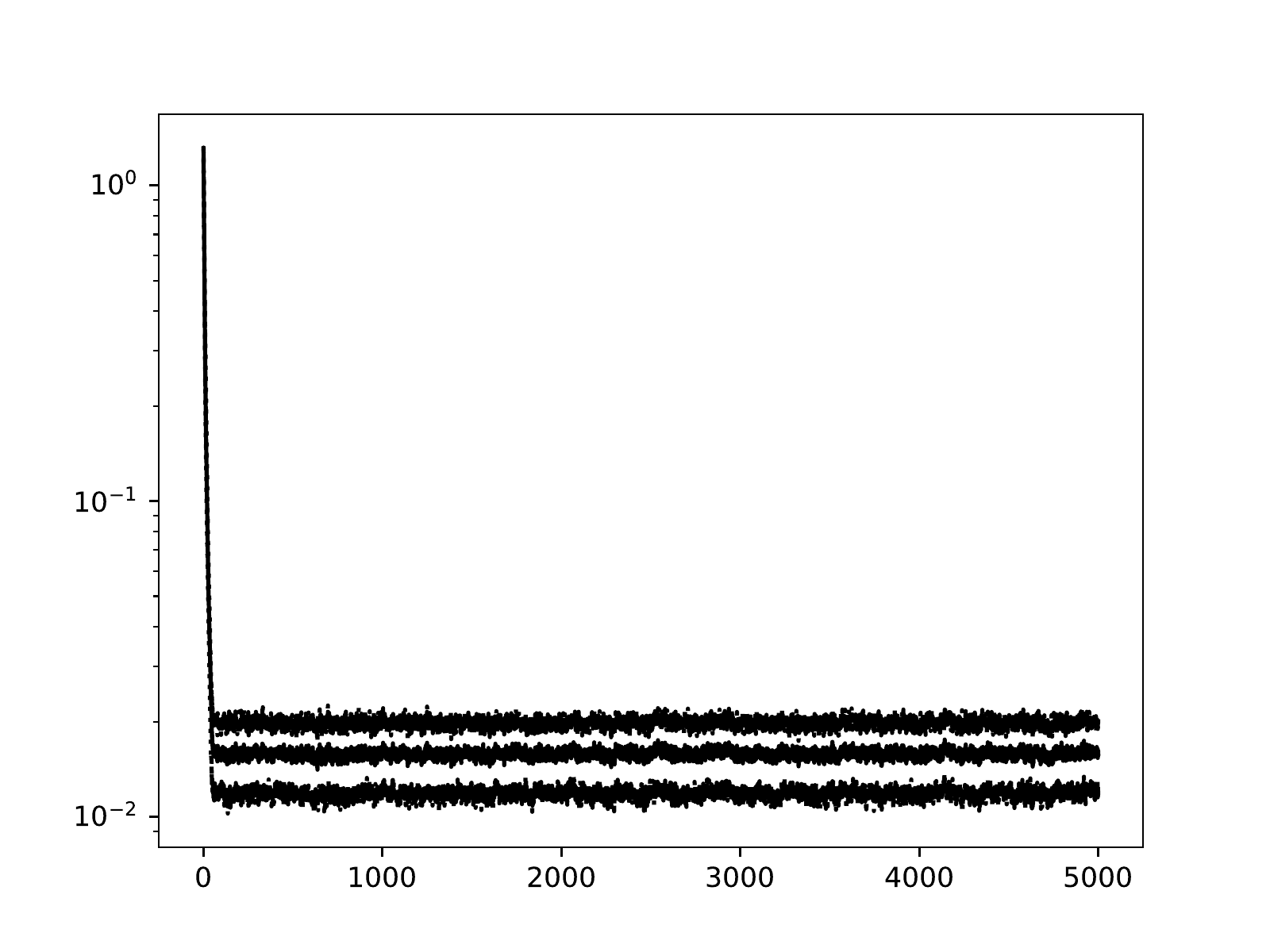}
    \includegraphics[width = 0.26\textwidth]{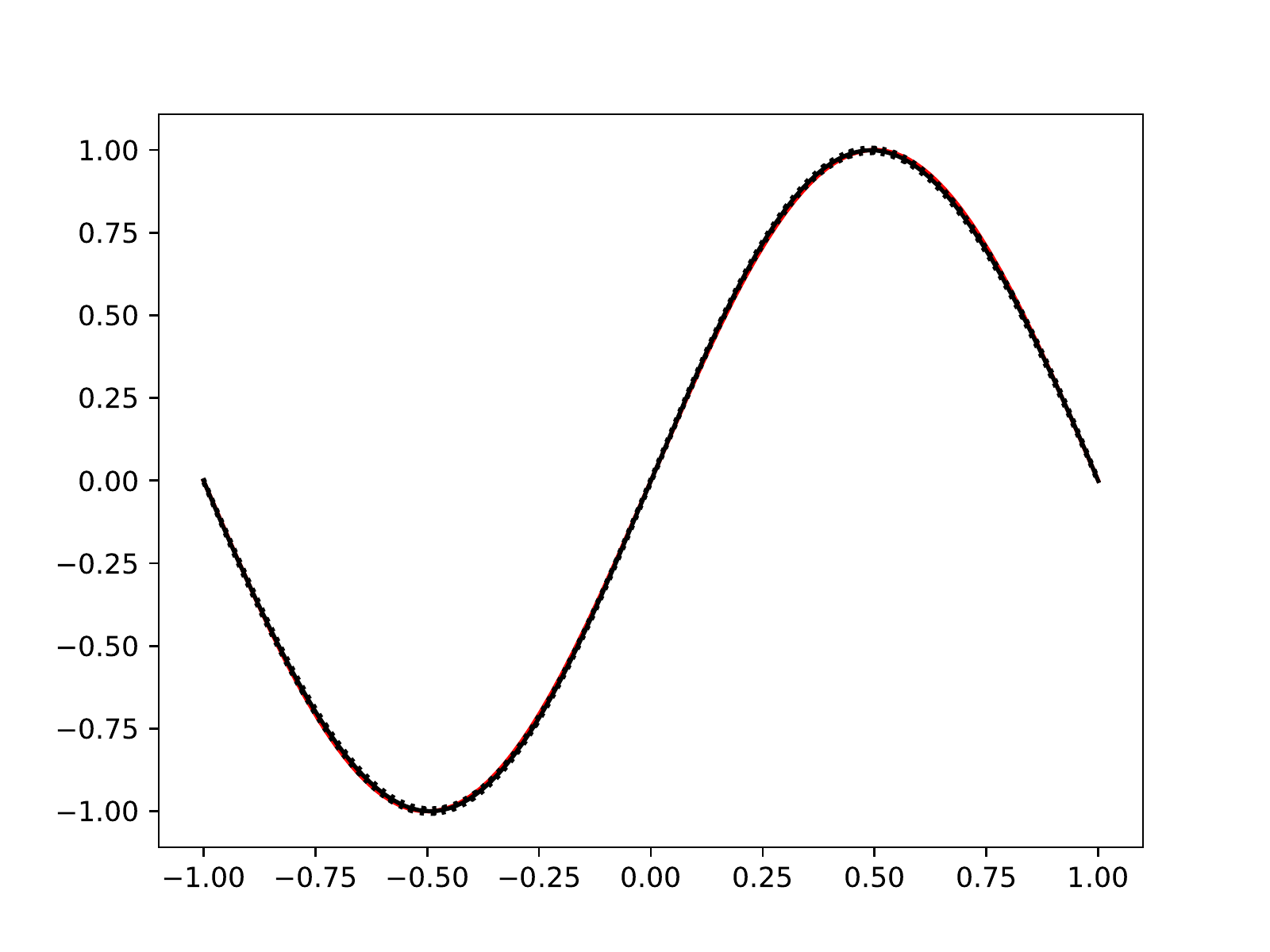}
    \includegraphics[width = 0.26\textwidth]{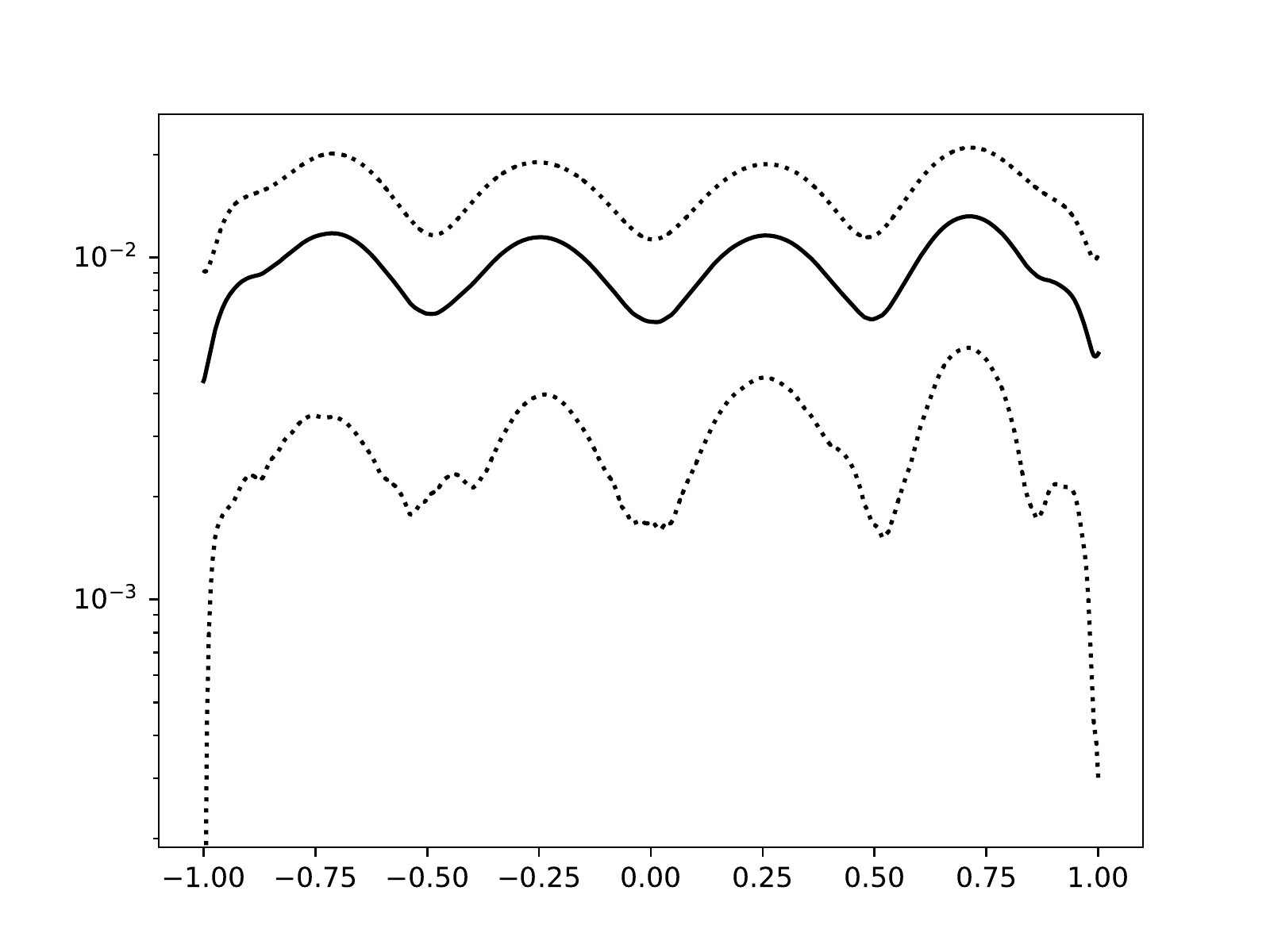}
    
      \includegraphics[width = 0.26\textwidth]{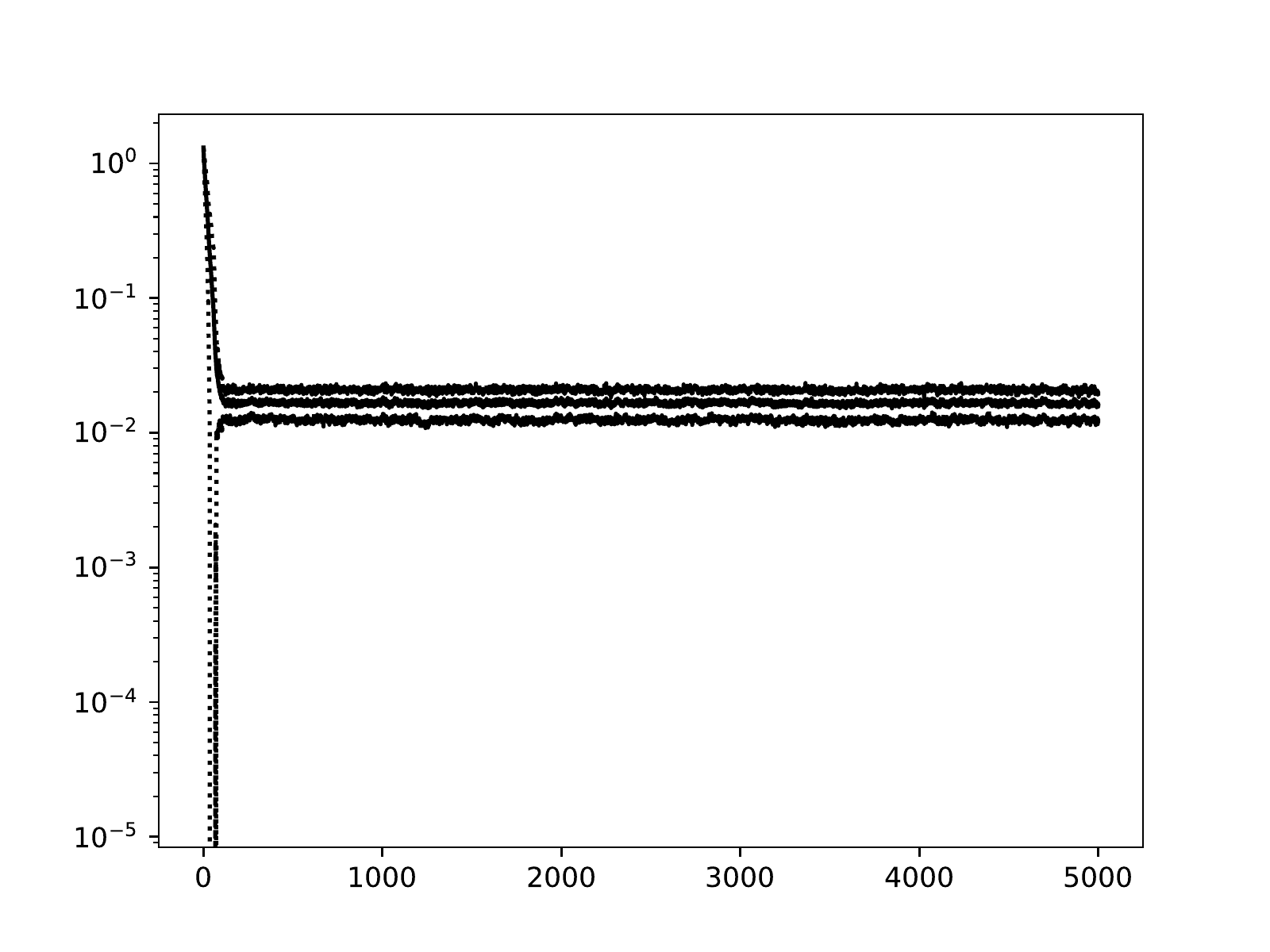}
    \includegraphics[width = 0.26\textwidth]{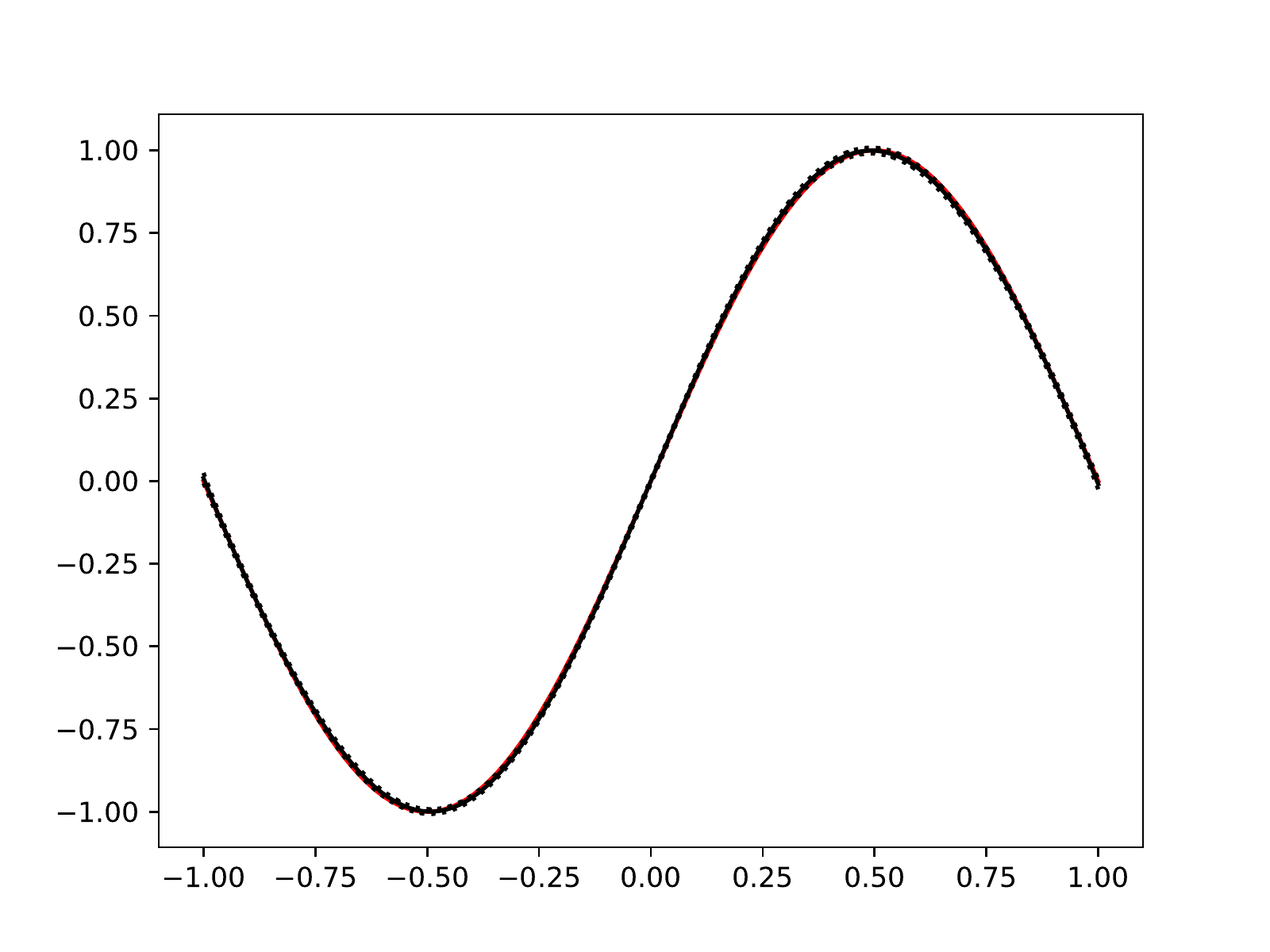}
    \includegraphics[width = 0.26\textwidth]{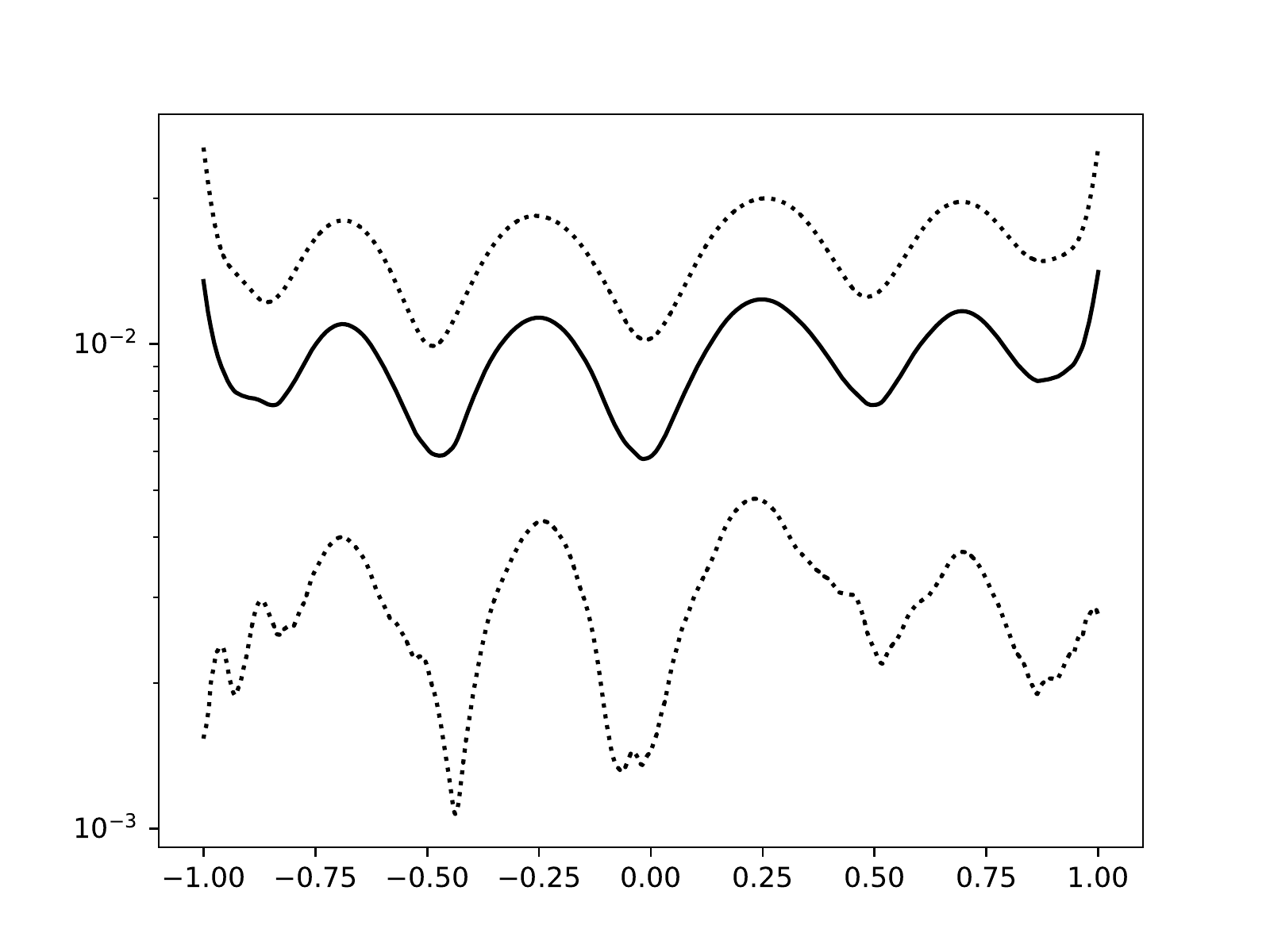}
    
      \includegraphics[width = 0.26\textwidth]{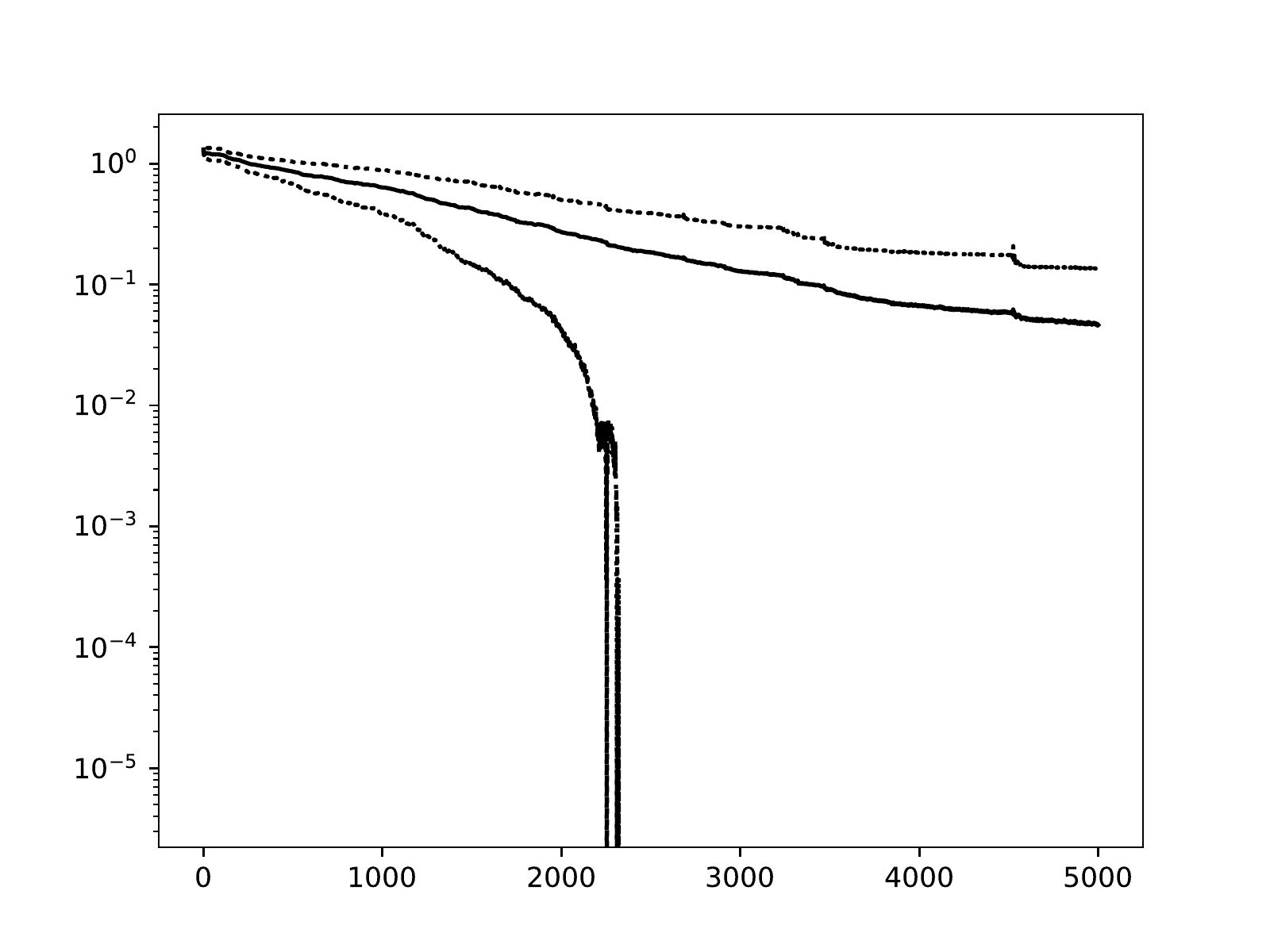}
    \includegraphics[width = 0.26\textwidth]{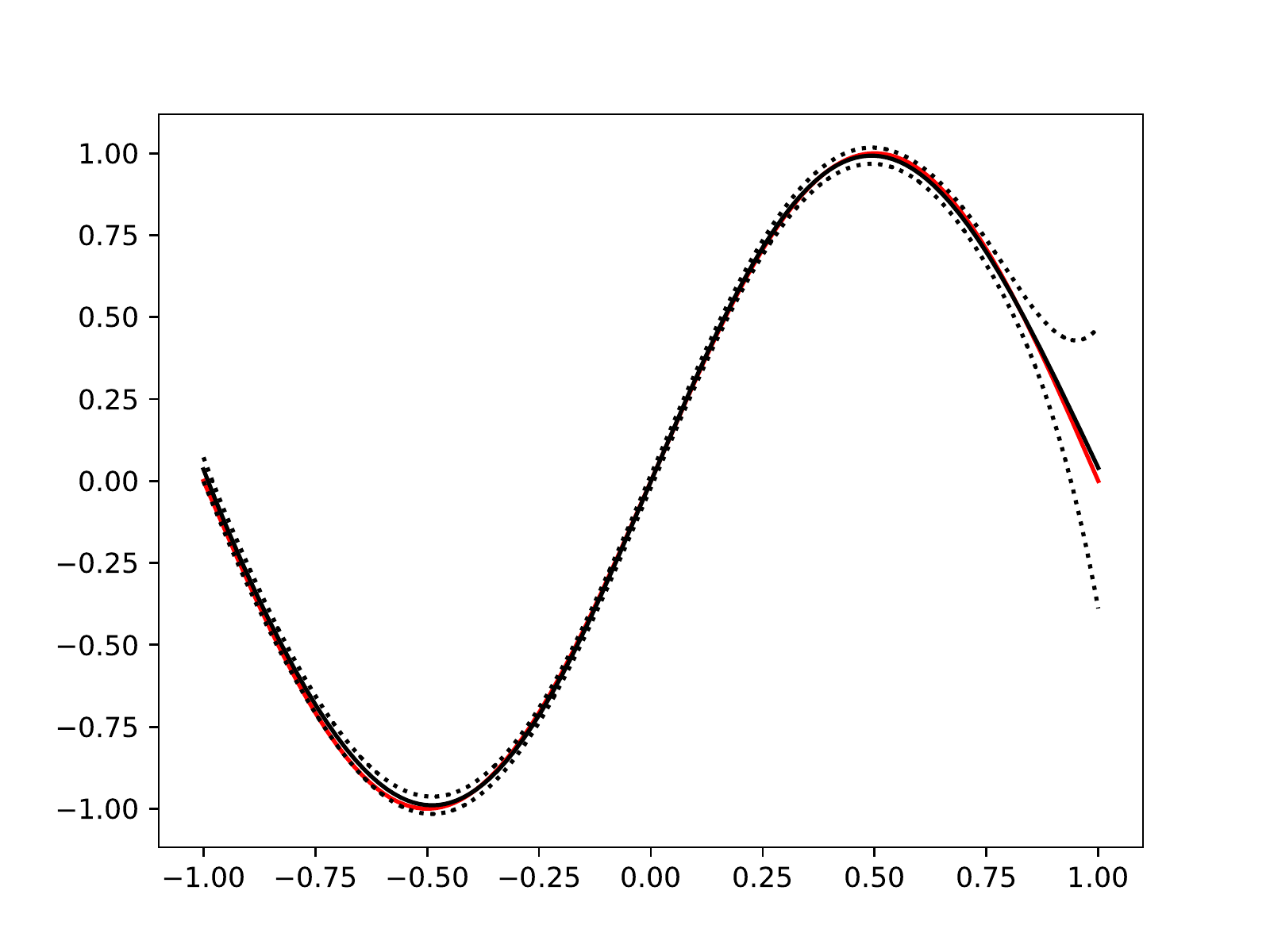}
    \includegraphics[width = 0.26\textwidth]{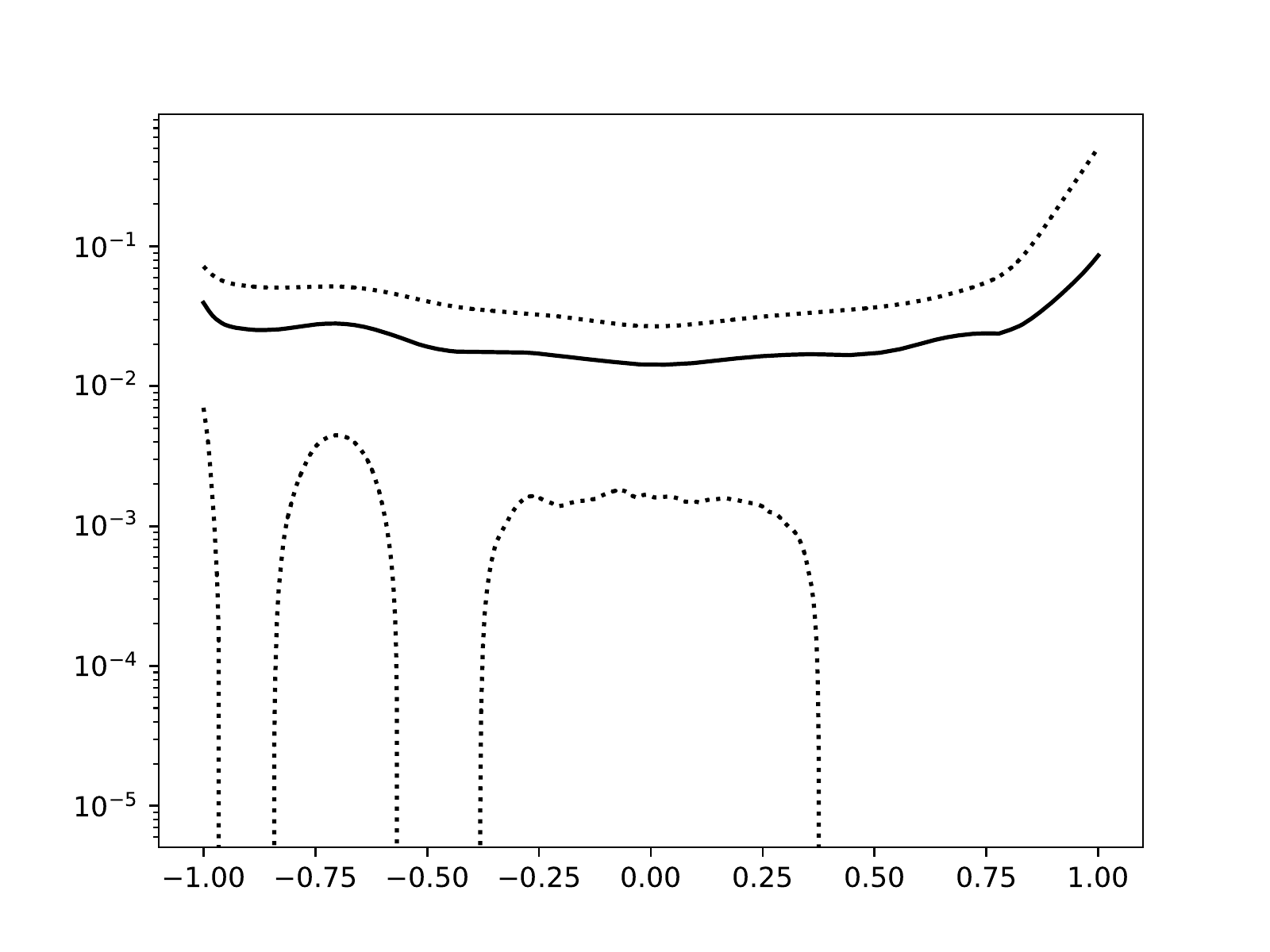}
\caption{Estimation results of the polynomial regression problem using the stochastic gradient process with reflected Brownian motion process with $\sigma = 5$ (top row), $\sigma = 0.5$ (centre row), and $\sigma = 0.05$ (bottom row).   The figures depict the mean over 100 runs (black solid line), mean $\pm$ standard deviation (black dotted line). Left column: trajectory of the rel\_err over time; centre column: comparison of $\Theta$ (solid red line) and estimated polynomial; right column: estimation error in terms of abs\_err.}
    \label{fig:Diff_results}
\end{figure}
\subsubsection*{Results and discussion}
For the polynomial regression problem we now study:
\begin{itemize}
\item stochastic gradient descent, as given in \eqref{Eq:SGD_discrete_time}, with constant learning rate $\eta_{(\cdot)} = h = 0.1$ (Figure~\ref{fig:sgd_polyn} top row),
    \item stochastic gradient descent algorithm, for which the forward Euler update is replaced by an implicit midpoint rule update, with constant learning rate $\eta_{(\cdot)} = h = 0.1$ (Figure~\ref{fig:sgd_polyn} bottom row),
    \item the stochastic gradient process with reflected Brownian motion as an index process with standard deviation $\sigma \in \{5, 0.5, 0.05\}$ (Figure~\ref{fig:Diff_results}), and
     \item the stochastic gradient process with Markov pure jump process as an index process with rate parameter $\lambda \in \{10, 1, 0.1, 0.01\}$ (Figure~\ref{fig:MJP_results}).
\end{itemize}
In addition to those plots, we give means and standard deviations of the relative errors at the terminal state of the iterations in Table~\ref{Table_Results_polynomial}. To compare the convergence behavior of the different methods, we plot the rel\_err within the first 2000 discrete time steps in Figure~\ref{fig:error_comparison}.
\begin{figure}[htb]
    \centering
     \includegraphics[width = 0.26\textwidth]{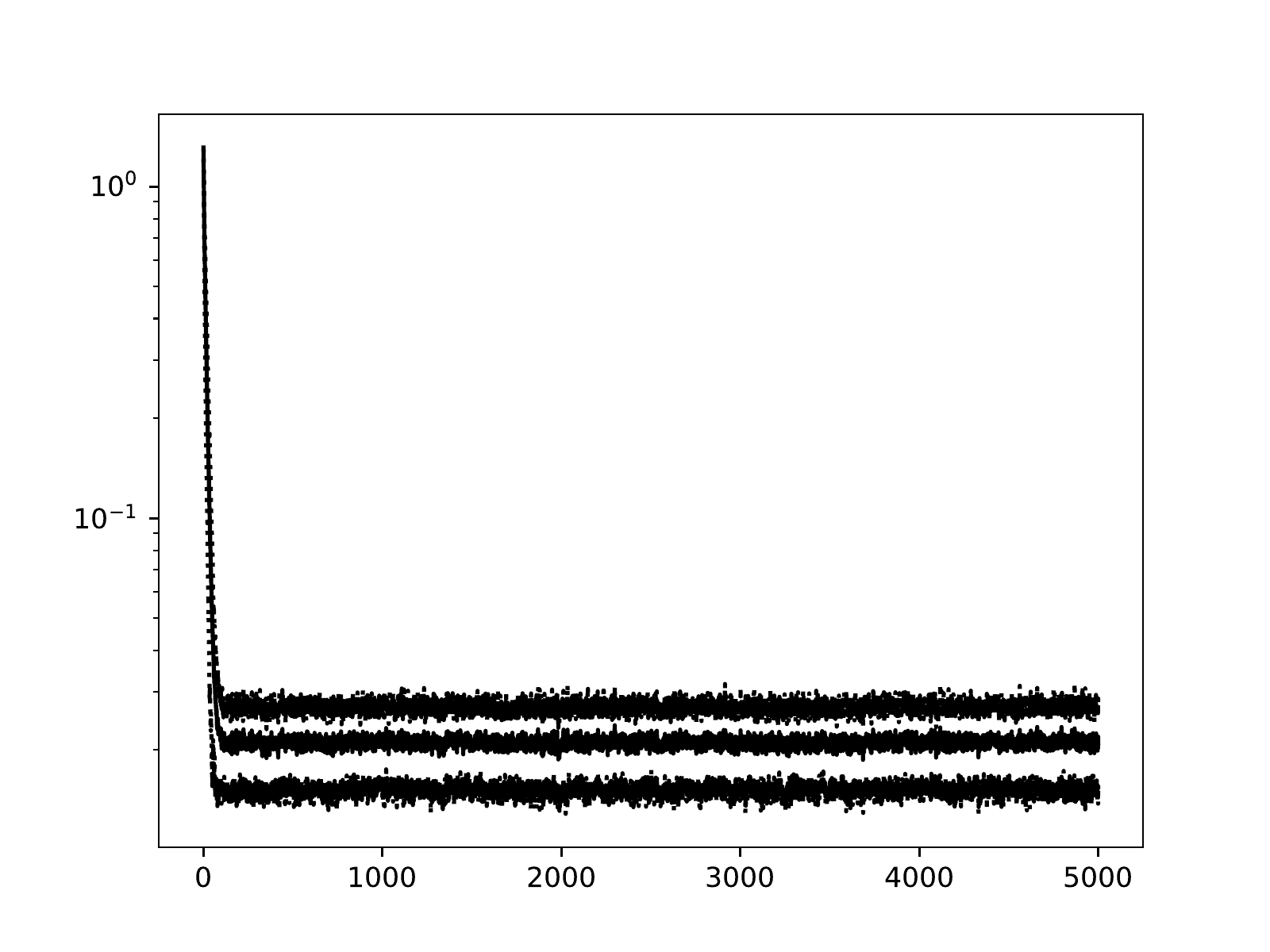}
     \includegraphics[width = 0.26\textwidth]{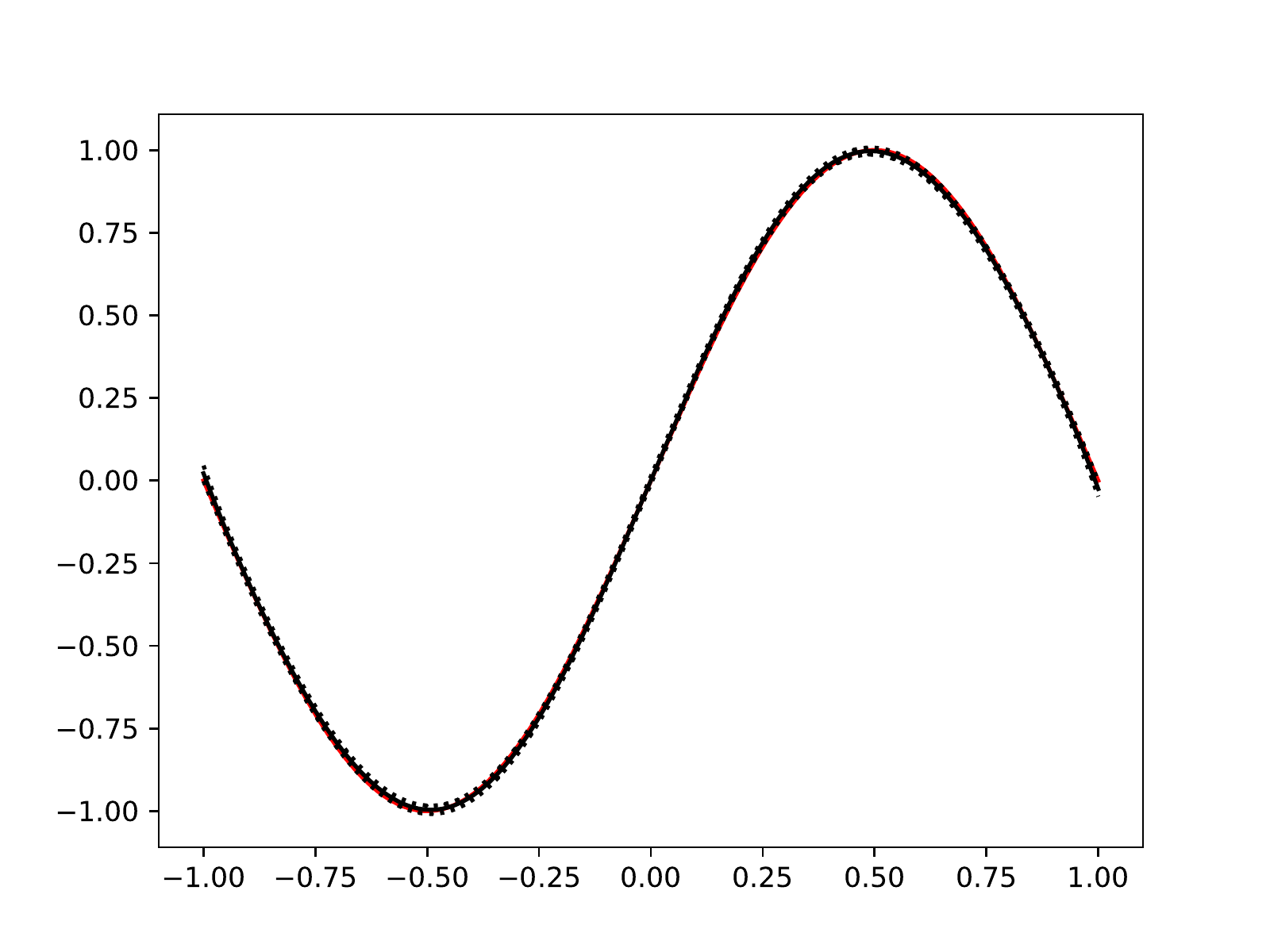}
      \includegraphics[width = 0.26\textwidth]{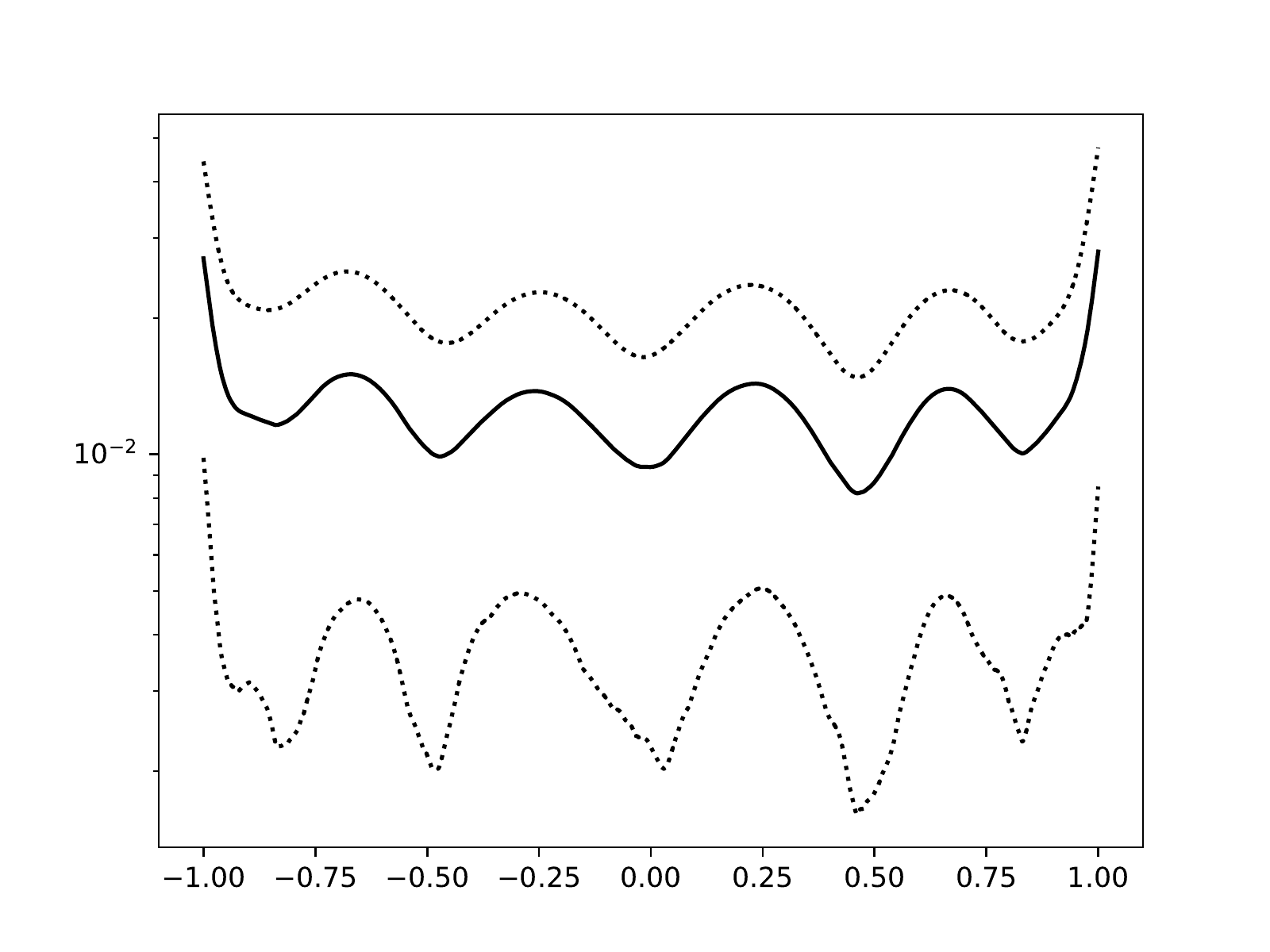}
    \includegraphics[width = 0.26\textwidth]{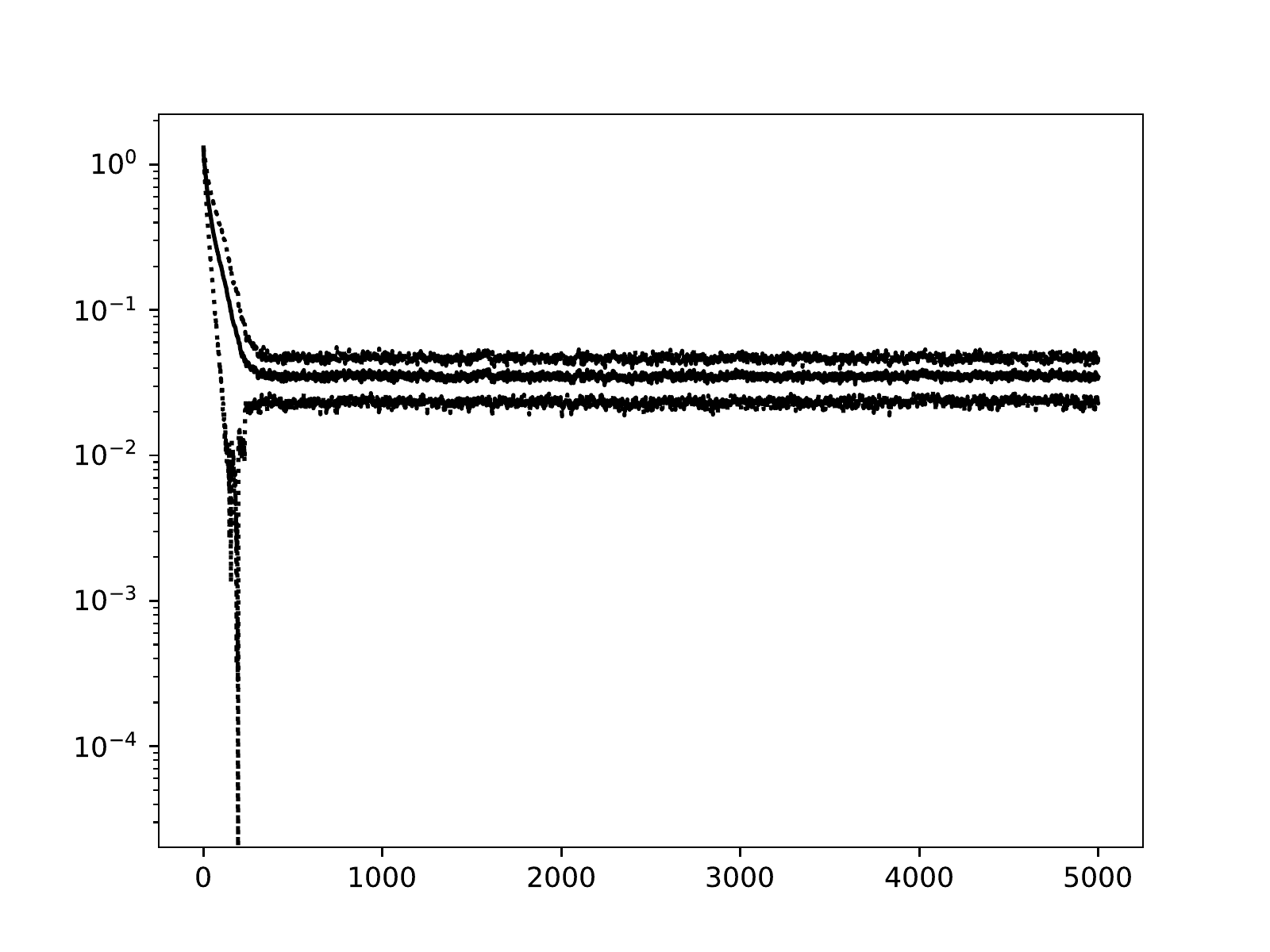}
     \includegraphics[width = 0.26\textwidth]{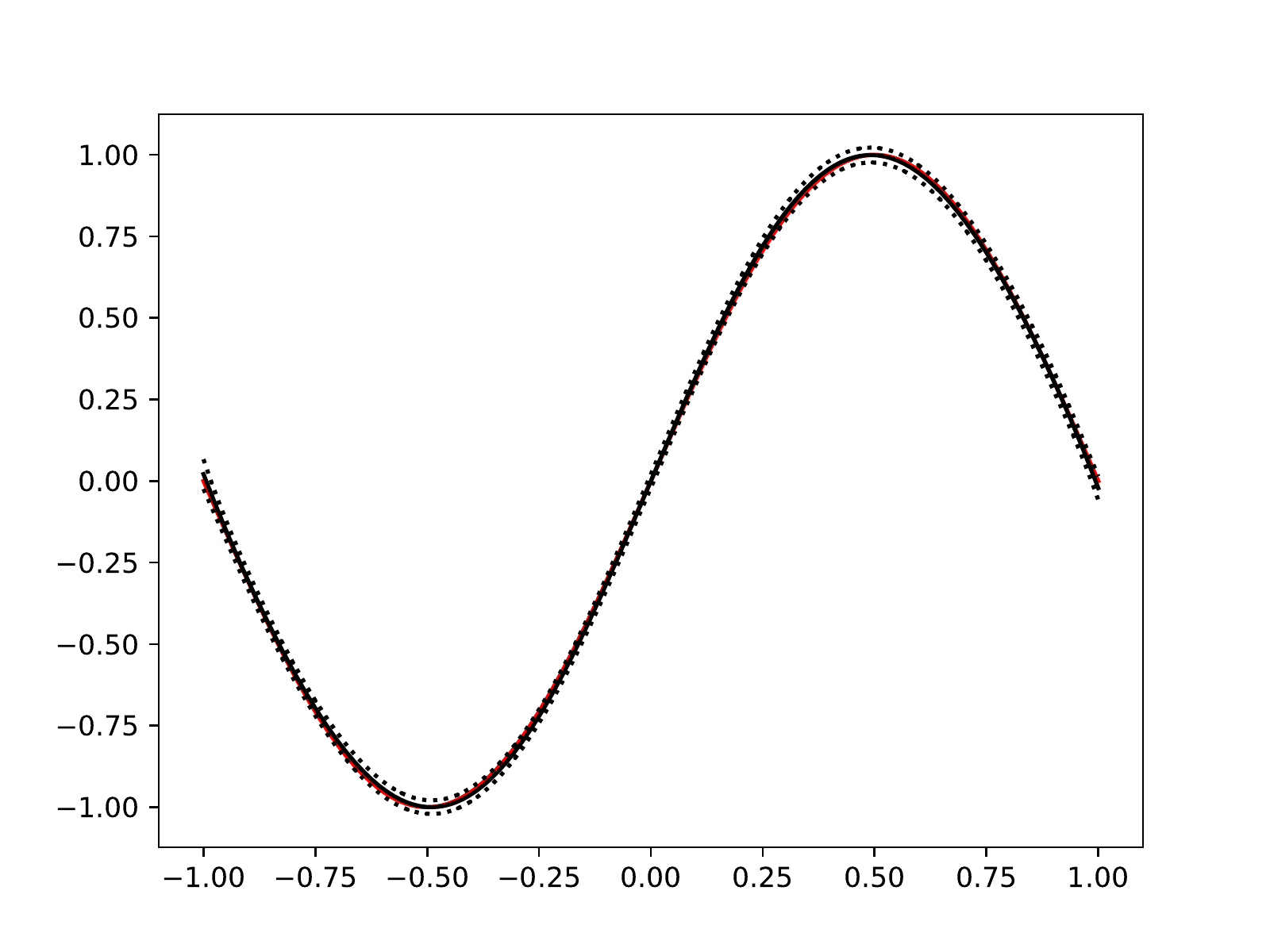}
      \includegraphics[width = 0.26\textwidth]{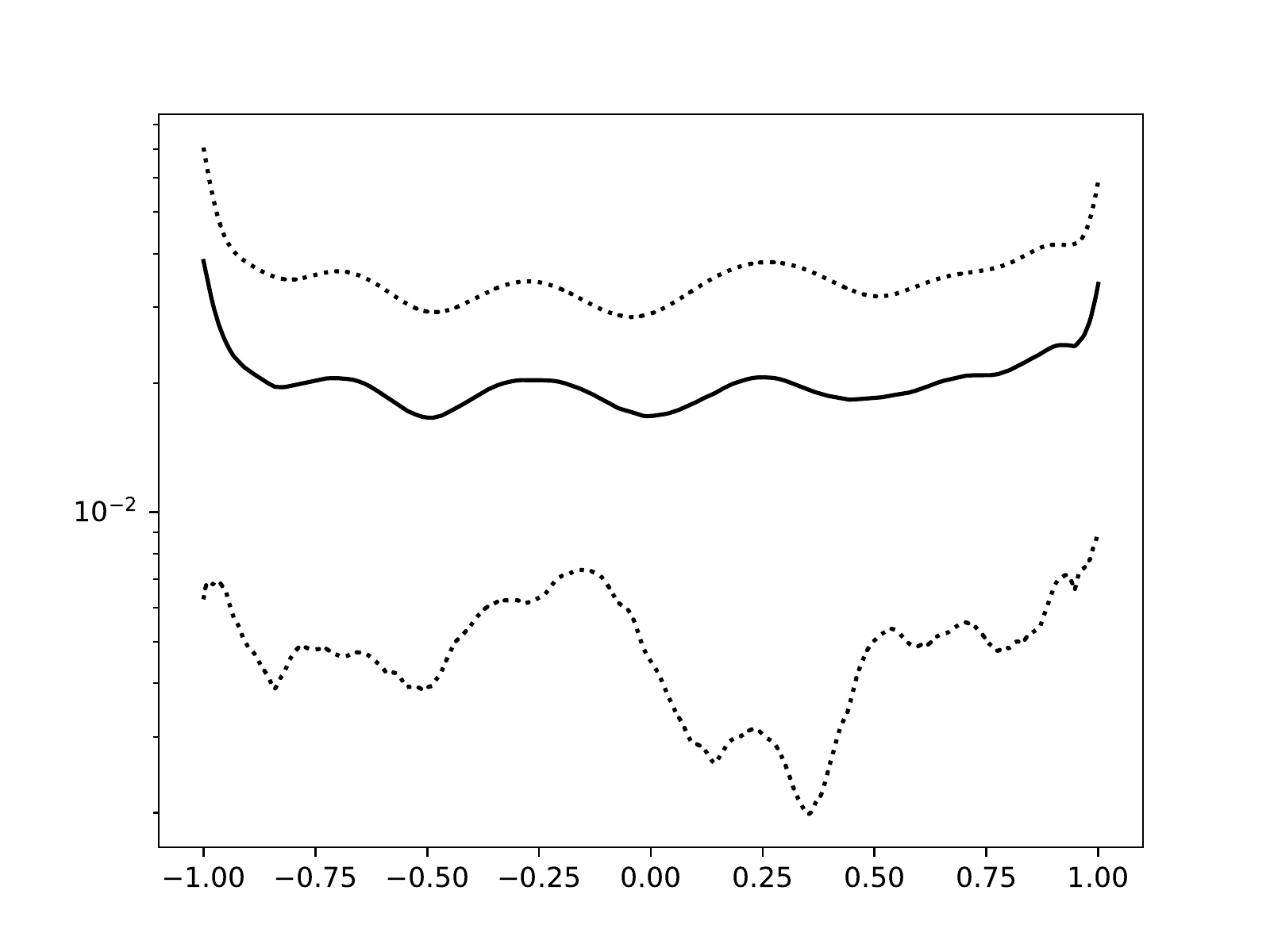}
          \includegraphics[width = 0.26\textwidth]{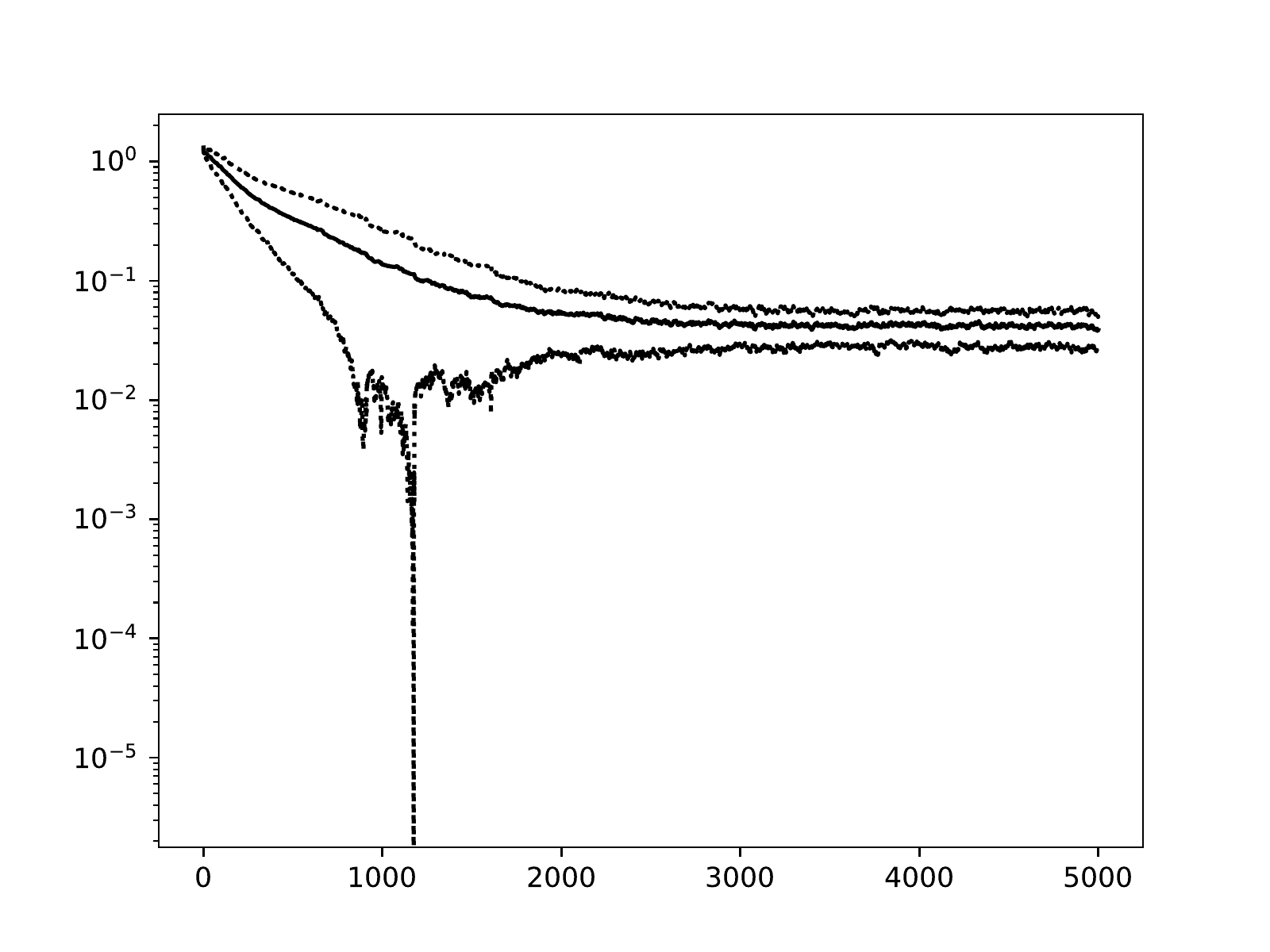}
     \includegraphics[width = 0.26\textwidth]{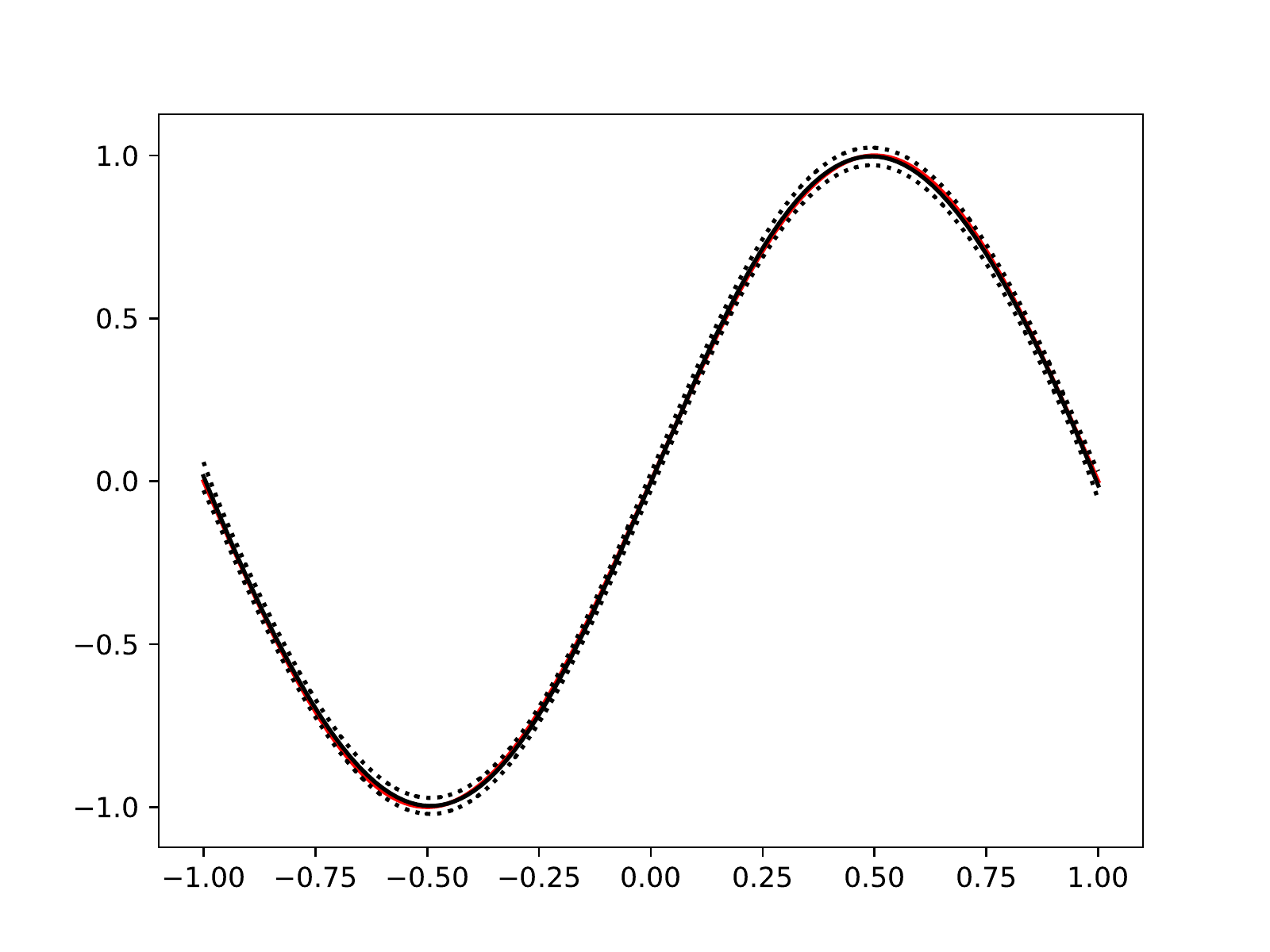}
      \includegraphics[width = 0.26\textwidth]{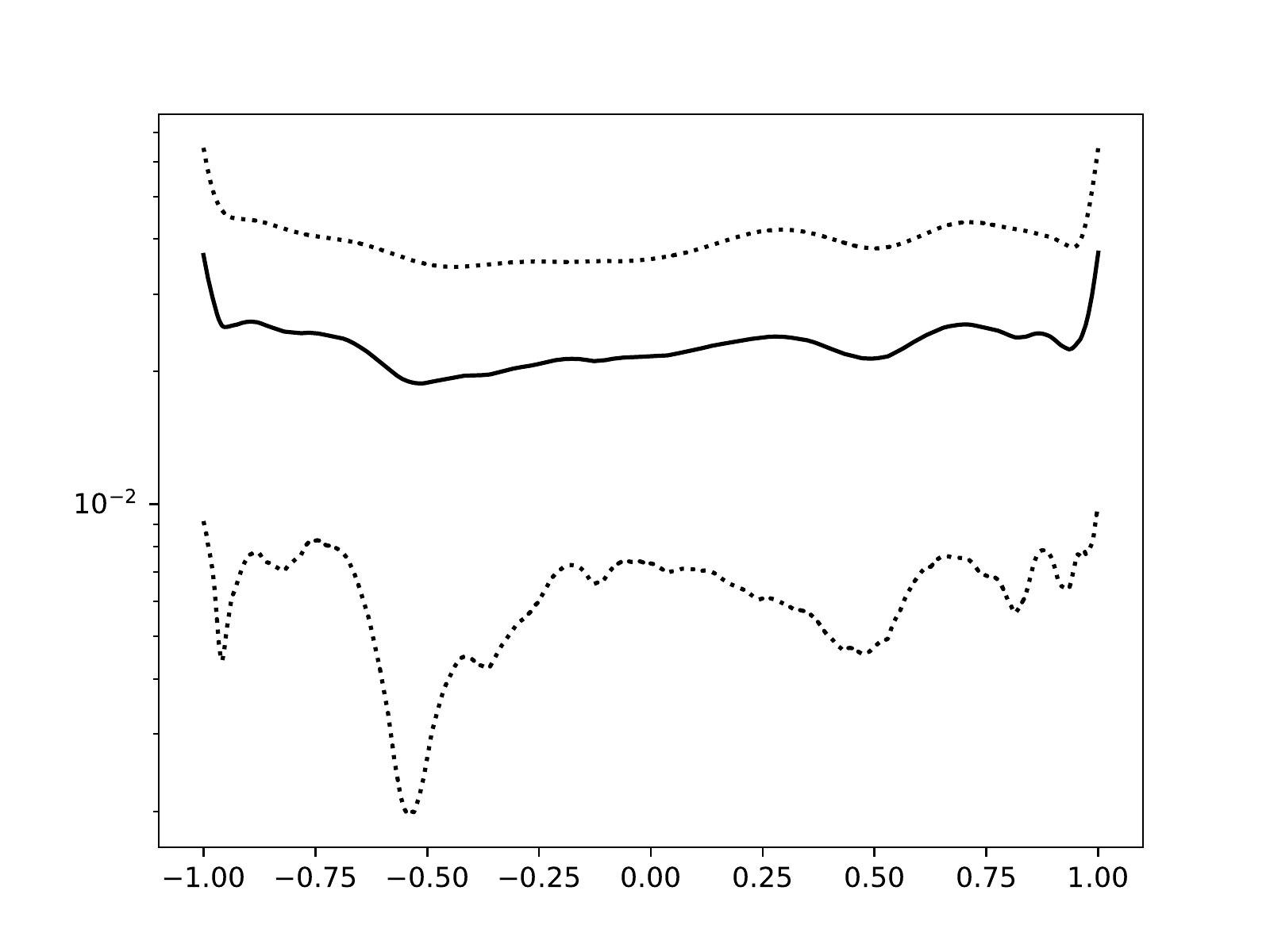}
          \includegraphics[width = 0.26\textwidth]{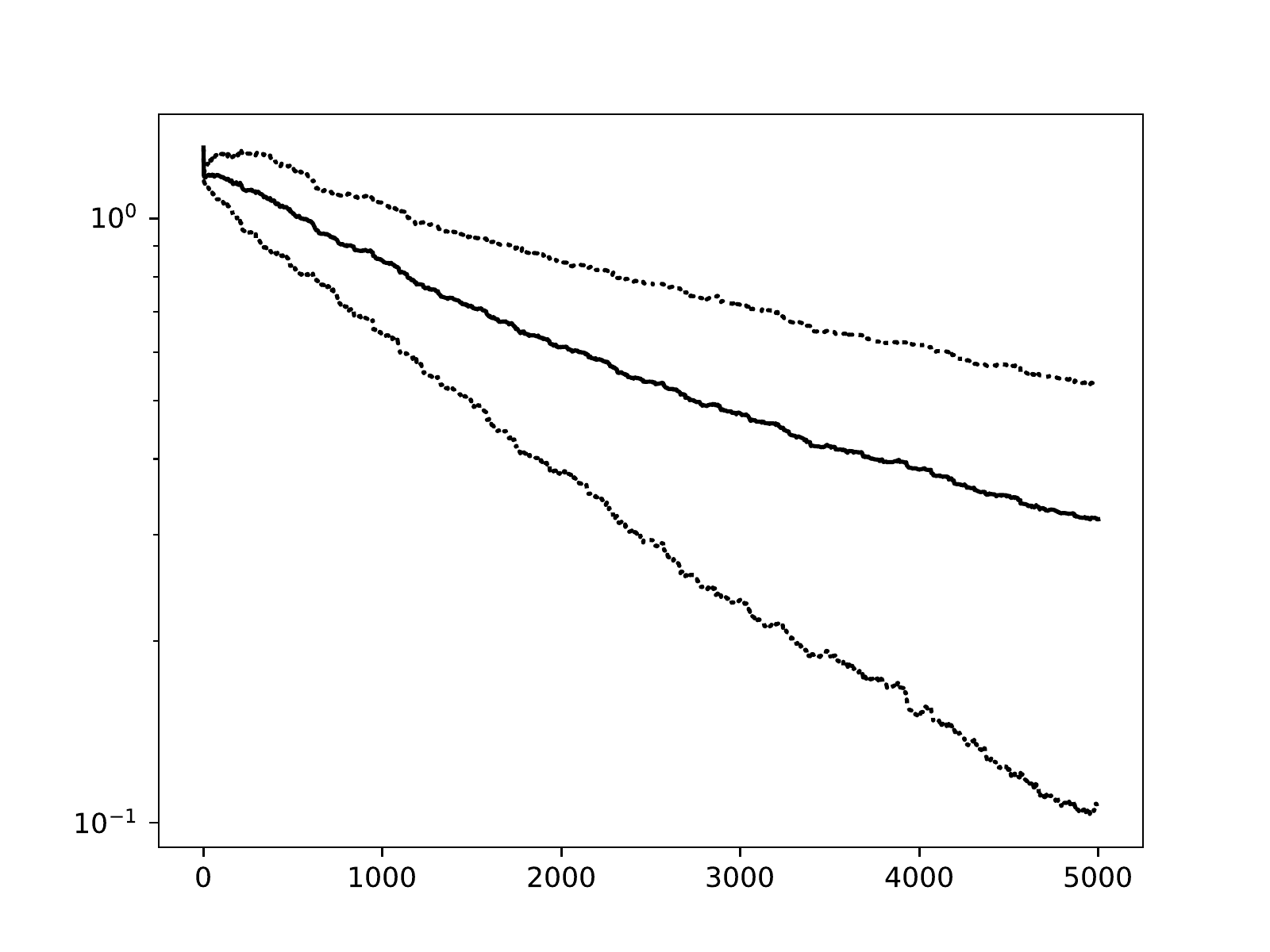}
     \includegraphics[width = 0.26\textwidth]{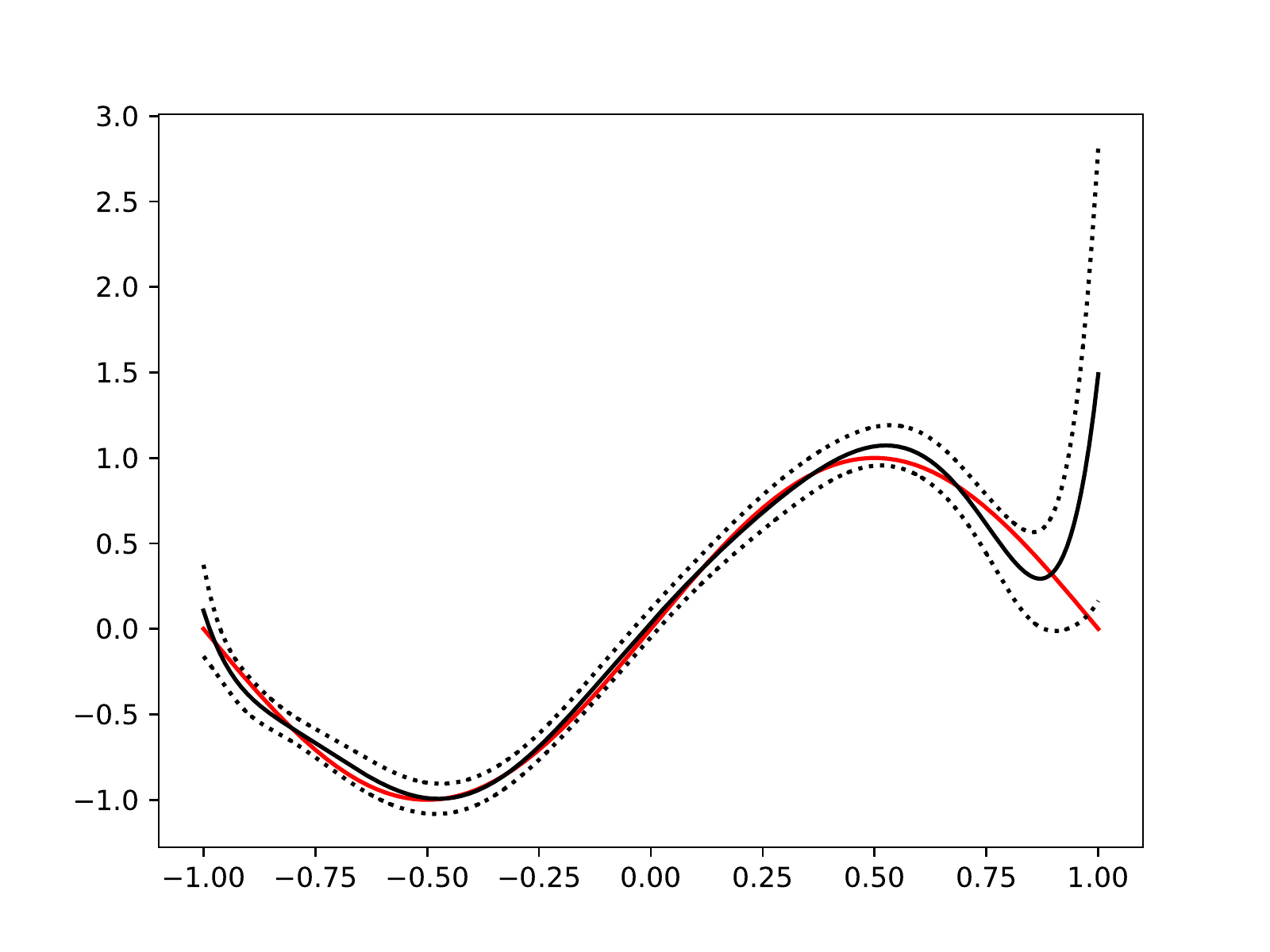}
      \includegraphics[width = 0.26\textwidth]{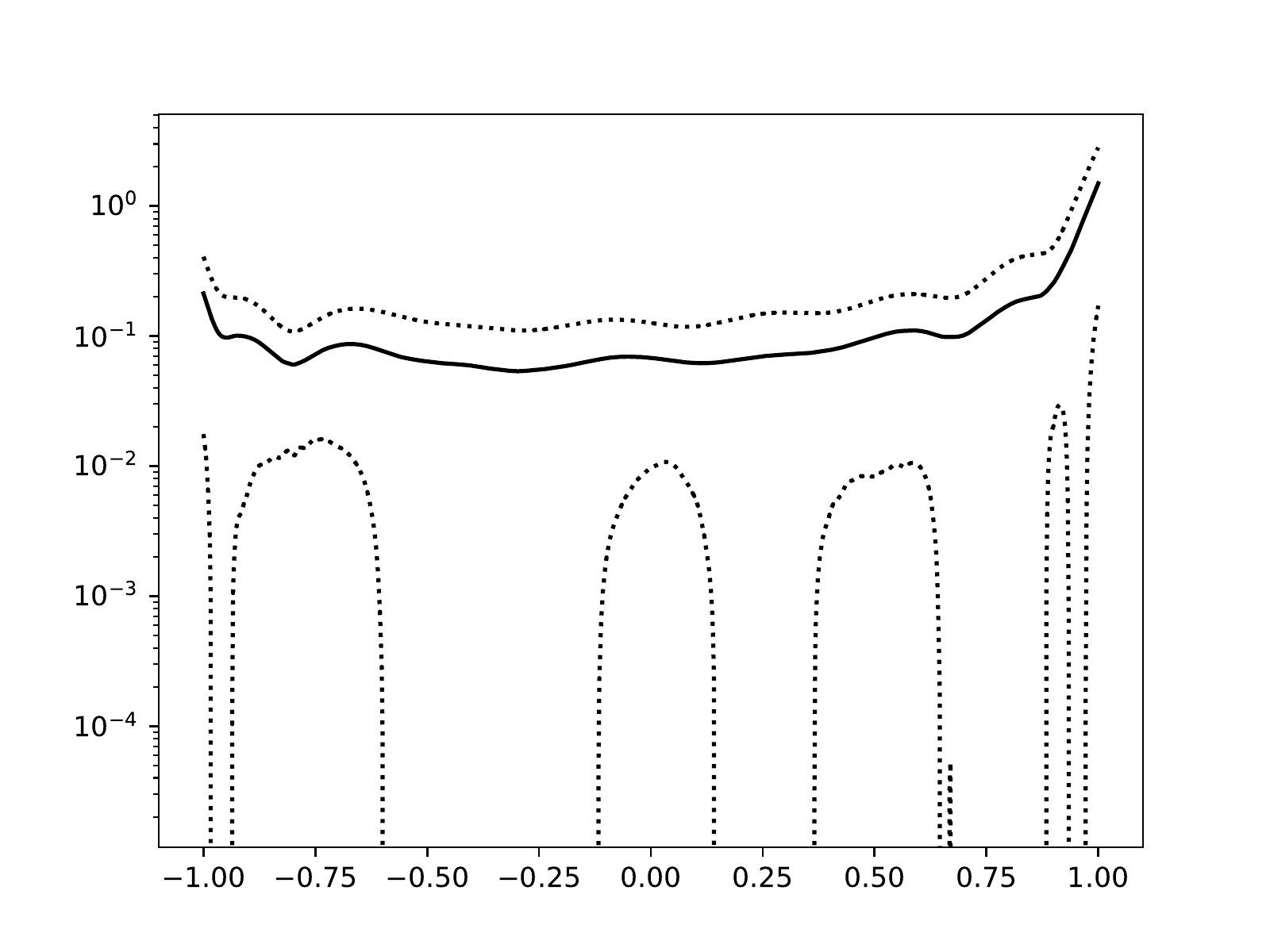}
    \caption{Estimation results of the polynomial regression problem using the stochastic gradient process with pure jump index process with $\lambda = 10$ (first row), $\lambda = 1$ (second row), $\lambda = 0.1$ (third row), and $\lambda = 0.01$ (fourth row).   The figures depict the mean over 100 runs (black solid line), mean $\pm$ standard deviation (black dotted line). Left column: trajectory of the rel\_err over time; centre column: comparison of $\Theta$ (solid red line) and estimated polynomial; right column: estimation error in terms of abs\_err.}
    \label{fig:MJP_results}
\end{figure}

\begin{table}[]
\begin{tabular}{l|l|ll}
\textbf{Method}                                                                                                     & \textbf{Parameters} & \textbf{Mean of $\mathrm{rel\_err}_{N,(\cdot)}$} & \textbf{$\pm$ StD} \\ \hline
{SGD}                                                                                                        &    $\eta_{(\cdot)} = 0.1$                 &     $1.844 \cdot 10^{-2}$                        &     $\pm 4.012 \cdot 10^{-3}$                       \\ \hline
{SGD implicit}                                                                                               &          $\eta_{(\cdot)} = 0.1$             &    $1.719 \cdot 10^{-2}$                         &            $\pm 3.939 \cdot 10^{-3}$                 \\ \hline
\multirow{3}{*}{{\begin{tabular}[c]{@{}l@{}}SGPC with \\ reflected diffusion \\ index process\end{tabular}}} &       $\sigma = 5 $            &    $1.586 \cdot 10^{-2}$                         &  $\pm 4.038 \cdot 10^{-3}$                          \\
                                                                                                                   &     $\sigma = 0.5 $                     &   $1.587 \cdot 10^{-2} $                         &      $\pm 2.979 \cdot 10^{-3}$                      \\
                                                                                                                      &     $\sigma = 0.05 $                  &   $ 4.637 \cdot 10^{-2} $                         &  $\pm 8.776 \cdot 10^{-2}$                          \\ \hline
\multirow{4}{*}{{\begin{tabular}[c]{@{}l@{}}SGPC with \\ Markov pure jump\\ index process\end{tabular}}}     &        $ \lambda = 10$            &       $2.100 \cdot 10^{-2}$                      &  $\pm 6.049 \cdot 10^{-3}$                          \\
                                                                                                                    &    $ \lambda = 1$                   &      $3.427 \cdot 10^{-2}$                       &    $\pm 1.105 \cdot 10^{-2}$                        \\
                                                                                                                    &     $ \lambda = 0.1$                  &        $3.866 \cdot 10^{-2}$                     &    $\pm 1.142 \cdot 10^{-2}$                        \\
                                                                                                                    &     $ \lambda = 0.01$                  &   $3.178 \cdot 10^{-1}$                          &    $\pm 2.124 \cdot 10^{-1}$                       
\end{tabular}
\caption{Accuracy of the estimation in the polynomial regression model. Mean and standard deviation of the relative error of the methods at the final point of their trajectory. In particular, sample mean and sample standard deviation of $ j \mapsto \mathrm{rel\_err}_{N,j}$, with $N = 5 \cdot 10^4$, computed over $100$ independent runs.   } \label{Table_Results_polynomial}
\end{table}
We learn several things from these results. Unsurprisingly, the index processes with a strong autocorrelation $(\lambda = 0.01, \sigma = 0.05)$ lead to  larger errors in the reconstruction: the processes move too slowly to capture the index spaces appropriately. In the other cases, we can assume that the processes have reached their stationary regime. 
Thus, in the figures and table, we should learn about the implicit regularization that is implicated by the different subsampling schemes, see \citet{Ali20a, smith2021on}. We especially see that the mean errors are reduced as $\sigma$ respectively $\lambda$ increases, which illustrates the approximation of the full gradient flow as shown in Theorem~\ref{wcovtheta}. Although, we should note that we compute the error to the truth $\Theta$, which is likely not the true minimizer of the full optimization problem \ref{eq:polyn_full_opt}.

\begin{figure}
    \centering
    \includegraphics[scale = 0.45]{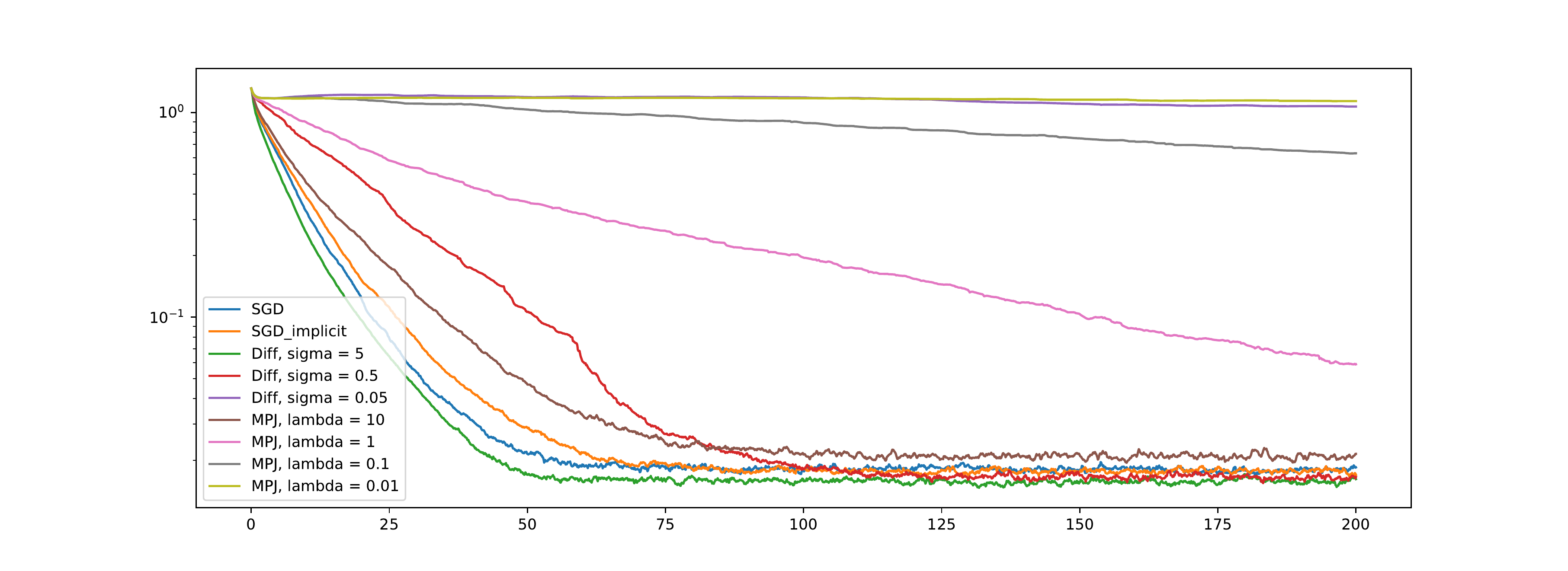}
    \caption{Comparison of the mean rel\_err of the stochastic methods for time $t \leq 200$.}
    \label{fig:error_comparison}
\end{figure}
It appears that the stochastic gradient processes with reflected diffusion index process $\sigma \in \{0.5, 5\}$ returns the best results. Looking at the error plots in the right column of Figure~\ref{fig:Diff_results}, we see that SGPC especially outperforms the other algorithms close to the boundary. For $\sigma = 5$ this could be seen as a numerical artefact due to the time step $t-t(\cdot-1)$ being too large. This though is likely not the case for $\sigma = 0.5$, where we see a similar effect, albeit a bit weaker.

In the convergence plot, Figure~\ref{fig:error_comparison}, we see for different methods different speeds of convergence to their respective stationary regime. Those speeds again depend on the autocorrelation of the processes. Interestingly, the SGPC with reflected diffusion index process and $\sigma = 5$ appears to be the best of the algorithms. 

\subsection{Solving partial differential equations using neural networks (NN)}
Partial differential equations (PDEs) are used in science and engineering to model systems and processes, such as: turbulent flow, biological growth, or elasticity. Due to the implicit nature of a PDE and its complexity, the model they represent usually needs to be approximated (`solved') numerically. Finite differences, elements, and volumes have been the state of the art for solving PDEs for the last decades. Recently, deep learning approaches have gained popularity for the approximation of PDE solutions. Here, deep learning is particularly successful in high-dimensional settings, where classical methods suffer from the curse of dimensionality. See for example \citet{PINN, DeepXDE} for physics-informed neural networks (PINN). Integrated PyTorch-based packages are available for example see \citet{NeuroDiffEq, nangs}. More recently, see \citet{FNO} for a state-of-the-art performance based on the Fourier neural operator.

Physics-informed neural networks are a very natural field of application of deep learning with continuous data. Below we introduce PINNs, the associated continuous-data optimization problem, and the state-of-the-art in the training of PINNs. Then we consider a particular PDE, showcase the applicability of SGP, and compare its performance with the standard SGD-type algorithm.

The basic idea of PINNs consists in representing the PDE solution by a deep neural network where the parameters of the network are chosen such that the PDE is optimally satisfied. Thus, the problem is reduced to an optimization problem with the loss function formulated from differential equations, boundary conditions, and initial conditions. More precisely, for PDE problems of Dirichlet type, we aim to solve a system of equations of type
\begin{equation}\label{eq:PDE}
\left\{ \begin{array}{rlll}
\mathcal{L}(u(t, x)) &= &s(t, x) \qquad &(t\in[0, \infty),\  x\in D) \\
u(0,x) &= &u_0(x) \qquad &(x\in D) \,\\
u(t, x) &= &b(t, x)  \qquad &(t\in[0, \infty),\  x\in D)
\end{array} \right.
\end{equation}
where $D\subset\mathbb{R}^d$ is an open, connected, and bounded set and $\mathcal{L}$ is a differential operator defined on a function space $V$ (e.g. $H^1(D)$). The unknown is $u:\bar{D}\to \mathbb{R}^n$. Functions $s(t, x)$, $b(t, x)$, and $u_0(x)$ are given. In numerical practice, we need to replace the infinite-dimensional space $V$ by a -- in some sense -- discrete representation. Traditionally, one employs a finite-dimensional subspace of $V$, say $\mathrm{span}\{\psi_1,\ldots,\psi_K\}$, where $\psi_1,\ldots,\psi_K$ are basis functions in a finite element method. To take advantage of the recent development of machine learning, one could solve the problem on a set of deep neural networks contained in $V$, say
\begin{align*}
    {\Big\lbrace} \psi(\cdot; \theta):\  &\psi(x; \theta) = (W^{(K)} \sigma(\cdot) + b^{(K)}) \circ \cdots \circ  (W^{(1)} \sigma(x) + b^{(1)}), x \in [0, \infty) \times D, \\
    &\theta = \left((W^{(K)}, b^{(K)}), \ldots, (W^{(1)}, b^{(1)})\right)\in \prod_{k=1}^K \left(\mathbb{R}^{n_{k} \times n_{k-1}} \times \mathbb{R}^{n_{k}}\right) =: X
    {\Big\rbrace},
\end{align*}
where $\sigma: \mathbb{R} \rightarrow \mathbb{R}$ is an activation function, applied component-wise, $n_0 = d+1$ and $n_K = 1$ to match input and output of the PDE's solution space, and $n_1, \ldots, n_{K-1}$ determine the network's architecture.

In simpler terms, let $u_{\rm}(\cdot; \theta) \in V$ be the output of a feedforward neural network (FNN) with  parameters (biases/weights) denoted by $\theta \in X$. The parameters can be learned by minimizing the mean squared error (MSE) loss
\begin{align*}
    \Phi(\theta; \mathcal{L}, s, u_0, b) :=& \int_0^\infty w(t) \int_D  \left(\mathcal{L}(u(t, x;\theta)) - s(x)\right)^2 \mathrm{d}x\mathrm{d}t+\int_{\partial D} \left(u(0, x;\theta) - u_0(x)\right)^2 \mathrm{d}x\\
    &\ +\int_0^\infty w(t) \int_{\partial D} \left(u(t, x;\theta) - b(t, x)\right)^2 \mathrm{d}x\mathrm{d}t,
\end{align*}
where the first term is the $L^2$ norm of the PDE residual, the second term is the $L^2$ norm of the residual for the initial condition, the third term is the $L^2$ norm of the residual for the boundary conditions, and $w: [0, \infty) \rightarrow [0, \infty)$ is an appropriate weight function. The FNN then represents the solution via solving the following minimization problem
\begin{equation} \label{EQ_Opt_PINNs_cont}
    \min_{\theta \in X} \Phi(\theta; \mathcal{L}, s, u_0, b).
\end{equation}
Note that in physics-informed neural networks, differential operators w.r.t. the input $x$ and the gradient w.r.t the parameter $\theta$ are both obtained using automatic differentiation. 

\subsubsection*{Training of physics-informed neural networks}
In practice, the optimization problem \eqref{EQ_Opt_PINNs_cont} is often replaced by an optimization problem with discrete potential
\begin{align*}
    \widehat{\Phi}(\theta; \mathcal{L}, s, u_0, b) :=& \sum_{k=1}^K  \left(\mathcal{L}(u(t_k, x_k;\theta)) - s(x_k)\right)^2 + \sum_{k'=1}^{K'}  \left(u(0, x'_{k'};\theta) - u_0(x'_{k'})\right)^2 \\
    &\ + \sum_{k''=1}^{K''} \left(u(t''_{k''}, x''_{k''};\theta) - b(t_{k''}'' , x_{k''}'')\right)^2,
\end{align*}
for appropriate continuous indices $$(x_{k}, t_k)_{k=1}^K \in [0, \infty)^K \times D^K, (x_{k'}')_{k'=1}^{K'} \in \partial D^K, (x_{k''}'', t_{k''}'')_{k''=1}^{K''} \in [0, \infty)^K \times \partial D^K$$ that may be chosen deterministically or randomly, see for example \citet{nangs, DeepXDE}. 

Focusing the training on a fixed set of samples can be problematic: fixing a set of random samples might be unreliable; a reliable cover of the domain will likely only be reached through tight meshing, which scales badly.   \citet{Sirignano} propose to use SGD on the continuous data space. They employ the discrete dynamic  in \eqref{Eq:SGD_discrete_time}. Naturally, we would like to follow \citet{Sirignano} and employ the SGP dynamic on the continuous index set.

To train the PINNs with SGP, we again choose the reflected Brownian motion as an index process, which we discretize with the Euler--Maruyama scheme in Algorithm~\ref{alg:RBM}. In addition, we employ mini-batching to reduce the variance in the estimator: We sample $M \in \mathbb{N}$ independent index processes $(V_t^{(1)})_{t \geq 0},\ldots, (V_t^{(M)})_{t \geq 0}$ and then employ the dynamical system
$$
\mathrm{d}\theta_t = - \frac{1}{M}\sum_{m=1}^M\nabla_\theta f(\theta_t, V_t^{(m)})\mathrm{d} t.
$$
Hence, rather than optimizing with respect to a single data set, we optimize with respect to $M$ different data sets in each iteration. While we only briefly mention the mini-batching throughout our analysis, one can easily see that it  is fully contained in our framework.

In preliminary experiments, we noticed that the Brownian motion for the sampling on the boundary is not very effective: possibly due to its localizing effect. Hence, we obtain training data on the boundary by sampling uniformly, which we consider justified as a mesh on the boundary scales more slowly as a mesh in the interior and as the boundary behavior of the considered PDE is rather predictable.

\subsubsection*{PDE and results} We now describe the partial differential equation that we aim to solve with our PINN model. After introducing  the PDE we immediately outline the PINN's architectures and show our estimation results.
We the train networks on Google Colab Pro using GPUs (often T4 and P100, sometimes K80). We are certain that a more efficient PDE solution could be obtained by classical methods, e.g., the finite element method. We do not compare the deep learning methods with classical methods, as we are mainly interested in SGP and SGD in non-convex continuous-data settings. Other methods that could approximate the PDE solution are not our focus.

The PDE we study is a transport equation; which is a linear first order, time-dependent model. One of the main advantages of studying this particular model is that we know an analytical solution that allows us to compute a precise test error.

\begin{example}[1D Transport equation] We solve the one-dimensional transport equation on the space $[0, 1]$ with periodic boundary condition:
\begin{equation}\label{eq:transport}
\left\{ \begin{array}{l}
u_t + u_x = 0, \ \ t\in[0, \infty),\ \  x\in [0, 1] \\
u(t=0) = \sin(2\pi x),\\
u(t,0) = u(t, 1).
\end{array} \right.
\end{equation}
\end{example} 
The neural network approximation of this PDE has already been studied by \citet{nangs}, our experiments partially use the code associated to this work.
The network architecture is defined by a three-layer deep neural network with 128 neurons per layer and a Rectified Linear Unit (ReLu) activation function. While theoretically the solution exists globally in time, we restrict $t$ to a compact domain and w.l.o.g, we assume $t\in[0, 1]$.
From the interior of the domain of time and space variables, i.e. $(0, 1)\times(0, 1)$, we use Algorithm~\ref{alg:RBM} with $\sigma=0.5$ to sample the train set of size $3 \cdot 10^4$ for SGPC and SGPD and we uniformly sample $600$ points for the train set of SGD. In addition, as a part of the train set for all three methods, we sample uniformly $20$ and $60$ points for the initial condition and periodic boundary condition, respectively.

The learning rate for SGD and SGPC is $0.01$. The learning rate for SGPD is defined as
$$\eta(t) = \frac{0.01}{\log(t+2)^{0.3}},$$
which is chosen such that the associated $\mu := 1/ \eta$ satisfies Assumption \ref{asmu}. %\todo[inline]{J: I don't understand what $\eta$ is in this setting and how it relates to Assumption \ref{asmu}. }
%\todo[inline,color=green!20!white]{Kexin: $\eta$ is the adaptive learning rate. We update as the follow
%$$w_{t+1} = w_t - \eta(t) \nabla f(w_t).$$
%It corresponds to $1/\mu$ in Assumption 4. Recall that for the constant case, the learning rate is $1/\Delta\frac{t}{\e}$, hence $\e$. For the decreasing case, the learning rate is $1/\Delta \beta(t)$, hence $1/\beta'=1/\mu$}
For all three methods, we use Adam \citep[see][]{Adam} as the optimizer  to speed up the convergence; we use an $L^2$ regularizer with weight $0.1$ to avoid overfitting. 
Each model is trained over $600$ iterations with batch size $50$. The training process for SGPC and SGPD contains only one epoch, while we train $50$ epochs in the SGD case. We evaluate the models by testing on a uniformly sampled test set of size $2 \cdot 10^3$ and compare the predicted values with the theoretical solution
$$u(t, x) = \sin(2\pi (x-t)).$$
We obtain the losses, the predicted solutions, and the test errors by averaging over $30$ random experiments, i.e. $30$ independent runs of SGD, SGPC, and SGPD, respectively. We give the results in Figures \ref{fig:transport_loss}, \ref{fig:transport_sol}, and \ref{fig:transport_error}. Note that the timings are very similar for each of the algorithms, the fact that SGPC and SGPD require us to first sample reflected Brownian motions is negligible.

\begin{figure}[h!]
  \centering
    \includegraphics[width=0.4\textwidth]{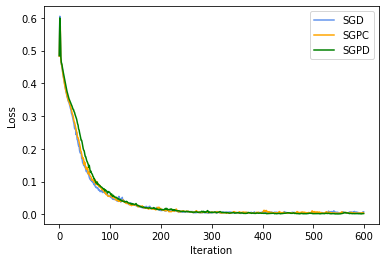}
    \includegraphics[width=0.4\textwidth]{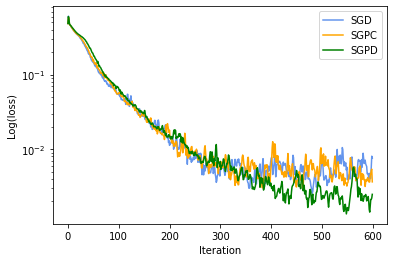}
\caption{The plots of the loss vs iteration and its log scale for SGD, SGPC, and SGPD. The losses are obtained by averaging over $30$ random experiments.}\label{fig:transport_loss}
\end{figure}

\begin{figure}[h!]
  \centering
    \includegraphics[width=0.3\textwidth]{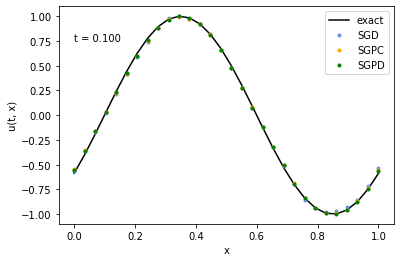}
    \includegraphics[width=0.3\textwidth]{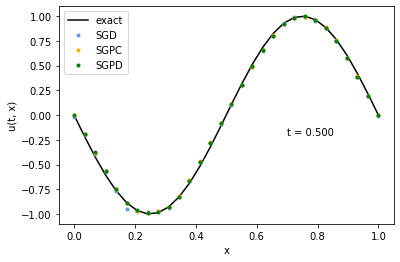}
    \includegraphics[width=0.3\textwidth]{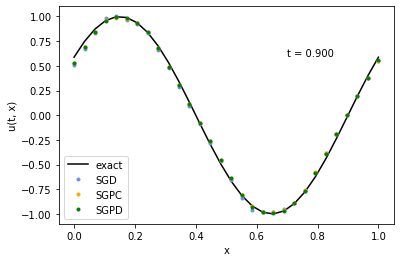}
\caption{The plots of the solutions at time $t=0.1,0.5,0.9$. We evaluate the models at $30$ uniformly sampled points. For each method, the predicted values are taken by averaging over the predicted values from the best models (the model that achieves the lowest training loss within the 600 iteration steps) in $30$ random experiments. The black curve is the theoretical solution.}\label{fig:transport_sol}
\end{figure}

\begin{figure}[h!]
  \centering
    \includegraphics[width=0.3\textwidth]{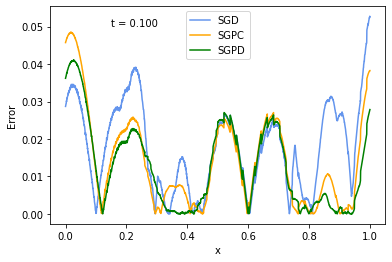}
    \includegraphics[width=0.3\textwidth]{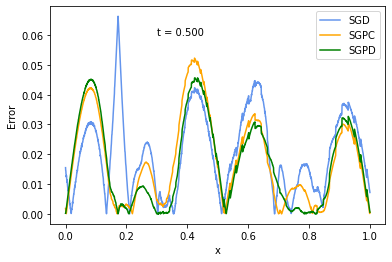}
    \includegraphics[width=0.3\textwidth]{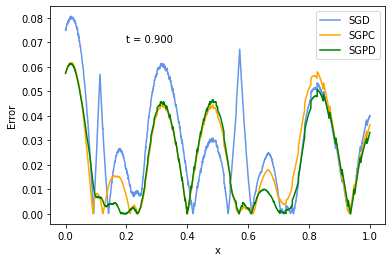}
\caption{The plots of the test error at time $t=0.1,0.5,0.9$. We evaluate the models at $2000$ uniformly sampled points. For each method, the predicted values are taken by averaging over the predicted values from the best models (the model that achieves the lowest training loss within the 600 iteration steps) in $30$ random experiments. At each point $x$, the error calculated by taking the absolute value of the difference between the predicted value and the true solution.}\label{fig:transport_error}
\end{figure}

%\todo[inline]{J:  What are exactly these numbers?} \todo[inline,color=green!20!white]{$$\frac{1}{|test\ set|}\sum_{test\ set} (y_{predicted} - y_{true})^2$$}

From Figure \ref{fig:transport_loss}, we notice that while SGD and SGPC behave similarly, SGPD does converge faster. Here, Assumption \ref{asmu} provides a way of designing a non-constant learning rate in practice. On the test set, the mean squared errors for SGD, SGPC, and SGPD are  $4.5\cdot10^{-4}$, $3.5\cdot10^{-4}$, and $2.8\cdot10^{-4}$. These test errors refer to the averaged model output of the 30 models from independent experiments. Combined with Figure \ref{fig:transport_sol} and Figure \ref{fig:transport_error}, we observe that SGPC and SGPD generalize at least slightly better on the test set. This improved generalization error might be due to the additional test data generated by the Brownian motion, as compared to the fixed training set used in PINNs. The combination with the reduction of the learning rate in SGPD, appears to be especially effective.

\section{Conclusions and outlook} \label{Sec_conclusions}
  In this work we have proposed and analyzed a continuous-time stochastic gradient descent method for optimization with respect to continuous data. Our framework is very flexible: it allows for a whole range of random sampling patterns on the continuous data space, which is particularly useful when the data is streamed or simulated. Our analysis shows ergodicity of the dynamical system under convexity assumptions -- converging to a stationary measure when the learning rate is constant and to the minimizer when the learning rate decreases. In experiments we see the suitability of the method and the effect of different sampling patterns on its implicit regularization.
  
We end this work by now briefly listing some interesting problems for future research in this area.
First, we would like to learn how the SGP sampling patterns perform in large-scale (adversarially-)robust machine learning and in other applications we have mentioned but not studied here. Moreover, from a both practical and analytical perspective, it would be interesting to also consider non-compact index spaces $S$. Those appear especially in robust optimal control and variational Bayes. Finally, we consider the following generalization of the optimization problem \eqref{Eq:OptProb} to be of high interest:
\begin{equation*}
    \min_{\theta \in X} \int_S f(\theta, y) \Pi(\mathrm{d}y|\theta),
\end{equation*}
where $\Pi$ is now a Markov kernel from $X$ to $S$. Hence, in this case the probability distribution and the sampling pattern itself depend on the parameter $\theta$. Optimization problems of this form appear in the optimal control of random systems (e.g., \cite{Deqing}) and empirical Bayes (e.g., \cite{Casella}) but also in reinforcement learning (e.g., \cite{Sutton}).

%%%%%%%%%%%%%%%%%%%%%%%%%%%%%%%%%%%%%%%%%%%%%%%%%%%%%%%%%%%%%
% Acknowledgements should go at the end, before appendices and references

\acks{JL and CBS acknowledge support from the EPSRC through grant EP/S026045/1 “PET++: Improving Localisation, Diagnosis and Quantification in Clinical and Medical PET Imaging with Randomised Optimisation”.}


\vskip 0.2in
\bibliography{sgd}

\end{document}